\documentclass[Afour,sageh,times]{sagej}

\usepackage{moreverb,url}

\usepackage[colorlinks,bookmarksopen,bookmarksnumbered,citecolor=red,urlcolor=red]{hyperref}

\newcommand\BibTeX{{\rmfamily B\kern-.05em \textsc{i\kern-.025em b}\kern-.08em
		T\kern-.1667em\lower.7ex\hbox{E}\kern-.125emX}}

\def\volumeyear{2018}

\usepackage{subcaption}
\usepackage{multicol}
\usepackage{graphicx}

\usepackage{algorithmicx}
\usepackage[Algorithm,ruled]{algorithm}
\usepackage[font={small}]{caption}
\newtheorem{theorem}{Theorem}

\newtheorem{prop}{Proposition}
\newtheorem{definition}{Definition}
\newtheorem{assumption}{Assumption}

\DeclareMathOperator{\spn}{span} 
\newtheorem{case}{\textbf\textit{Case}} 
\usepackage{enumerate}
\usepackage{mathdots}
\usepackage{resizegather}
\usepackage{xspace}
\newcommand{\MATLAB}{\textsc{Matlab}\xspace}
\hypersetup{
	linkcolor = black,
	citecolor = black,
	urlcolor = blue,
	colorlinks = true,
}

\usepackage{mathtools}
\def\mathclap#1{\text{\hbox to 0pt{\hss$\mathsurround=0pt#1$\hss}}}
\usepackage{amsmath,amsfonts, amsthm, amssymb}
\let\originalleft\left
\let\originalright\right
\renewcommand{\left}{\mathopen{}\mathclose\bgroup\originalleft}
\renewcommand{\right}{\aftergroup\egroup\originalright}

\makeatletter
\def\@seccntformat#1{\@ifundefined{#1@cntformat}%
	{\csname the#1\endcsname\quad} 
	{\csname #1@cntformat\endcsname}
}
\let\oldappendix\appendix 
\renewcommand\appendix{%
	\oldappendix
	\newcommand{\section@cntformat}{\appendixname~\thesection\quad}
}

\usepackage{etoolbox}
\makeatletter
\let\lunderbrace\underbrace
\let\runderbrace\underbrace
\let\lupbracefill\upbracefill
\let\rupbracefill\upbracefill
\patchcmd{\lunderbrace}{\upbracefill}{\lupbracefill}{}{}
\patchcmd{\runderbrace}{\upbracefill}{\rupbracefill}{}{}
\patchcmd{\lupbracefill}{\braceru}{\leaders\vrule\@height\ht\z@\@depth\z@\hskip\wd\z@}{}{}
\patchcmd{\rupbracefill}{\bracelu}{\leaders\vrule\@height\ht\z@\@depth\z@\hskip\wd\z@}{}{}
\usepackage{soul,xcolor}

\begin{document}
	
	\runninghead{Mamakoukas, MacIver, and Murphey}
	
	\title{Feedback Synthesis For Underactuated Systems Using Sequential Second-Order Needle Variations}
	
	\author{Giorgos Mamakoukas\affilnum{1}, Malcolm A. MacIver\affilnum{1,2,3} and Todd D. Murphey\affilnum{1}}
	
	\affiliation{\affilnum{1}Department of Mechanical Engineering, Northwestern University, Evanston, IL, USA\\
		\affilnum{2}Department of Biomedical Engineering, Northwestern University, Evanston, IL, USA\\
		\affilnum{3}Department of Neurobiology, Northwestern University, Evanston, IL, USA}
	
	\corrauth{Giorgos Mamakoukas, Department of Mechanical Engineering, Northwestern University, 2145 Sheridan Road, Evanston, IL 60208, USA.
		\email{giorgosmamakoukas@u.northwestern.edu}}
	
	\begin{abstract}
		This paper derives nonlinear feedback control synthesis for general control affine systems using second-order actions---the second-order needle variations of optimal control---as the basis for choosing each control response to the current state. A second result of the paper is that the method provably exploits the nonlinear controllability of a system by virtue of an explicit dependence of the second-order needle variation on the Lie bracket between vector fields. As a result, each control decision necessarily decreases the objective when the system is nonlinearly controllable using first-order Lie brackets. Simulation results using a differential drive cart, an underactuated kinematic vehicle in three dimensions, and an underactuated dynamic model of an underwater vehicle demonstrate that the method finds control solutions when the first-order analysis is singular. Lastly, the underactuated dynamic underwater vehicle model demonstrates convergence even in the presence of a velocity field. 
	\end{abstract}
	
	
	\maketitle

	\section{Introduction}
	With many important applications in aerial or underwater missions, systems are underactuated either by design---in order to reduce actuator weight, expenses, or energy consumption---or as a result of technical failures. In both cases, it is important to develop control policies that can exploit the nonlinearities of the dynamics, are general enough for this broad class of systems, and easily computable.
	Various approaches to nonlinear control range from steering methods using sinusoid controls~\citep{MurraySin}, sequential actions of Lie bracket sequences \citep{murray1994book} and backstepping \citep{kokotovic1992joy,	seto1994control} to perturbation methods \citep{junkins1986asymptotic}, sliding mode control (SMC) \citep{perruquetti2002sliding,utkin2013sliding, xu2008sliding}, intelligent control \citep{brown1997intelligent, harris1993intelligent} or hybrid control \citep{fierro1999hybrid} and nonlinear model predictive control (NMPC) methods \citep{allgower2004nonlinear}. These schemes have been successful on well-studied examples including, but not limited to, the rolling disk, the kinematic car, wheeling mobile robots, the Snakeboard, surface vessels, quadrotors, and cranes \citep{bullo2000controllability, nonholonomiccrane,escareno2012trajectory,reyhanoglu1996nonlinear,fang2003nonlinear,toussaint2000tracking,bouadi2007sliding,bouadi2007modelling,chen2013adaptive, nakazono2008vibration, shammas2012analytic, morbidi2007sliding, roy2007closed, becker2010motion, kolmanovsky1995developments,boskovic1999intelligent}. 
	
	The aforementioned methods have limitations. In the case of perturbations, the applied controls assume a future of control decisions that do not take the disturbance history into account; backstepping is generally ineffective in the presence of control limits and NMPC methods are typically computationally expensive. SMC methods suffer from chattering, which results in high energy consumption and instability risks by virtue of exciting unmodeled high-frequency dynamics \citep{khalil1996noninear}, intelligent control methods are subject to data uncertainties \citep{el2014intelligent}, while other methods are often case-specific and will not hold for the level of generality encountered in robotics. We address these limitations by using needle variations to compute real-time feedback laws for general nonlinear systems affine in control, discussed next.
	\subsection{Needle Variations Advantages to Optimal Control}
	In this paper, we investigate using needle variation methods to find optimal controls for nonlinear controllable systems. Needle variations consider the sensitivity of the cost function to infinitesimal application of controls and synthesize actions that reduce the objective \citep{aseev2014needle,shaikh2007hybrid}. Such control synthesis methods have the advantage of efficiency in terms of computational effort, making them appropriate for online feedback---similar to other model predictive control methods, such as iLQG \citep{todorov2005generalized}, but with the advantage, as shown here, of having provable formal properties over the entire state space. For time evolving objectives, as in the case of trajectory tracking tasks, controls calculated from other methods (such as sinusoids or Lie brackets for nonholonomic integrators) may be rendered ineffective as the target continuously moves to different states. In such cases, needle variation controls have the advantage of computing actions that directly reduce the cost, without depending on future control decisions. However, needle variation methods, to the best of our knowledge, have not yet considered higher than first-order sensitivities of the cost function. 
	
	We demonstrate analytically in Section III that, by considering second-order needle variations, we obtain variations that explicitly depend on the Lie brackets between vector fields and, as a consequence, the higher-order nonlinearities in the system. Later, in Section III, we show that, for classically studied systems, such as the differential drive cart, this amounts to being able to guarantee that the control approach is \emph{globally} certain to provide descent at every state, despite the conditions of Brockett's theorem \citep{brockett1983asymptotic} on nonexistence of smooth feedback laws for such systems. We extend this result by proving that second-order needle variations controls necessarily decrease the objective for the entire class of systems that are controllable with first-order Lie brackets. As a consequence, provided that the objective is convex with respect to the state (in the unconstrained sense), second-order needle variation controls provably guarantee that the agent reaches the target in a collision-free manner in the presence of obstacles without relying on predefined trajectories.
	
	\subsection{Paper Contribution and Structure}
	This paper derives the second-order sensitivity of the cost function with respect to infinitesimal duration of inserted control, which we will refer to interchangeably as the second-order mode insertion gradient or mode insertion Hessian (MIH). We relate the MIH expression to controllability analysis by revealing its underlying Lie bracket structure and present a method of using second-order needle variation actions to expand the set of states for which individual actions that guarantee descent of an objective function can be computed. Finally, we compute an analytical solution of controls that uses the first two orders of needle variations. 
	
	This paper expands the work presented in \citet{mamakoukas2017feedback} by including the derivations of the MIH, the proofs that guarantee descent, and extensive simulation results that include comparisons to alternative feedback algorithms. Further, we extend the results to account for obstacles and prove the algorithm finds collision-free solutions for the controllable systems considered, including simulations of obstacle-avoidance for static and moving obstacles.
	
	The content is structured as follows. In Section II, we provide relevant research background in the field of motion planning for controllable systems. In Section III, we prove that second-order needle variations guarantee control solutions for systems that are nonlinearly controllable using first-order Lie brackets. We use this result to provably generate collision-free trajectories that safely reach the target among obstacles provided convex objectives in the unconstrained sense. In Section IV, we present an analytical control synthesis method that uses second-order needle actions. In Section V, we implement the proposed synthesis method and present simulation results on a controllable, underactuated model of a 2D differential drive vehicle, a 3D controllable, underactuated kinematic rigid body and a 3D underactuated dynamic model of an underwater vehicle. 
	\section{Existing methods for controllable systems}
	In this section, we present a review of some popular methods that are available for underactuated, controllable systems, followed by a discussion of techniques for collision-avoidance. An introduction to these methods, as well as additional algorithms for controllable systems, can be found in \citet{la2011motion}.
	\subsection{Optimization Algorithms for Nonholonomic Controllable Systems}\label{subsection: IIA}
	
	Nonholonomic systems are underactuated agents subject to nonintegrable differential constraints. Examples include wheeled agents that are not allowed to skid (e.g., unicycle, differential drive, tricycle). Nonholonomic systems are of interest to the control community because one needs to obtain solutions for motion planning tasks \citep{kolmanovsky1995developments}.
	
	The concept of controllability is indispensable in the study of nonholonomic systems. Controllability analytically answers the existence of control solutions that move a certain agent between arbitrary states in finite time, and, in doing so, it reveals all possible effects of combined control inputs of underactuated systems that are subject to velocity, but not displacement, constraints.
	
	A popular approach in controlling nonholonomic systems is piecewise constant motion planning \citep{Sussmann91twonew, lafferriere1993differential}. Lafferriere and Sussmann \citep{lafferriere1991motion,lafferriere1993differential} extend the original dynamics with \textit{fictitious} action variables in the direction of the nested Lie brackets to determine a control for the extended system. They first compute the time the system must flow along each vector field, in a sequential manner, to accomplish a given motion of the extended system. Then, using the Campbell-Baker-Hausdorff-Dynkin (CBHD) formula \citep{strichartz1987campbell,rossmann2002lie, la2011motion}, they recover the solution in terms of the original inputs of the system.
	
	On the other hand, piecewise constant motion planning is model-specific, since the process changes for different number of inputs. In addition, solutions involve a sequence of individual actions that generate the Lie bracket motion and the actuation sequence grows increasingly larger for higher order brackets. Compensating for the third-order error in the CBHD formula involves two second-order Lie brackets and twenty successive individual inputs, each of infinitesimal duration \citep{mcmickell2007motion}. The sequence is described in detail by \citet{lafferriere1993differential}. In practice, such actuation becomes challenging as the number of switches grows. The theoretically infinitesimal duration of each input may be hard to reproduce in hardware, while, in the face of uncertainty and time-evolving trajectories, actuation consisting of a large sequence of controls (e.g., of twenty actions) is likely to change once feedback is received.
	
	Another popular approach is steering using sinusoids \citep{brockett1982control, MurraySin, murray1994book, sastry2013nonlinear, teel1995non, laumond1998guidelines}. This method applies sinusoidal control inputs of integrally related frequencies. States are sequentially brought into the desired configuration in stages, while the rest of the states remain invariant over a single cycle. This approach has been validated in generating motion of an underactuated robot fish \citep{morgansen2001nonlinear}. 
	
	Steering using sinusoids suffers from the complicated sequence of actions that grows as a function of the inputs involved. Moreover, besides also being model-specific, the method addresses each state separately, meaning each state gets controlled by its own periodic motion, requiring $N$ periods for an $N$-dimensional system, leading to slow convergence. Further, solutions focus on the final states (at the end of each cycle) and not their time evolution, hence they may temporarily increase the running cost (consider the car example of Fig. 7 in \citet{MurraySin}). As with the method of piecewise constant motion planning, when tracking a moving target, these factors also compromise the performance of this approach.
	
	Other trajectory generation techniques for controllable systems involve differential flatness \citep{lamiraux1997flatness, rathinam1998configuration, ross2004pseudospectral, fliess1995flatness, rouchon1993flatness} and kinematic reduction \citep{bullo2001kinematic, lynch2000collision, murphey2006power}. Control based on differential flatness uses outputs and their derivatives to determine control laws. However, as discussed in \citet{choudhury2004trajectory}, there is not an automatic procedure to discover whether flat outputs exist. Further, differential flatness does not apply to all controllable systems and motion planning is further complicated when control limits or obstacles are present \citep{bullo2001kinematic}. 
	
	\subsection{Motion Planning for Controllable Systems in the Presence of Obstacles}
	Controllability in its classical sense concerns itself with the existence of an action trajectory that can move the agent to a desired state, subject to the differential constraints posed by the dynamics, in the absence of obstacles. Controllability is an inherent property of the dynamics and reveals all allowable motion, disregarding the presence of physical constraints in the environment. This is true for the methods discussed in Section \ref{subsection: IIA}.
	
	Feasible path planning amidst obstacles is often treated separately from the optimal control problem. Most commonly, feasible trajectories are generated with efficient path planners, such as rapidly-exploring random tree (RRT) and probabilistic road map (PRM) methods \citep{lavalle2001randomized, hsu2002randomized}. The distinction between path planning and optimal control can be seen in work by \citet{choudhury2004trajectory, lynch2000collision} that uses such motion planners to generate trajectories among obstacles and then uses them as a reference to compute the optimal control. In this setting, nonholonomic motion consists of two stages, the path planning and the feedback synthesis that tracks the feasible trajectory.
	
	Another solution to obstacle avoidance in motion planning is the use of barrier certificates \citep{prajna2007framework}. Barrier certificates provably enforce collision-free behavior by minimally perturbing, in a least-squares sense, the control response in order to satisfy safety constraints. Feedback synthesis proceeds without accounting for obstacles and solutions are modified, only when necessary, via a quadratic program (QP) subject to constraints that ensure collision avoidance \citep{borrmann2015control, xu2015robustness, wang2017safety, ames2014control, wu2016safety}. 
	
	Additional solutions to obstacle avoidance include compensating functions that eliminate local minima in the objective caused by the obstacles \citep{deng2008lyapunov}, as well as designing navigation functions using inverse Lyapunov expressions \citep{tanner2001nonholonomic}. The former method computes the local minima in the objective and constructs a plane surface function to remove them and make the objective convex. This process can be cumbersome, as one would have to locate all local minima in the objective induced by the obstacles and then calculate the compensating function. On the other hand, navigation functions, described in \citet{rimon1992exact, tanner2001nonholonomic}, are globally convergent potential functions and are system-specific.
	
	Several of these collision-avoidance algorithms are not system-specific and could be implemented with our controller, later outlined in Section IV. In simulation results, presented in Section V, we show collision-avoidance using only penalty functions in the objective, demonstrating that the proposed controller succeeds in tasks (collision avoidance) that traditionally require sophisticated treatment. 

	\section{Needle Variation Controls based on Non-Linear Controllability}
	In this section, we relate the controllability of systems to first- and second-order needle variation actions. After presenting the MIH expression, we relate the MIH to the Lie bracket terms between vector fields. Using this connection, we tie the descent property of needle variation actions to the controllability of a system and prove that second-order needle variation controls can produce control solutions for a wider set of the configuration state space than first-order needle variation methods. As a result, we are able to constructively compute, via an analytic solution, control formulas that are guaranteed to provide descent, provided that the system is controllable with first-order Lie brackets. Generalization to higher-order Lie brackets appears to have the same structure, but that analysis is postponed to future work.
	
	\subsection{Second-Order Mode Insertion Gradient}
	Needle variation methods in optimal control have served as the basic tool in proving the Pontryagin's Maximum Principle \citep{pontryagin1962, dmitruk2014proof, garavello2005hybrid}. Using piecewise dynamics, they introduce infinitesimal perturbations in control that change the default trajectory and objective (see Fig. \ref{fig: Needle}). Such dynamics are typically used in optimal control of hybrid systems to optimize the schedule of a-priori known modes \citep{egerstedt2006transition, caldwell2016projection}.
		
	Here, instead, we consider dynamics of a single switch to obtain a new control mode $u$ at every time step that will optimally perturb the trajectory\citep{SAC}. The feedback algorithm presented in \citet{SAC}, however, only considers the first-order sensitivity of the cost function to a needle action and, as a result, often fails to provide solutions for controllable underactuated systems. By augmenting the algorithm with higher order information (via the MIH), we are able to provide solutions in cases when the first-order needle variation algorithm in \citet{SAC} is singular.

	Consider a system with state $x : \mathbb{R} \mapsto \mathbb{R}^{N} $ and control $u : \mathbb{R} \mapsto \mathbb{R}^{M \times 1} $ with control-affine dynamics of the form
	\begin{align}\label{dynamics}
	f(t,x(t), u(t))=g(t, x(t)) + h(t, x(t)) u(t),
	\end{align} 
	where $g(t,x(t))$ is the drift vector field. Consider a time period $[t_o, t_f]$ and control modes described by 

	\begin{align}\label{Dynamics}
	\dot{x}(t)=
	\begin{cases}
	f_1 (x(t), v(t)), &~~~~~~~ t_o\leq t < \tau -\frac{\lambda}{2}\\
	f_2 (x(t), u(\tau)), & \tau - \frac{\lambda}{2}\leq t < \tau + \frac{\lambda}{2} \\
	f_1 (x(t), v(t)), & \tau + \frac{\lambda}{2} \leq t \le t_f ,
	\end{cases}
	\end{align}
	where $f_1$ and $f_2$ are the dynamics associated with \textit{default} and \textit{inserted} control $v$ and $u$, respectively. Parameters $\lambda$ and $\tau$ are the duration of the inserted dynamics $f_2$ and the switching time between the two modes. 
	
	Note that the default control $v(t)$ is the input for the nominal trajectory---$v(t)$ could itself be the result of a different controller---which is then improved by the insertion of a new control vector $u(t)$ creating a switched mode $f_2$. In addition, while the default control $v(t)$ of the switched mode sequence in (2) may be time-dependent, the dynamics $f_2$ have control $u(\tau)$ that has a fixed value over $[\tau-\frac{\lambda}{2}, \tau+\frac{\lambda}{2}]$.

	\begin{figure}[]
 		\centering
 		\includegraphics[width=0.7\linewidth, height = 0.15\textheight]{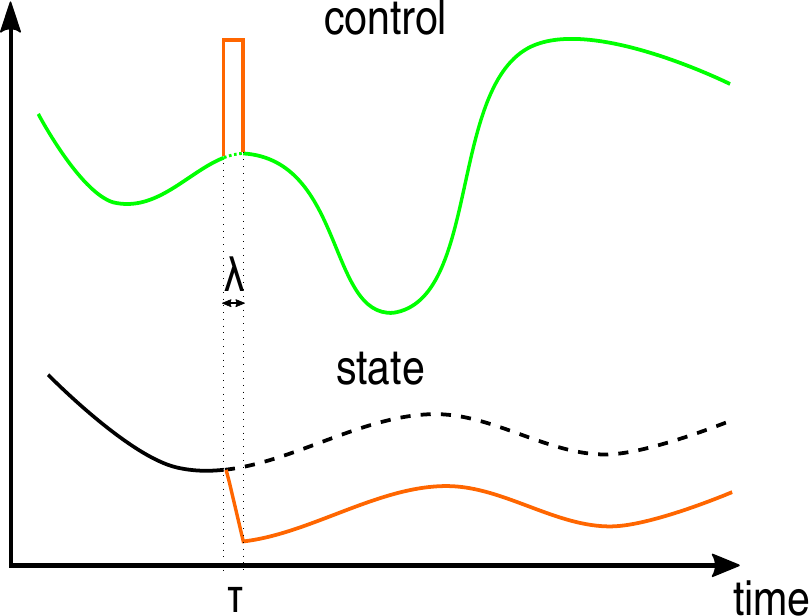}
 		\caption{A fixed-value perturbation in the nominal control, introduced at time $\tau$ and with duration $\lambda$, and the associated variation in the state. In the limit $\lambda \rightarrow 0$, the control perturbation becomes a needle variation.} \label{fig: Needle}
 	\end{figure}
 	Given a cost function $J$ of the form
	\begin{equation}\label{cost}
	J(x(t))=\int_{t_o}^{t_f} l_1(x(t)) \mathrm{d}t + m(x(t_f)),
	\end{equation}
	where $l_1(x(t))$ is the running cost and $m(x(t))$ the terminal cost, the mode insertion gradient (MIG), derived in \citet{egerstedt2006transition}, is
	\medmuskip=0.5mu
	\begin{align}\label{MIG}
	\frac{dJ}{d\lambda_+}=\rho^T (f_2 - f_1),
	\end{align}
	where $\rho : \mathbb{R} \mapsto \mathbb{R}^{N \times 1}$ is the first-order adjoint state, which is calculated from the default trajectory and given by
	\begin{gather} \label{eq:: rho}
		\dot{\rho} = -{D_xl_1}^T - D_xf_1^T\rho\\
		\text{subject to: }  \rho(t_f)~=~D_x m(x(t_f))^T.\notag
	\end{gather}
	We use the subscript $\lambda_+$ to indicate that a certain variable is considered after evaluating the limit $\lambda\rightarrow0$. For brevity, the dependencies of variables are dropped. While the objective for needle variation controls has typically not included a control term, doing so is straightforward and yields similar performance. Work in \citet{SAC} has considered objectives with control terms, and one can recompute the mode insertion gradient and mode insertion Hessian assuming the objective depends on $u$ without impacting any of the rest of the approach.
	
	The derivation of the mode insertion Hessian is similar to \citet{caldwell2011switching} and is presented in the Appendix. For dynamics that do not depend on the control duration, the mode insertion Hessian (MIH)\endnote{In this work, we consider the second-order sensitivity with respect to an action centered at one single application time $\tau$. It is also possible to consider the second-order sensitivity with respect to two application times $\tau_i$ and $\tau_j$ in the same iteration. Assuming that the entire control curve is a descent direction over the time horizon for second-order needle variation solutions, as we have proved is the case for first-order needle variation methods in recently submitted work \citep{mamakoukas_RAL}, multiple second-order needle actions at different application times would still decrease the objective. On the other hand, searching for two application times would slow down the algorithm and was not preferred in this work.} is given by
	\medmuskip = 0.5mu
	\begin{align}\label{MIH}
	\frac{d^2J}{d\lambda_+^2} =~& (f_2 - f_1)^T\Omega(f_2-f_1) + \rho^T(D_xf_2 \cdot f_2 + D_xf_1\cdot f_1 \notag\\
	&-2 D_xf_1\cdot f_2) - D_x l_1 \cdot (f_2 - f_1),
	\end{align}
	where $\Omega : \mathbb{R} \mapsto \mathbb{R}^{N \times N}$ is the second-order adjoint state, which is calculated from the default trajectory and is given by
	\begin{gather}\label{eq:: Omega}
	\dot{\Omega} = -{D_xf_1}^T\Omega - \Omega D_xf_1 - D_x^2l_1 - \sum_{i=1}^N \rho_i D_x^2f_1^i \\
	\text{subject to: }\Omega(t_f)~=~D_x^2 m(x(t_f))^T.\notag
	\end{gather}
	 The superscript $i$ in the dynamics $f_1$ refers to the $i^{th}$ element of the vector. 
	
	\subsection{Dependence of Second Order Needle Variations on Lie Bracket Structure}
	The Lie bracket of two vectors $f(x)$, and $g(x)$ is
	\begin{align*}
	[f, g](x)=\frac{\partial g}{\partial x} f(x) - \frac{\partial f}{\partial x}g(x), 
	\end{align*}
	which generates a control vector that points in the direction of the net infinitesimal change in
	state $x$ created by infinitesimal noncommutative flow $\phi_\epsilon^f\,\circ\,\phi_\epsilon^g\, \circ\,\phi_\epsilon^{-f}\,\circ\,\phi_\epsilon^{-g}\,\circ\,x_0$, where $\phi_\epsilon^f$ is the flow along a vector field $f$ for time $\epsilon$ \citep{murray1994book, jakubczyk2001introduction}. Lie brackets are most commonly used for their connection to controllability \citep{rashevsky1938connecting,Chow1940}, but here they will show up in the expression describing the second-order needle variation. 
	
	We relate second-order needle variation actions to Lie brackets in order to connect the existence of descent-providing controls to the nonlinear controllability of a system. Let $h_i : \mathbb{R} \mapsto \mathbb{R}^{N \times 1}$ denote the column control vectors that make up $h : \mathbb{R} \mapsto \mathbb{R}^{N \times M}$ in \eqref{dynamics} and $u_i \in \mathbb{R}$ be the individual control inputs. Then, we can express dynamics as
	\begin{align*}
	f=g + \sum_i^M h_iu_i.
	\end{align*}
	and, for default control $v=0$, we can re-write the MIH as
	\medmuskip = 0.3mu
	\begin{gather*}\begin{align*}
	\frac{d^2J}{d\lambda_+^2}=&\big(\sum_{i=1}^M h_iu_i\big)^T\,\Omega\sum_{j=1}^Mh_ju_j + \rho^T\Big(\sum_{i=1}^M(D_xh_iu_i)\cdot\,g \\
	&-D_xg\cdot(h_iu_i)+\sum_{i=1}^MD_xh_iu_i\sum_{i=1}^Mh_iu_i\Big)-D_xl_1\sum_{i=1}^Mh_iu_i.
	\end{align*}\end{gather*}
	Splitting the sum expression into diagonal ($i=j$) and off-diagonal ($i\ne j$) elements, and by adding and subtracting $2\sum_{i}^M\sum_{j=1}^{i-1}(D_xh_iu_i)(h_ju_j)$, we can write
	\begin{align*}
	\sum_{i=1}^MD_xh_iu_i\sum_{i=1}^Mh_iu_i =& \sum_{i}^M\sum_{j=1}^{i-1}[h_i,h_j]u_iu_j\\&
	+ 2\sum_{i}^M\sum_{j=1}^{i-1}(D_xh_iu_i)(h_ju_j) \\
	&+ \sum_{i=j=1}^M(D_xh_iu_i)(h_iu_i).
	\end{align*}
	Then, we can express the MIH as
	\begin{align*}
	\frac{d^2J}{d\lambda_+^2} =& \sum_{i=1}^M \sum_{j=1}^M u_i u_j h_i^T \Omega h_j + \rho^T\Big(\sum_{i=2}^M \sum_{j=1}^{i-1} [h_i,h_j] u_i u_j \notag \\
	&+ 2 \sum_{i=2}^M\sum_{j=1}^{i-1} (D_x h_i)h_j u_iu_j + \sum_{i=1}^{M}(D_x h_i)h_iu_iu_i \notag \\
	&+\sum_{i=1}^M [g, h_i] u_i\Big) - D_x l(\sum_{i=1}^M h_iu_i).
	\end{align*}
	The expression contains Lie bracket terms of the control vectors that appear in the system dynamics, indicating that second-order needle variations incorporate higher-order nonlinearities. By associating the MIH to Lie brackets, we next prove that second-order needle variation actions can guarantee decrease of the objective for systems that are controllable with first-order Lie brackets. 
	
	\subsection{Existence of Control Solutions with First- and Second-Order Mode Insertion Gradients}
	In this section, we prove that the first two orders of the mode insertion gradient can be used to guarantee controls that reduce objectives of the form \eqref{cost} for systems that are controllable with first-order Lie brackets. The analysis is applicable to optimization problems that satisfy the following assumptions.
	\begin{assumption}\label{as:1}
		The vector elements of dynamics $f_1$ and $f_2$ are real, bounded, $\mathcal{C}^2$ in $x$, and $\mathcal{C}^0$ in $u$ and $t$.
	\end{assumption}
	\begin{assumption}\label{as:2}
		The incremental cost $l_1(x)$ is real, bounded, and $\mathcal{C}^2$ in $x$. The terminal cost $m(x(t_f))$ is real and twice differentiable with respect to $x(t_f)$.
	\end{assumption}
	\begin{assumption}\label{as:3}
		Default and inserted controls $v$ and $u$ are real, bounded, and $\mathcal{C}^0$ in $t$.
	\end{assumption}

	Under Assumptions \ref{as:1}-\ref{as:3}, the MIG and MIH expressions exist and are unique. Then, as we show next, there are control actions that can improve any objective as long as there exists $t \in [t_o, t_f]$ for which $x(t) \ne x^*(t)$. 
\begin{definition}
	A trajectory $x^*$ described by a pair $(x^*, u^*)$ is the global minimizer of the objective function $J(x^*(t))$ for which $J(x^*(t)) \le J(x(t))~\forall~x(t)$.
\end{definition}
	
	Given Definition 1, a trajectory $x^*$ described by a pair $(x^*, u^*)$ is the global minimizer of the cost function in the unconstrained sense (not subject to the dynamics of the system) and satisfies $D_xJ(x^*(t))= 0$ throughout the time horizon considered. 
	\begin{assumption}\label{as:4}
		The pair $(x^*, u^*)$ describes the only trajectory $x^*$ for which the unconstrained derivative of the objective is equal to zero (i.e., $D_xJ(x^*(t))= 0~\forall~t\in[t_o, t_f]$).
	\end{assumption}

	Assumption 4 is necessary to prove that the first-order adjoint is non-zero, which is a requirement for the controllability results shown in this work. It assumes that the objective function in the unconstrained sense does not have a maximizer or saddle point and has only one minimizer $x^*$ described by $(x^*, u^*)$ that indicates the target trajectory or location. It is an assumption that, among other choices, can be easily satisfied with a quadratic cost function that even includes penalty functions associated with physical obstacles.

	\begin{prop}\label{nonzerorho}
		Consider a pair $(x, v)$ that describes the state and default control of \eqref{Dynamics}. If $(x, v)~\ne~(x^*, v^*)$, then the first-order adjoint $\rho$ is a non-zero vector.
	\end{prop}
	\begin{proof}
		Using \eqref{cost}, and by Assumption 4,
		\begin{align*}
		x~\ne~x^* &\Rightarrow D_xJ(x(t))~\ne~0 \\
		&\Rightarrow \int_{t_o}^{t_f} D_xl_1(x(t)) \mathrm{d}t + D_xm(x(t_f))~\ne~0 \\
		&\Rightarrow \int_{t_o}^{t_f} D_xl_1(x(t)) \mathrm{d}t~\ne~0~\text{OR}~D_xm(x(t_f))~\ne~0 \\
		&\Rightarrow D_xl_1(x(t))~\ne~0~\text{OR}~D_xm(x(t_f))~\ne~0\\
		&\Rightarrow \dot{\rho}~\ne~0~\text{OR}~\rho(t_f)~\ne~0.
		\end{align*}
		Therefore, if $x~\ne~x^*$, then $\exists~ t\in[t_o,t_f]$ such that $\rho~\ne~0$.
	\end{proof}

	\begin{prop}\label{NegGrad}
	Consider dynamics given by \eqref{Dynamics} and a trajectory described by state and control $(x, v)$. Then, there are always control solutions $u \in \mathbb{R}^M$ such that $\frac{dJ}{d\lambda_+} \le 0$ for some $t \in [t_o, t_f].$ 
\end{prop}
\begin{proof}
	Using dynamics of the form in \eqref{dynamics}, the expression of the mode insertion gradient can be written as
	\begin{align*}
	\frac{dJ}{d\lambda_+}~=~\rho^T(f_2 - f_1)~=~\rho^T\big(h(u-v)\big).
	\end{align*}
	Given controls $u$ and $v$ that generate a positive mode insertion gradient, there always exist control $u'$ such that the mode insertion gradient is negative, i.e. $u'-v~=~- (u-v)$. 
	The mode insertion gradient is zero for all $u\in\mathbb{R}^M$ if the costate vector is orthogonal to each control vector $h_i$ or if the costate vector is zero everywhere.\endnote{If the control vectors span the state space $\mathbb{R}^N$, the costate vector $\rho \in \mathbb{R}^N$ cannot be orthogonal to each of them. Therefore, for first-order controllable (fully actuated) systems, there always exist controls for which the cost can be reduced to first order.}
\end{proof}

	\begin{prop}\label{AdjVec}
		Consider dynamics given by \eqref{Dynamics} and a pair of state and control $(x, v)$ $~\ne~(x^*, v^*)$ for which $\frac{dJ}{d\lambda_+}~=~0 ~ \forall ~ u~\in \mathbb{R}^M$ and $\forall ~ t \in [t_o, t_f]$. Then, the first-order adjoint $\rho$ is orthogonal (under the Euclidean inner product) to all control vectors $h_i$. 
	\end{prop}
	\begin{proof}
		We rewrite \eqref{MIG} as
		\begin{align*}
		\frac{d J}{d \lambda_+}=0 &\Rightarrow \rho^T \sum_i^M h_i (u_i - v_i)=0 \\
		&\Rightarrow \sum_i^M k_i w_i=0 ~ \forall ~ w_i,
		\end{align*}
		where $w_i=(u_i - v_i)$ and $k_i=\rho^Th_i\in\mathbb{R}$. The linear combination of the elements of $k$ is zero for any $w_i$, which means $k$ must be the zero vector. By Proposition \ref{nonzerorho}, $\rho~\ne~0$ for a trajectory described by a pair of state and control $(x, u) \ne (x^*, u^*)$ and, as a result, $\rho^T h_i=0~\forall~i\in[1,M]$.
	\end{proof}
	\begin{prop}\label{AdjLie}
		Consider dynamics given by \eqref{Dynamics} and a pair of state and control $(x, v)~\ne~(x^*, v^*)$ for which $\frac{dJ}{d\lambda_+}=0 ~ \forall ~ u \in \mathbb{R}^M$ and $\forall ~ t \in [t_o, t_f]$. Further assume that the control vectors $h_i$ and the Lie Bracket terms $[h_i, h_j]$ and $[g, h_i]$---where $i, j\in[1, M]$---span the state space $\mathbb{R}^N$. Then, there exist $i$ and $j$ such that either $\rho^T [h_i, h_j]\ne0$ or $\rho^T	[g, h_i]\ne0$. 
	\end{prop}
	
	\begin{proof}
		Let $S =\{h_i, [h_i,h_j], [g,h_i]\}~\forall~i,j\in[1,M]$ be a set of vectors that span the state space $\mathbb{R}^N$ ($\spn\{S\}~=~\mathbb{R}^N$). Then, any vector in $\mathbb{R}^N$ can be written as a linear combination of the vectors in $S$. The first-order adjoint is an $N$-dimensional vector, which is non-zero for a trajectory described by a pair of state and control $(x, u) \ne (x^*, u^*)$ by Proposition \ref{nonzerorho}. Therefore, it can be expressed as
		\begin{equation}\label{RHO}
		\rho~=~c_1 h_1 + \dots + c_M h_M + \sum_{i=2}^M\sum_{j=1}^{i-1} c'_{i,j}[h_i, h_j] + \sum_{i=1}^M c''_i[g,h_i],
		\end{equation}
		where $c_i, c'_i, c''_i\in\mathbb{R}$. Left-multiplying \eqref{RHO} by $\rho^T$ yields
		\begin{align*}
		\rho^T \rho =~& \sum_{i=1}^{M}c_i \rho^Th_i + \sum_{i=2}^M\sum_{j=1}^{i-1} c'_{i,j}\rho^T[h_i, h_j] + \sum_{i=1}^M c''_i\rho^T[g,h_i].
		\end{align*}
		Given that $\frac{dJ}{d\lambda_+}=0$, and by Proposition \ref{AdjVec}, $\rho$ is orthogonal to all control vectors $h_i$ (which also implies that the control vectors $h_i$ do not span $\mathbb{R}^N$), the above equation simplifies to
		\begin{align*}
		\rho^T \rho =~& \sum_{i=2}^M\sum_{j=1}^{i-1} c'_{i,j}\rho^T[h_i, h_j] + \sum_{i=1}^M c''_i\rho^T[g,h_i].
		\end{align*}
		Because $\rho^T\rho~\ne~0$, there exists $i,j\in[1,M]$ and a Lie bracket term $[h_i, h_j]$, or $[g, h_i]$ that is not orthogonal to the costate $\rho$. That is,
		\begin{align*}
		\exists~i, j \in[1, M]\text{ such that }\rho^T[h_i,h_j] \neq 0~\text{OR}~\rho^T[g,h_j].
		\end{align*}
	\end{proof}

	First-order needle variation methods are singular when the mode insertion gradient is zero. When that is true, the next result---that is the main piece required for the main theoretical result of this section in Theorem 1---demonstrates that the second-order mode insertion gradient is guaranteed to be negative for systems that are controllable with first-order Lie Brackets, which in turn implies that a control solution can be found with second-order needle variation methods.
	\begin{prop}\label{nonkin}
		Consider dynamics given by \eqref{Dynamics} and a trajectory described by state and control $(x, v)~\ne~(x^*, v^*)$ for which $\frac{dJ}{d\lambda_+}~=~0$ for all $u \in \mathbb{R}^M$ and $t \in [t_o, t_f]$. If the control vectors $h_i$ and the Lie brackets $[h_i, h_j]$ and $[g, h_i]$ span the state space ($\mathbb{R}^N$), then there always exist control solutions $u\in \mathbb{R}^M$ such that $\frac{d^2J}{d\lambda_+^2} < 0$. 
	\end{prop}
	\begin{proof}
		See Appendix.
	\end{proof}
	\begin{theorem}\label{Theorem}
		Consider dynamics given by \eqref{Dynamics} and a trajectory described by state and control $(x, v)~\ne~(x^*, v^*)$. If the control vectors $h_i$ and the Lie brackets $[h_i, h_j]$ and $[g, h_i]$ span the state space $(\mathbb{R}^N)$, then there always exists a control vector $u \in \mathbb{R}^M$ and a duration $\lambda$ such that the cost function \eqref{cost} can be reduced. 
	\end{theorem}
	\begin{proof}
		The local change of the cost function \eqref{cost} due to inserted control $u$ of duration $\lambda$ can be approximated with a Taylor series expansion
		\begin{align*}
		J(\lambda) - J(0) \approx \lambda \frac{dJ}{d\lambda_+} + \frac{\lambda^2}{2} \frac{d^2J}{d\lambda_+^2}.
		\end{align*}
		By Propositions \ref{NegGrad} and \ref{nonkin}, either 1) $\frac{dJ}{d\lambda_+} <0$ or 2) $\frac{dJ}{d\lambda_+}~=~0$ and $\frac{d^2J}{d\lambda_+^2}<0$. Therefore, there always exist controls that reduce the cost function \eqref{cost} to first or second order.
	\end{proof}

	\section{Control Synthesis}
	In this section, we present an analytical solution of first- and second-order needle variation controls that reduce the cost function \eqref{cost} to second order. We then describe the algorithmic steps of the feedback scheme used in the simulation results in Section V. 
	
	\subsection{Analytical Solution for Second Order Actions}
	For underactuated systems, there are states at which $\rho$ is orthogonal to the control vectors $h_i$ (see Proposition \ref{AdjVec}). At these states, control calculations based only on first-order sensitivities fail, while controls based on second-order information still have the potential to decrease the objective provided that the control vectors and their Lie brackets span the state space (see Theorem \ref{Theorem}). We use this property to compute an analytical synthesis method that expands the set of states for which individual actions that guarantee descent of an objective function can be computed. 
	
	Consider the Taylor series expansion of the cost around control duration $\lambda$. Given the expressions of the first- and second-order mode insertion gradients, we can write the cost function \eqref{cost} as a Taylor series expansion around the infinitesimal duration $\lambda$ of inserted control $u$:
	\begin{align}\label{eq:: DJ}
	J(\lambda) & \approx J(0) + \lambda \frac{dJ}{d\lambda_+} + \frac{\lambda^2}{2} \frac{d^2J}{d\lambda_+^2}.
	\end{align}
	The first- and second-order mode insertion gradients used in the expression are functions of the inserted control $u$ in \eqref{Dynamics}. Equation \eqref{eq:: DJ} is quadratic in $u$ and, for a fixed $\lambda$, has a unique solution which is used to update the control actions. Controls that minimize the Taylor expansion of the cost will have the form
	\begin{align}\label{Taylor}
	u^{\>*}(t)=\underset{u}{\operatorname{argmin}} ~J(0) + \lambda \frac{dJ}{d\lambda_+} + \frac{\lambda^2}{2}\frac{d^2J}{d\lambda_+^2} +\frac{1}{2} \lVert u \rVert^2_R,
	\end{align}
	where the MIH has both linear and quadratic terms in $u$. We compute the minimizer of \eqref{Taylor} to be
	\begin{align}\label{optcon}
	u^{\>*}(t)=[\frac{\lambda^2}{2}\,\Gamma + R] ^{-1} \, [\frac{\lambda^2}{2}\,\Delta + \lambda (-h^T\rho)],
	\end{align}
	where $\Delta: \mathbb{R} \mapsto \mathbb{R}^{M\times1}$ and $\Gamma: \mathbb{R}\mapsto \mathbb{R}^{M\times M}$ are respectively the first- and second-order derivatives of $d^2J/d\lambda_+^2$ with respect to the control $u$ (see Appendix). These quantities are given by
	\begin{gather*}
	\begin{align*}
	\Delta\triangleq&~\Big[\big[h^T \big(\Omega^T + \Omega\big)h +
	2 h^T \cdot(\sum_{k=1}^{n} (D_xh_k)\rho_{k})^T\big]v
	\notag\\
	& + {(D_xg \cdot{h})}^{T} \rho
	-
	(\sum_{k=1}^{n} (D_xh_k)\rho_{k}) \cdot g 
	+
	h^T D_xl^T\Big]\\
	\Gamma \triangleq&~[h^T \big(\Omega^T + \Omega\big)h + 
	h^T \cdot (\sum_{k=1}^{n} (D_xh_k)\rho_{k})^T+ 
	\sum_{k=1}^{n} (D_xh_k)\rho_{k}\cdot h].
	\end{align*}
	\end{gather*}
	The parameter $R$, a positive definite matrix, denotes a metric on control effort.
	
	The existence of control solutions in \eqref{optcon} depends on the inversion of the Hessian $H~=~\frac{\lambda^2}{2}\,\Gamma + R$. To practically ensure H is positive definite, we implement a spectral decomposition on the Hessian $H~=~VDV^{-1}$, where matrices $V$ and $D$ contain the eigenvectors and eigenvalues of $H$, respectively. We replace all elements of the diagonal matrix $D$ that are smaller than $\epsilon$ with $\epsilon$
	to obtain $\bar{D}$ and replace $H$ with $\bar{H}~=~V\bar{D}V^{-1}$ in \eqref{optcon}. We prefer the spectral decomposition approach to the Levenberg-Marquardt method ($\bar{H}~=~H + \kappa I \succ 0$), because the latter affects all eigenvalues of the Hessian and further distorts the second-order information. At saddle points, we set the control equal to the eigenvector of $H$ that corresponds to the most negative eigenvalue in order to descend along the direction of most negative curvature \citep{murray2010newton,schnabel1990new, boyd2004convex, nocedal2006sequential}. 
	
	Synthesis based on \eqref{optcon} provides controls at time $t$ that guarantee to reduce the cost function \eqref{cost} for systems that are controllable using first-order Lie brackets. Control solutions are computed by forward simulating the state over a time horizon $T$ and backward simulating the first- and second-order costates $\rho$ and $\Omega$. As we see next, this leads to a very natural, and easily implementable, algorithm for applying cost-based feedback while avoiding iterative trajectory optimization. 
	
	\subsection{Algorithmic Description of Control Synthesis Method}
	\begin{algorithm}[t!]
		\begin{enumerate}[{1.}]
			\item Simulate states and costates with default dynamics $f_1$ over a time horizon $T$
			\item Compute optimal needle variation controls
			\item Saturate controls
			\item Use a line search to find control duration that ensures reduction of the cost function \eqref{cost}\endnote{The application time of the inserted action is typically chosen to correspond to the most negative first-order mode insertion gradient.}
		\end{enumerate} 
		\caption{}
		\label{algorithm}
	\end{algorithm}
	\begin{figure}[]
		\centering
		\includegraphics[width=0.95\linewidth, height = 0.40\textheight]{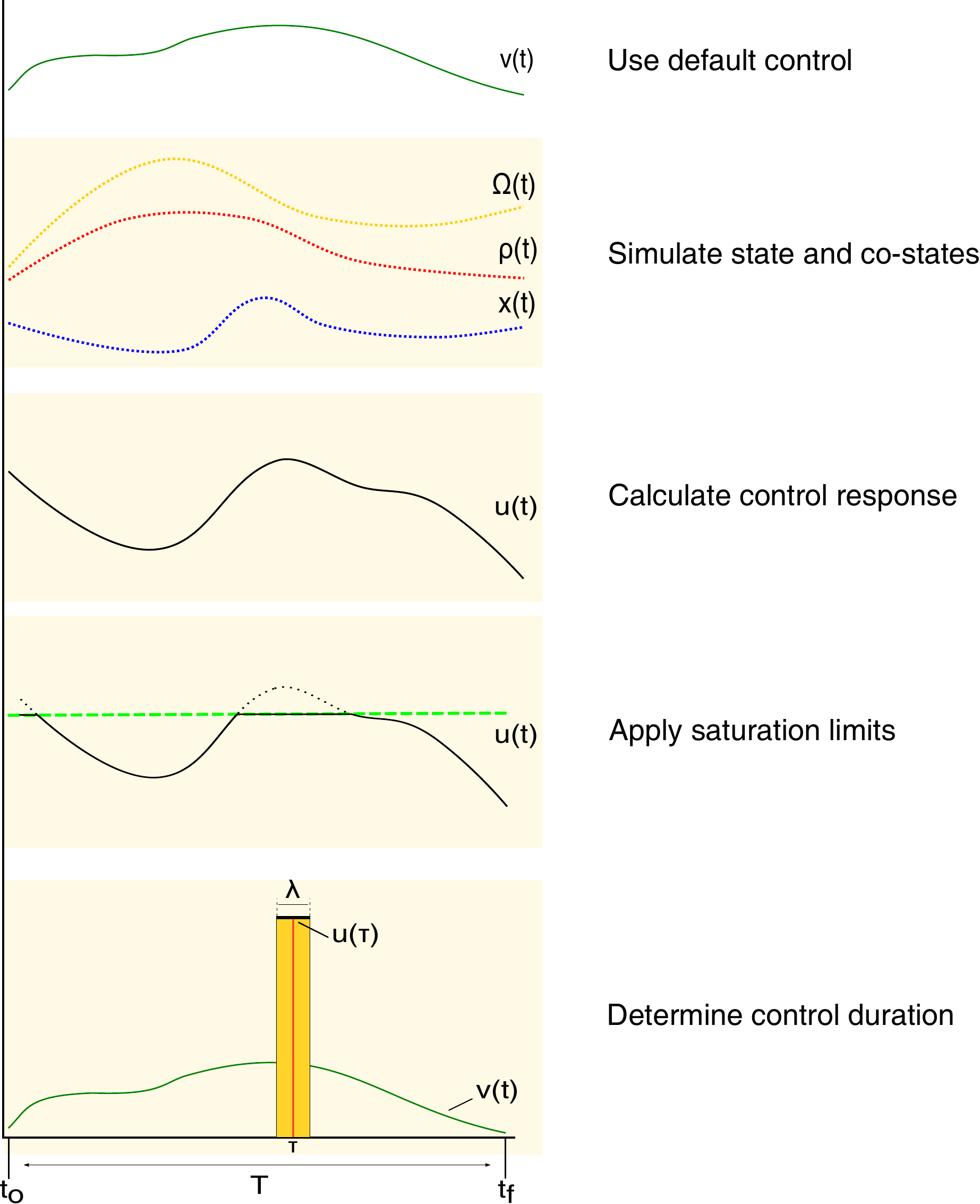}
		\caption{The steps of the controller outlined by Algorithm \ref{algorithm}. Using the default control, the states and co-states are forward-simulated for the time horizon $[t_o, t_o + T]$. The optimal control response is computed from \eqref{optcon}, and saturated appropriately. At the end, the algorithm determines the finite duration of the inserted single action, evaluated at an application time $\tau$, with a line search.} \label{fig: algo}
	\end{figure}
	The second-order controls in \eqref{optcon} are implemented in a series of steps shown in Algorithm \ref{algorithm} and visualized in Fig. \ref{fig: algo}. We compare first- and second-order needle variation actions by implementing different controls in Step 2 of Algorithm \ref{algorithm}. For the first-order case, we implement controls 
	that are the solution to a minimization problem of the first-order sensitivity of the cost function \eqref{cost} and the control effort
	\begin{align}\label{optimalu}
	u^*(t) =&\min\limits_{u} ~~ \frac{1}{2} (\frac{dJ_1}{d\lambda^+_i}-\alpha_d)^2+\frac{1}{2} \lVert u \rVert^2_R \notag\\ 
	=&(\Lambda + R^T)^{-1}(\Lambda v + h^T\rho \alpha_d),
	\end{align}
	where $\Lambda \triangleq h^T\rho\rho^Th$ and $\alpha_d \in \mathbb{R}^- $ expresses the desired value of the mode insertion gradient term (see, for example, \citet{mamakoukas2016}). Typically, $\alpha_d~=~\gamma J_o$, where $J_o$ is the cost function \eqref{cost} computed using default dynamics $f_1$. For second-order needle variation actions, we compute controls using \eqref{optcon}. As Fig. \ref{fig: algo} indicates, the applied actuation is the saturated value of the control response of either \eqref{optcon} or \eqref{optimalu}, evaluated at the application time $\tau$.
		
	While we do not show it in this paper, it is shown in \citet{SAC} that the first-order needle variation control solutions \eqref{optimalu} remain a descent direction after saturation. This result is extended in \citet{mamakoukas_RAL} to show that the entire control signal over the time horizon, and not a needle action, remains a descent direction when saturated by an arbitrary amount. While we have not yet formally proved a similar property for the second-order needle variation controls (11), one can test and identify if the saturated controls would decrease the cost function before applying any actuation. In addition, the results of this work rely on the sign and not the magnitude of the control solutions, suggesting that the saturated second-order solutions in \eqref{optcon} also provide a descent direction.
	
	Further, needle variation actuation as shown in Fig.~\ref{fig: algo} may be practically infeasible or at least problematic for motors due to the abrupt changes in the control. There are two remedies to this issue. First, introducing additional filter states associated with the control can constraint the changes in actuation \citep{taosha}. Second, one can show that the entire curve of the first-order needle variation solution is a descent direction \citep{mamakoukas_RAL}. Assuming the same is true for the second-order solutions as well, one could either apply part of the continuous control solution around the time of application $\tau$ or filter the discontinuous actuation in hardware and still provide descent with more motor-friendly actuation.

	\subsection{Convergence in the presence of obstacles}
We use Theorem \ref{Theorem} to show that the proposed needle-variation controller will always converge to the global minimizer for convex functions. This statement is true independent of the presence of obstacles. 
\begin{theorem}\label{thm:: Ob_converge}
	Consider dynamics given by \eqref{Dynamics}, a trajectory described by state and control $(x, v)\ne(x^*, v^*)$. Let $\tilde{x}^k$ describe the trajectory generated after $k$ iterations of Algorithm \ref{algorithm}. Further let $x\in\mathcal{X}_{free}\forall~t\in[t_o, t_f]$, where $\mathcal{X}_{free}\subset\mathcal{X}$ denotes the collision-free part of the state-space. Consider an objective \eqref{cost} that satisfies Assumption 2 and whose running cost term penalizes collisions, such that $J(\tilde{x}^k) > J(x)$ if $\exists~t\in[t_o, t_f]$ where $\tilde{x}\notin\mathcal{X}_{free}$. If the objective $J(x)$ is convex with respect to the state $x$ in the unconstrained sense and the control vectors $h_i$ and the Lie brackets $[h_i, h_j]$ and $[g, h_i]$ span the state space $(\mathbb{R}^N)$, then the sequence of solutions $\{\tilde{x}^k\}$ generated by Algorithm \ref{algorithm} converges to $x^*$. Further, $\tilde{x}^k\in\mathcal{X}_{free}\forall~k\in\mathbb{R}^+$. 
\end{theorem}

\begin{proof}
	Algorithm \ref{algorithm} constructs control responses out of the first- and second-order mode insertion gradients. By extension of Theorem \ref{Theorem}, it can guarantee to reduce the objective with each iteration (for some control $u$ of duration $\lambda$). Therefore, 
	\begin{align}\label{eq:: cost_decrease}
	J(\tilde{x}^k) > J(\tilde{x}^{k+1}) \ge J_{min},
	\end{align}
	where $J_{min} = J(x^*)$ is the (only) minimizer of the convex objective. It follows that,
	\begin{align}\label{eq:: converge}
	\lim\limits_{k\rightarrow \infty}J(\tilde{x}^k) = J_{min}.
	\end{align}
	Further, assume that there exists $t\in[t_o, t_f]$ and $k\in\mathbb{R}^+$ such that $\tilde{x}^k\notin\mathcal{X}_{free}$. 
	Then $J(\tilde{x}^k) > J(x)$, which contradicts \eqref{eq:: cost_decrease}. Using proof by contradiction, we conclude that 
	\begin{align}\label{eq:: collision_free}
	\tilde{x}^k\in\mathcal{X}_{free}~\forall~t\in[t_o, t_f],~\forall~k\in\mathbb{R}^+. 
	\end{align}
	Assuming that a collision-free path exists between the agent and the target, it is straightforward that the minimizer trajectory is the target's location. Therefore, from \eqref{eq:: converge} and \eqref{eq:: collision_free}, Algorithm \ref{algorithm} generates a sequence $\{\tilde{x}^k\}$ that converges to the target safely.\endnote{Although control responses are constructed from a second-order Taylor series approximation of the objective, the iterated sequence is guaranteed to decrease the real cost function at each iteration. If the real cost function is convex with respect to  the state (in the unconstrained sense), the iterated sequence will converge towards the only minimizer. Using a sufficient descent condition, the algorithm is guaranteed to converge to a point where the unconstrained derivative of the objective is zero (i.e.,  $D_xJ(x) = 0$), which, from Assumption 4, only happens at the global minimizer.}
\end{proof}

With regards to Theorem \ref{thm:: Ob_converge}, we should alert the reader that Algorithm \ref{algorithm} may not guarantee collision-avoidance if the default trajectory is not collision-free, that is if there exists $t\in[t_o, t_f]$ such that $x\notin\mathcal{X}_{free}$. Further, the result of Theorem \ref{thm:: Ob_converge} can generalize to non-convex functions that have only one minimum. 
	\subsection{Comparison to Alternative Optimization Approaches}
	Algorithm \ref{algorithm} differs from controllers that compute control sequences over the entire time horizon in order to locally minimize the cost function. Rather, the proposed scheme utilizes the \emph{time-evolving sensitivity} of the objective to an infinitesimal switch from $v$ to $u$ and searches a one-dimensional space for a finite duration of a single action that will optimally improve the cost. It does so using a closed-form expression and, as a result, avoids the expensive iterative computational search in high-dimensional spaces, while it may still get closer to the optimizer with one iterate. 
	
	Specifically, in terms of computational effort, Algorithm \ref{algorithm} computes two $n \times 1$ (state \eqref{Dynamics} and first-order adjoint \eqref{eq:: rho} variables) and one $n \times n$ (second-order adjoint \eqref{eq:: Omega}) differential equations and searches. All simulations presented in this paper are able to run in real time, including the final 13-dimensional system. However, real-time execution is not guaranteed for higher dimensional systems. Nevertheless, the presented algorithm runs faster than the iLQG method for the simulations considered here.
	\begin{figure*}[h!]
		\centering
		\begin{subfigure}{0.5\columnwidth}
			\includegraphics[width=\linewidth, height = \linewidth]{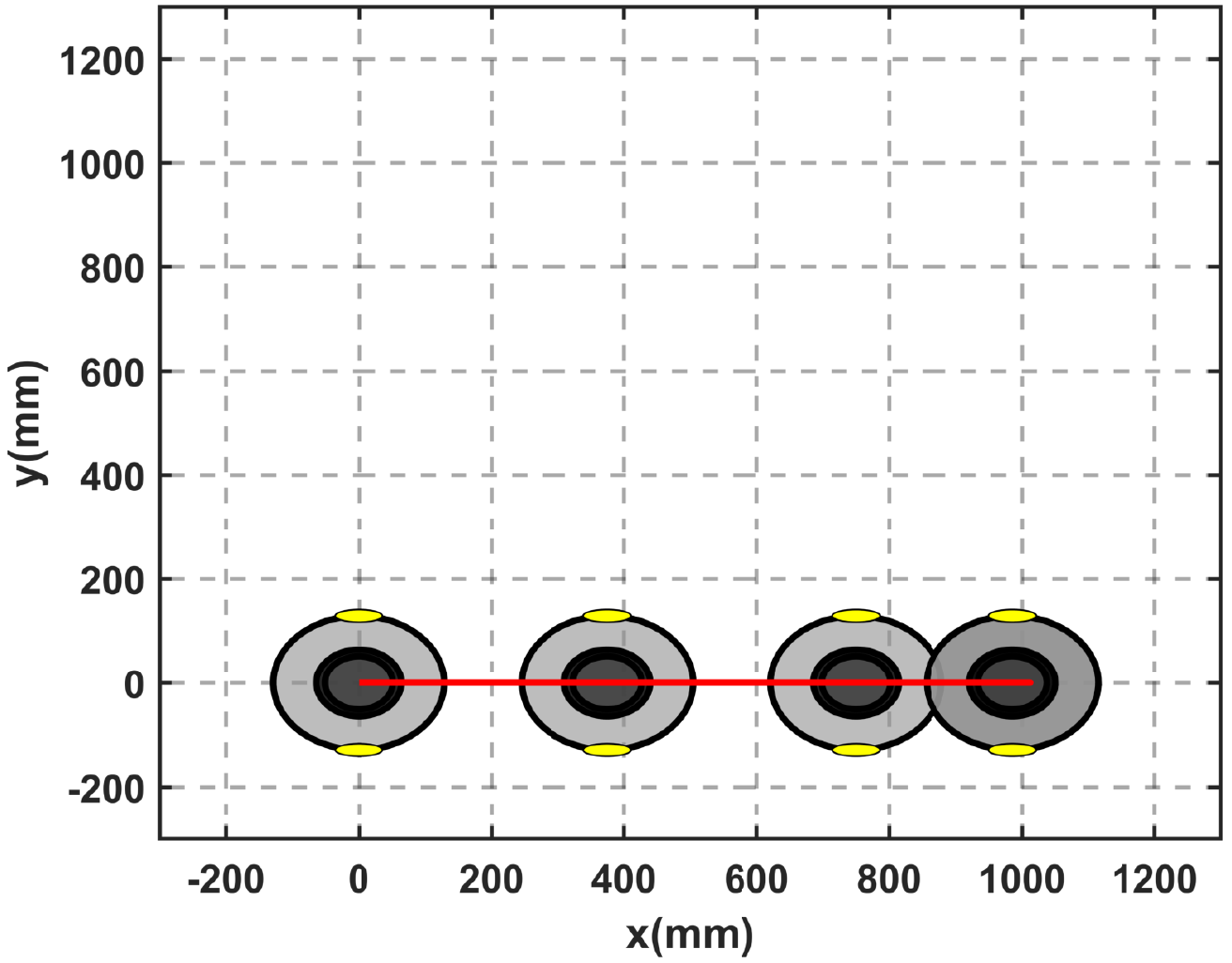}%
			\caption{}%
			\label{fig7:a} %
		\end{subfigure}\hfill%
		\begin{subfigure}{0.5\columnwidth}
			\includegraphics[width=\linewidth, height = \linewidth]{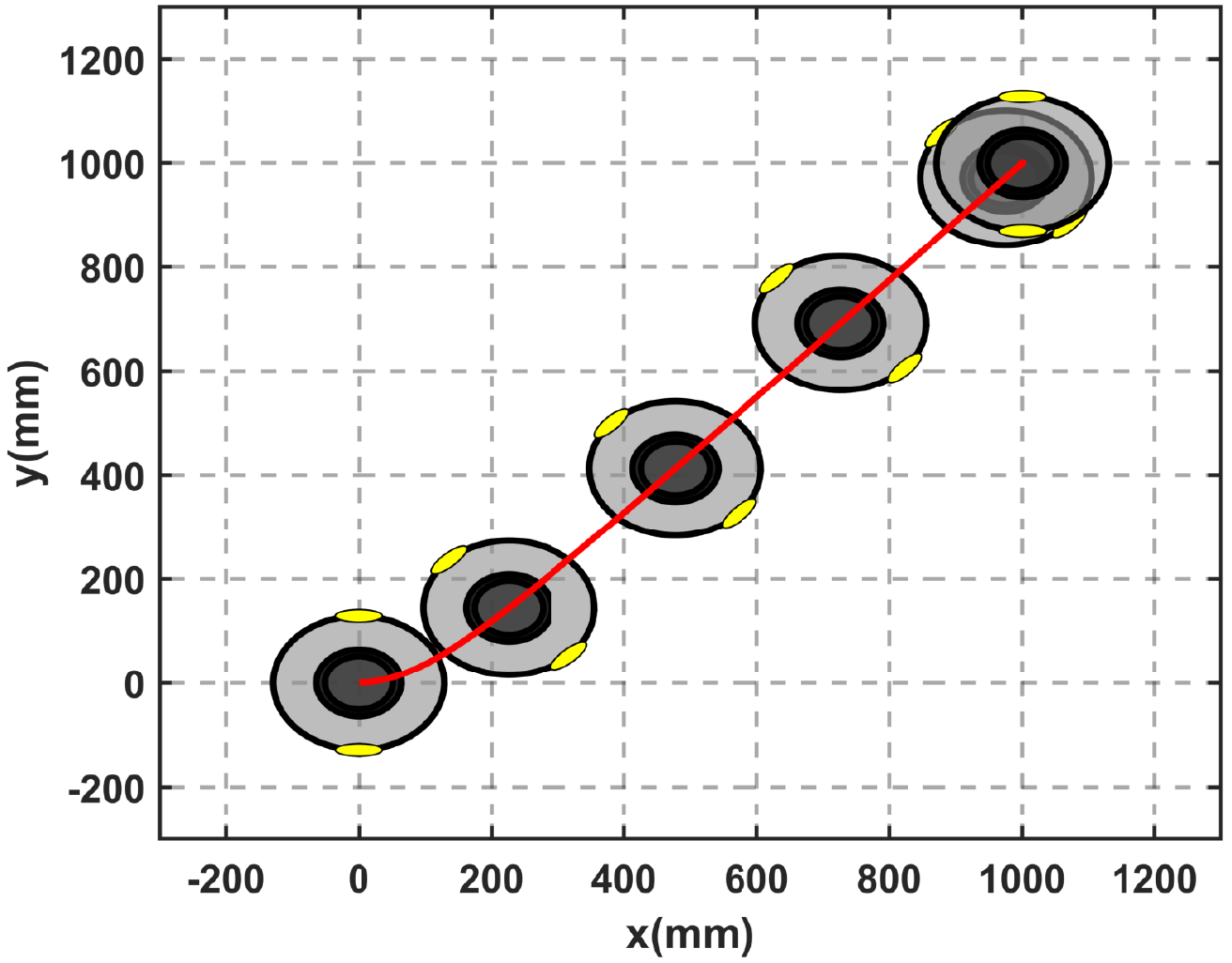}%
			\caption{}%
			\label{fig7:b} 
		\end{subfigure}\hfill%
		\begin{subfigure}{0.5\columnwidth}
			\includegraphics[width=\linewidth, height = \linewidth]{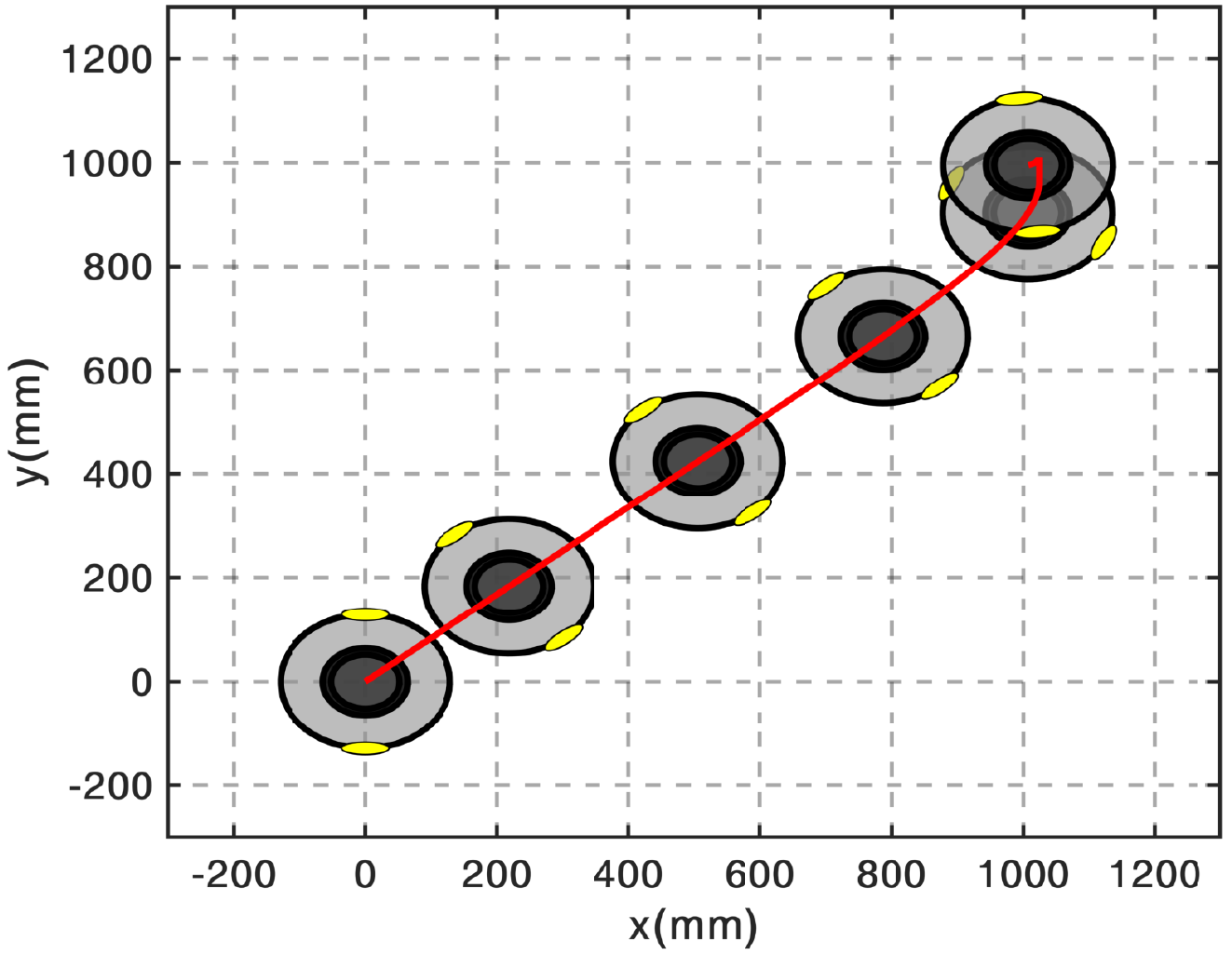}%
			\caption{}%
			\label{fig7:c} 
		\end{subfigure}\hfill%
	\begin{subfigure}{0.5\columnwidth}
		\includegraphics[width=\linewidth, height = \linewidth]{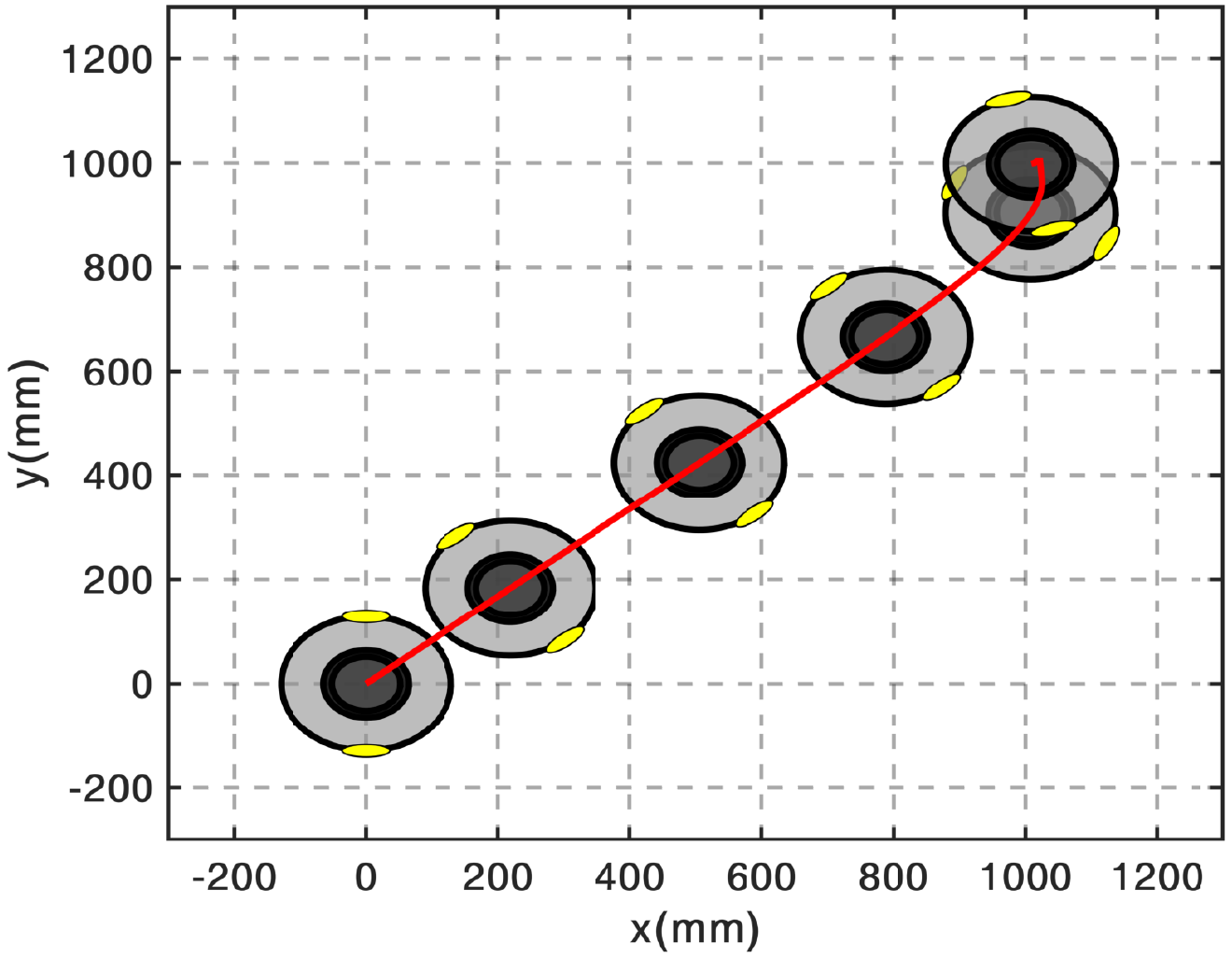}%
		\caption{}%
		\label{fig7:cD} 
	\end{subfigure}\hfill%
		\\
		\begin{subfigure}{0.5\columnwidth}
			\includegraphics[width=\linewidth, height = \linewidth]{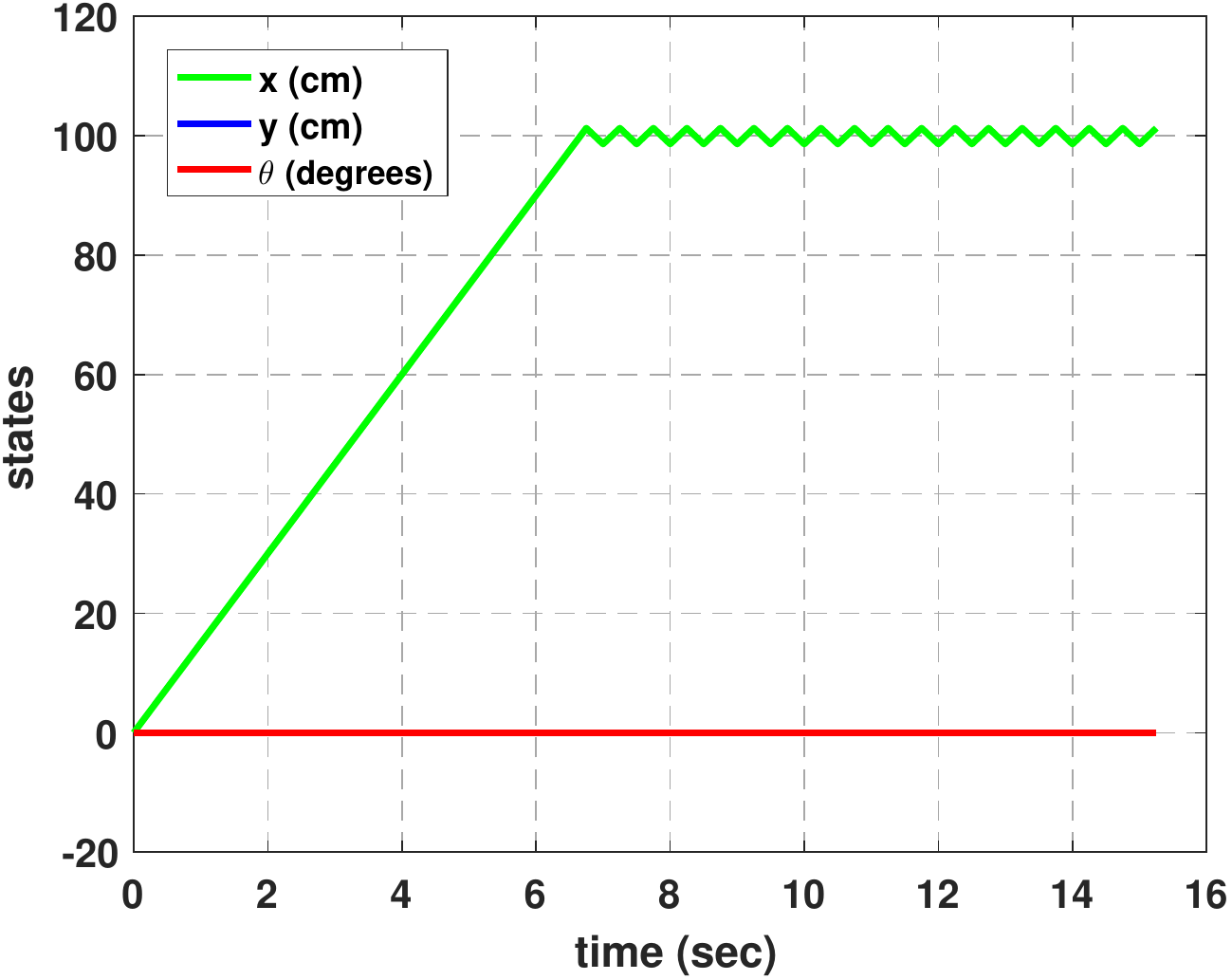}%
			\caption{}%
			\label{fig7:d} %
		\end{subfigure}\hfill%
		\begin{subfigure}{0.5\columnwidth}
			\includegraphics[width=\linewidth, height = \linewidth]{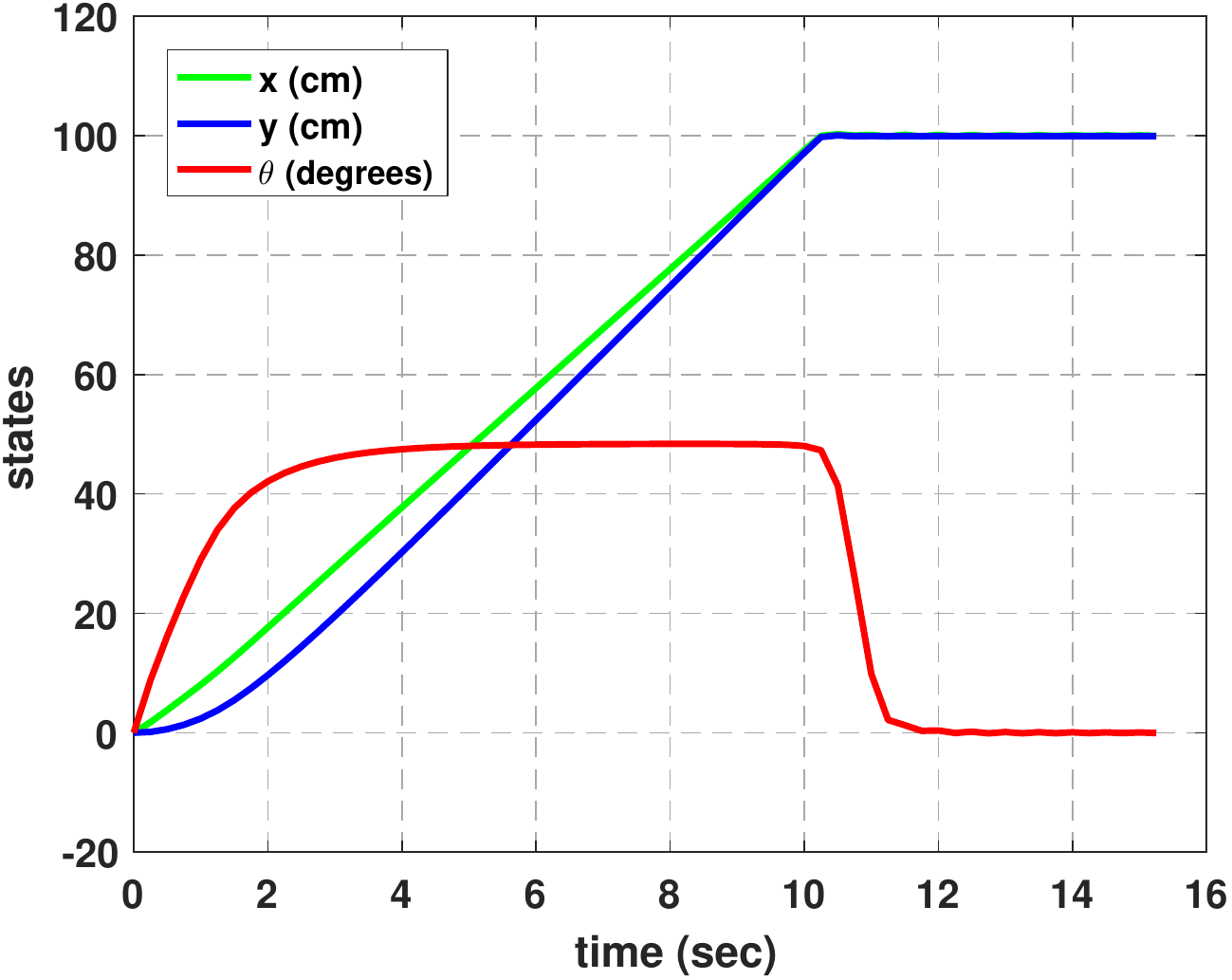}%
			\caption{}%
			\label{fig7:e} 
		\end{subfigure}\hfill%
		\begin{subfigure}{0.5\columnwidth}
			\includegraphics[width=\linewidth, height = \linewidth]{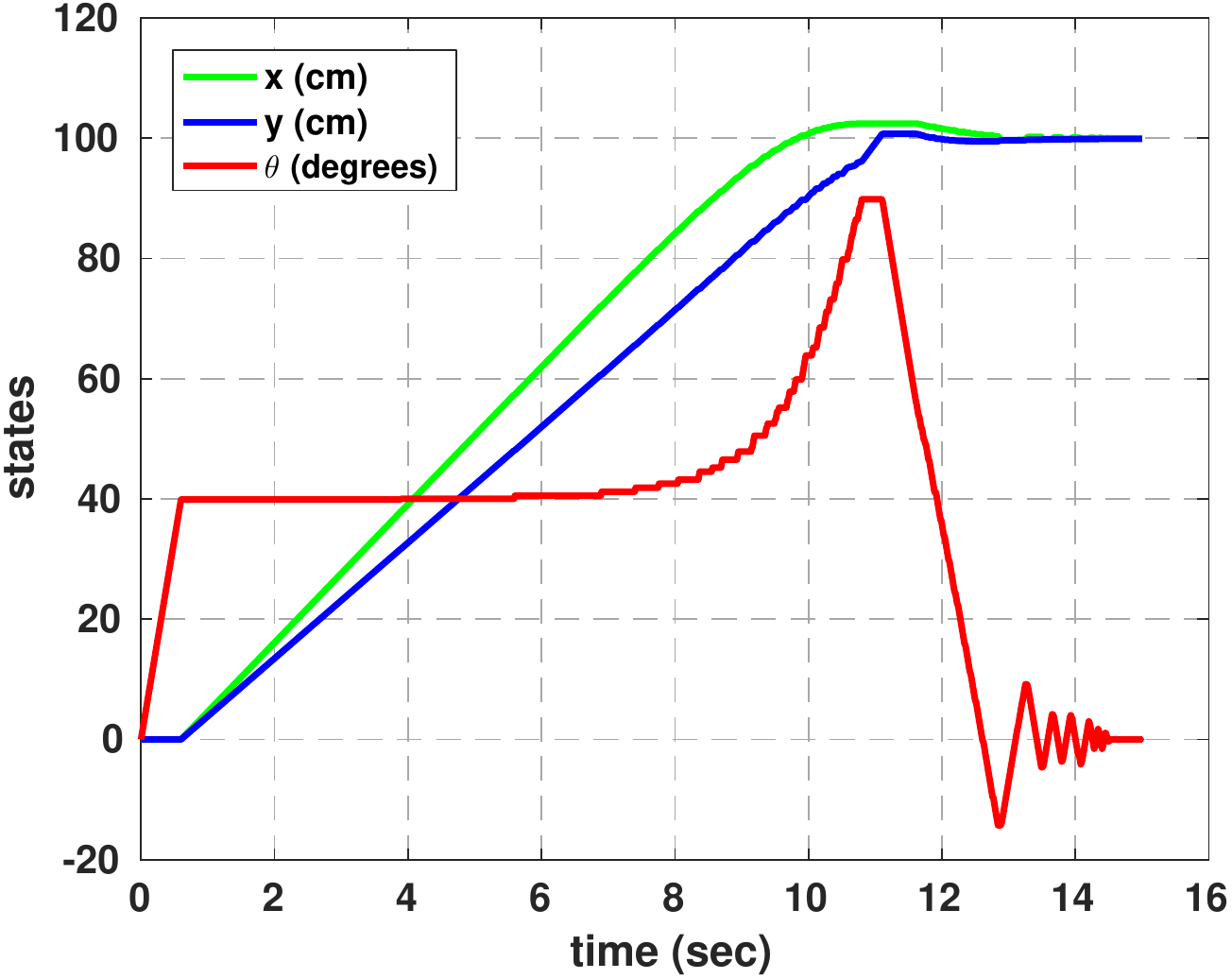}%
			\caption{}%
			\label{fig7:f} 
		\end{subfigure}\hfill%
	\begin{subfigure}{0.5\columnwidth}
		\includegraphics[width=\linewidth, height = \linewidth]{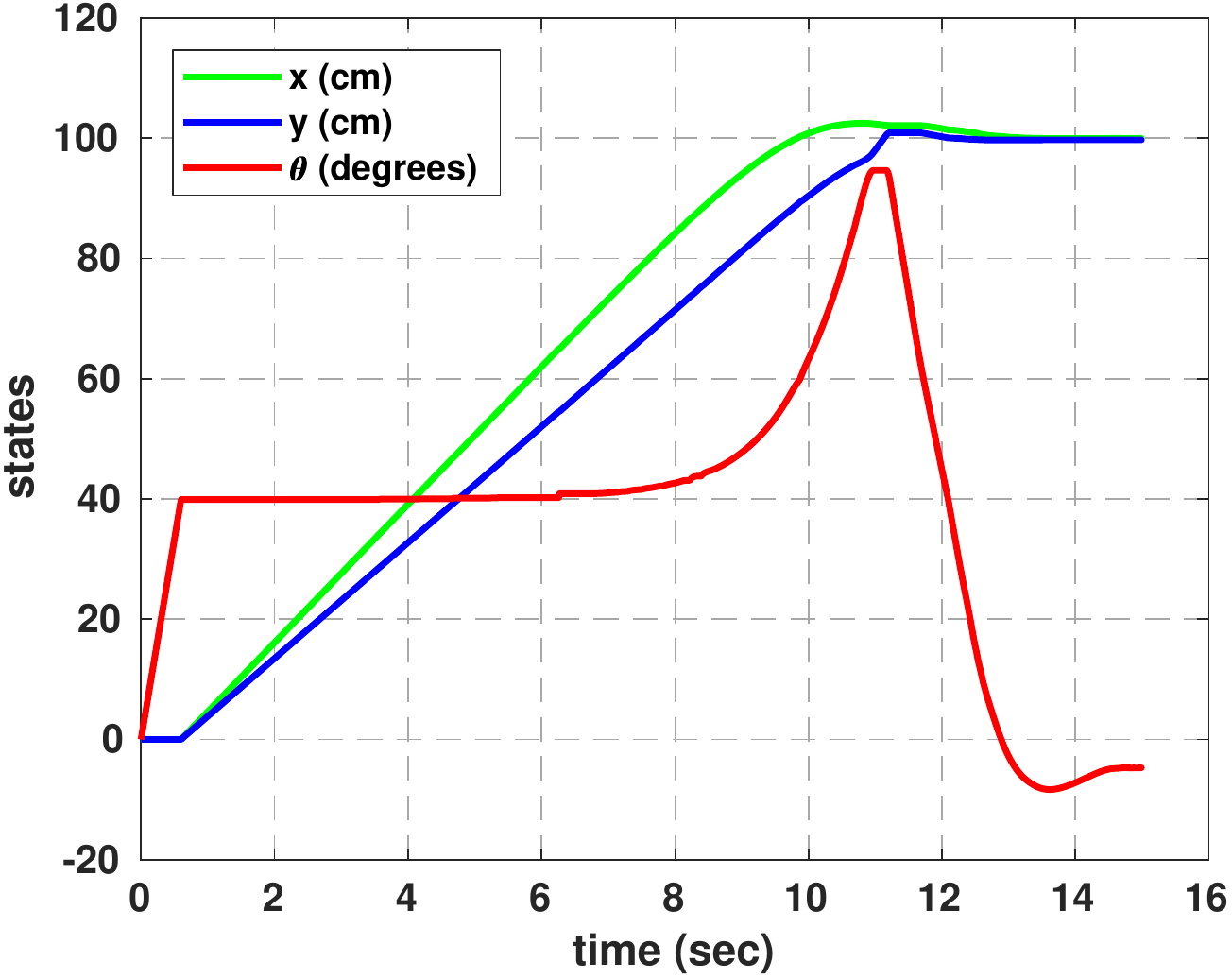}%
		\caption{}%
		\label{fig7:fD} 
	\end{subfigure}\hfill%
		\caption{Differential drive using first-, second-order needle variation actions, iLQG, and DDP, from left to right. Snapshots of the system are shown at $t~=~0, 2.5, 5, 7.5, 10$, and $12.5$~sec. The target state is $[x_d, y_d, \theta_d]~=~[1000$~mm$,1000$~mm$,0]$.}
		\label{Differential Drive} 
	\end{figure*}
		\begin{figure}[]
		\begin{subfigure}[b]{0.49\linewidth}
			\centering
			\includegraphics[width=4cm,height = 4cm]{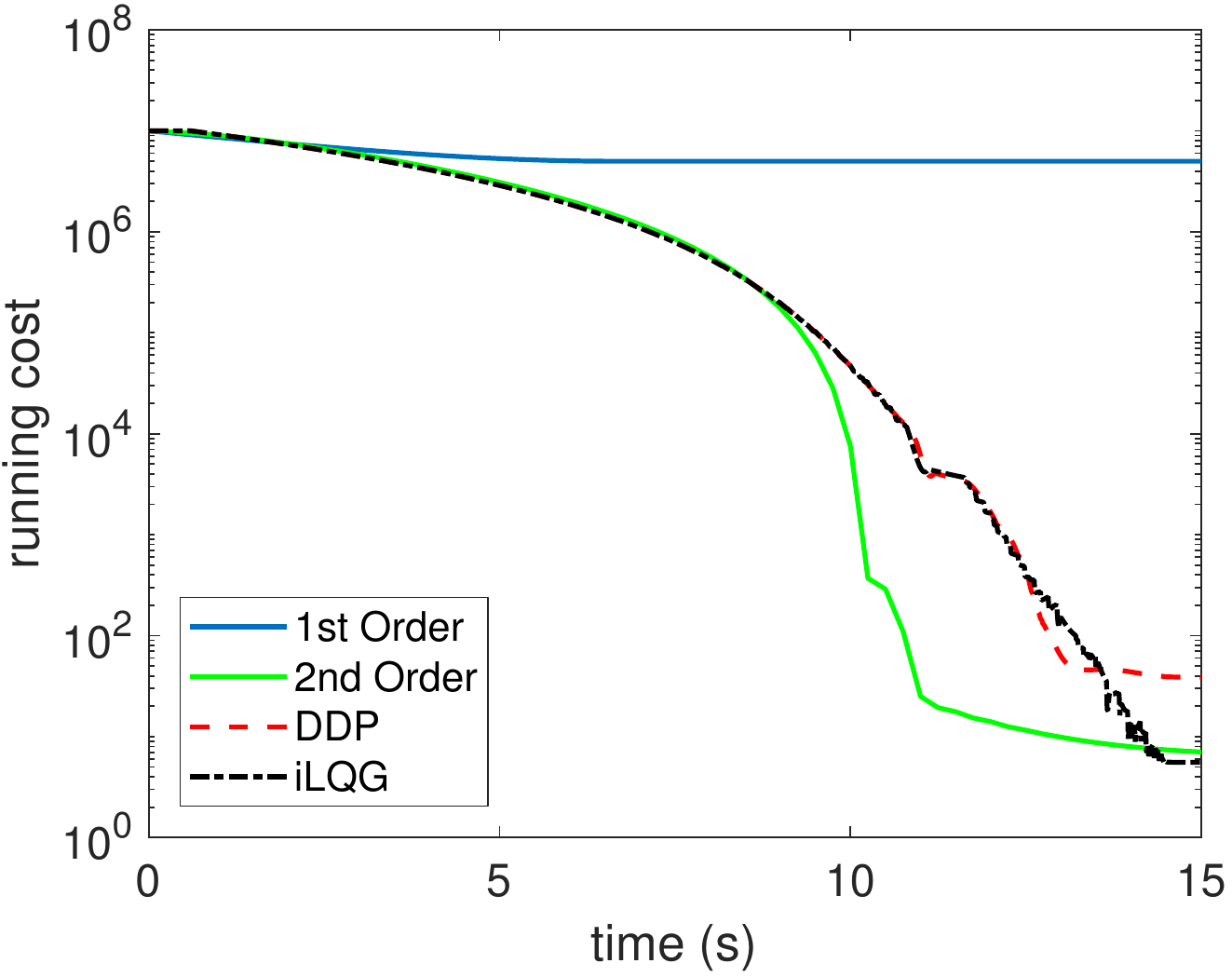} 
			\caption{} \label{fig:: RunCost}
		\end{subfigure} 
		\hfill%
		\begin{subfigure}[b]{0.49\linewidth}
			\centering
			\includegraphics[width=4cm,height =4cm]{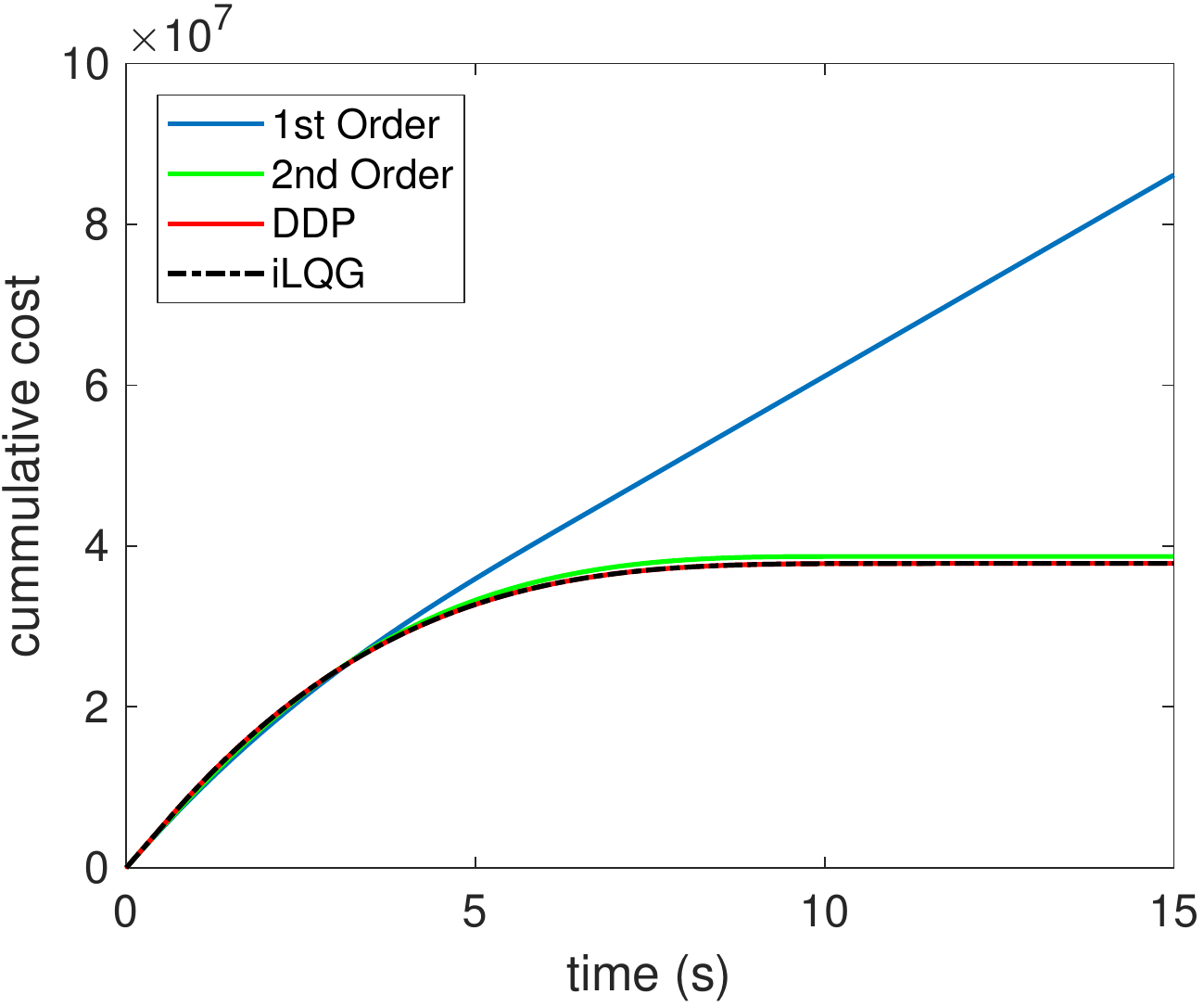} 
			\caption{} 
			\label{fig:: CompDifDrive_Cum} 
		\end{subfigure} 
		\caption{Fig. \ref{fig:: RunCost} plots the running state cost; Fig. \ref{fig:: CompDifDrive_Cum} plots the integrated (cumulative) cost, including the control cost. DDP and iLQG obtain the same cumulative cost, with slightly different trajectories (see Fig. \ref{Differential Drive}). Second-order needle variation actions demonstrate improved convergence to the target over DDP and iLQG, despite optimizing over one single action at each iteration.}
		\label{fig:: Comparisons_DifDrive}
	\end{figure}

	Further, compared to traditional optimization algorithms such as iLQG, needle variation solutions exist globally, demonstrate a larger region of attraction and have a less complicated representation on Lie Groups \citep{taosha}. These traits naturally transfer to second-order needle controls \eqref{optcon} that also contain the first-order information present in \eqref{optimalu}. In addition, as this paper demonstrates, the suggested second-order needle variation controller has formal guarantees of descent for systems that are controllable with first-order Lie brackets, which---to the best of our knowledge---is not provided by any alternative method. 
	
	Given these benefits, the authors propose second-order needle variation actions as a complement to existing approaches for time-sensitive robotic applications that may be subject to large initial error, Euler angle singularities, or fast-evolving (and uncertain) objectives.
	Next, we implement Algorithm \ref{algorithm} using first or second-order needle variation controls (shown in \eqref{optimalu} and \eqref{optcon}, respectively) to compare them in terms of convergence success on various underactuated systems. 
	\section{Simulation Results}
	The proposed synthesis method based on \eqref{optcon} is implemented on three underactuated examples---the differential drive cart, a 3D kinematic rigid body, and a dynamic model of an underwater vehicle. The kinematic systems of a 2D differential drive and a 3D rigid body are controllable using first-order Lie brackets of the vector fields and help demonstrate Theorem \ref{Theorem}. The underactuated dynamic model of a 3D rigid body serves to compare controls in \eqref{optcon} and \eqref{optimalu}, as well as make comparisons to other control techniques, in a more sophisticated environment. In all simulation results, we start with default control $v~=~0$ and an objective function of the form
	\begin{align*}
	J(x(t))=\frac{1}{2}\int_{t_o}^{t_f} \lVert \vec{x}(t)-\vec{x}_d (t) \rVert^2_Q dt+\frac{1}{2}\lVert \vec{x}(t_f)-\vec{x}_d(t_f)\rVert^2_{P_1},
	\end{align*}
	where $\vec{x}_d$ is the desired state-trajectory, and $Q=Q^T \geq 0$, $P_1=P_1^T \geq 0$ are metrics on state error.
	\subsection{2D Kinematic Differential Drive}
	
	We use the differential drive system to demonstrate that first-order controls shown in \eqref{optimalu} that are based only on the first-order sensitivity of the cost function \eqref{cost} can be insufficient for controllable systems, contrary to controls shown in \eqref{optcon} that guarantee decrease of the objective for systems that are controllable using first-order Lie brackets (see Theorem \ref{Theorem}). 
	
	The system states are its coordinates and orientation, given by $s~=~[x, y, \theta]^T$, with kinematic ($g=0$) dynamics
	\begin{align*}
	f=r\begin{bmatrix} cos(\theta) & cos(\theta) \\
	sin(\theta) & sin(\theta) \\
	\frac{1}{L} & -\frac{1}{L}\end{bmatrix}
	\begin{bmatrix}u_R \\ u_L \end{bmatrix},
	\end{align*}
	where $r~=~3.6$~cm, $L~=~25.8$~cm denote the wheel radius and the distance between them, and $u_R$, $u_L$ are the right and left wheel control angular velocities, respectively (these parameter values match the specifications of the iRobot Roomba). The control vectors $h_1$, $h_2$ and their Lie bracket term $[h_1, h_2]~=~2\frac{r^2}{L}\big[-sin(\theta),-cos(\theta)\big]^T$ span the state space ($\mathbb{R}^3$). Therefore, from Theorem \ref{Theorem}, there always exist controls that reduce the cost to first or second order. 
	
	Fig.~\ref{Differential Drive} and \ref{fig:: Comparisons_DifDrive} demonstrate how first-, second-order needle variations, iLQG, and DDP \citep{todorov2005generalized, tassa2014control} perform on reaching a nearby target. We implement the iLQG and DDP algorithms to generate offline trajectory optimization solutions using the publicly available software.\endnote{Available at http://www.mathworks.com/matlabcentral/\\fileexchange/52069-ilqg-ddp-trajectory-optimization.} Actions based on first-order needle variations \eqref{optimalu} do not generate solutions that turn the vehicle, but rather drive it straight until the orthogonal displacement between the system and the target location is minimized. Actions based on second-order needle variations \eqref{optcon}, on the other hand, converge successfully. The solutions differ from the trajectories computed by iLQG and DDP, despite using the same simulation parameters. 
	
	We note the fact that, besides the computational benefits, single-action approaches appear to be rich in information and perform comparably to offline schemes that attempt to minimize the objective by computing different control responses over the entire horizon. Given that the solutions of iLQG and DDP are very similar, and the fact that DDP is slower than iLQG due to expanding the dynamics to second order, we use only the iLQG algorithm as a means of comparison for the rest of the simulations presented in this work. The results in Fig.~\ref{Differential Drive} based on second-order needle variations are generated in real time in \MATLAB and approximately forty times faster than the iLQG implementation. 
	\begin{figure}[]
		\centering
		\includegraphics[width=0.5\linewidth, height=0.15\textheight]{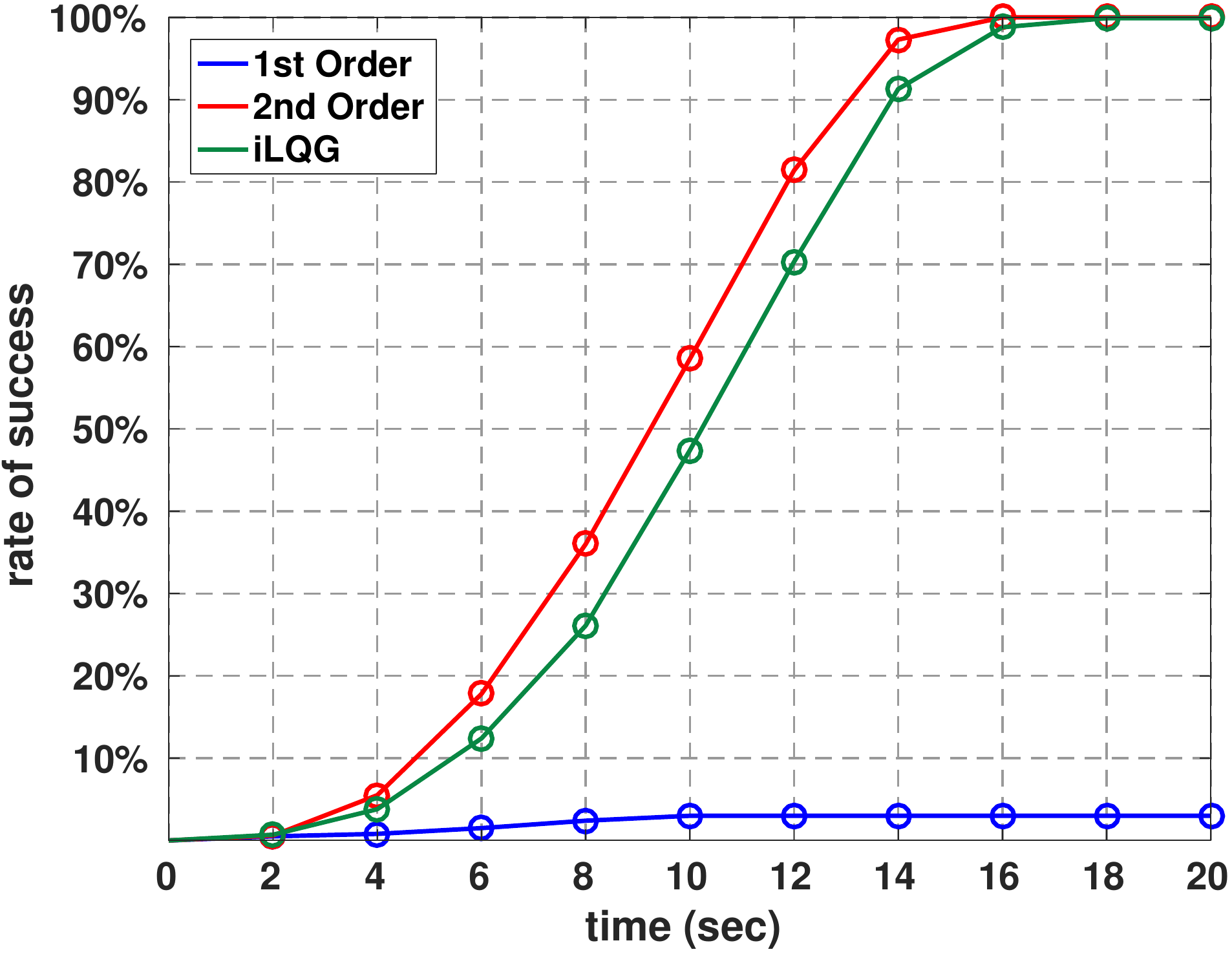}
		\caption{Convergence success rates of first- \eqref{optimalu} and second-order \eqref{optcon} needle variation controls for the kinematic differential drive model. Simulation runs: 1000.}\label{DifDriveMC_2nd}
	\end{figure}
	
	Fig. \ref{DifDriveMC_2nd} shows a Monte Carlo simulation that compares convergence success using first- and second-order needle variations controls shown in \eqref{optimalu} and \eqref{optcon}, respectively, and iLQG. We sampled over initial coordinates $x_0, y_0 \in [-1500, 1500]$~mm using a uniform distribution and keeping only samples for which the initial distance from the origin exceeded $L/5$; $\theta_0~=~0$ for all samples. Successful samples are defined by being within $L/5$ from the origin with an angle $\theta <\pi/12$ within 60 seconds using feedback sampling rate of 4 Hz. Results are generated using $Q~=~\text{diag}(10,10,1000)$, $P_1~=~\text{diag}(0,0,0)$, $T~=~0.5$~s, $R~=~\text{diag}(100,100)$ for \eqref{optimalu}, $R~=~\text{diag}(0.1,0.1)$ for \eqref{optcon}, $\gamma~=~-15$, $\lambda~=~0.1$ and saturation limits on the angular velocities of each wheel $\pm$150/36~mm/s for each control approach.\endnote{The metric on control effort is necessarily smaller for \eqref{optcon}, due to parameter $\lambda$. The parameter is chosen carefully to ensure that control solutions from \eqref{optcon} and \eqref{optimalu} are comparable in magnitude.} As shown in Fig.~\ref{DifDriveMC_2nd}, the system always converges to the target using second-order needle variation actions, matching the theory. 
	\begin{figure}[]
		\begin{subfigure}[b]{0.49\linewidth}
			\centering
			\includegraphics[width=4cm,height = 4cm]{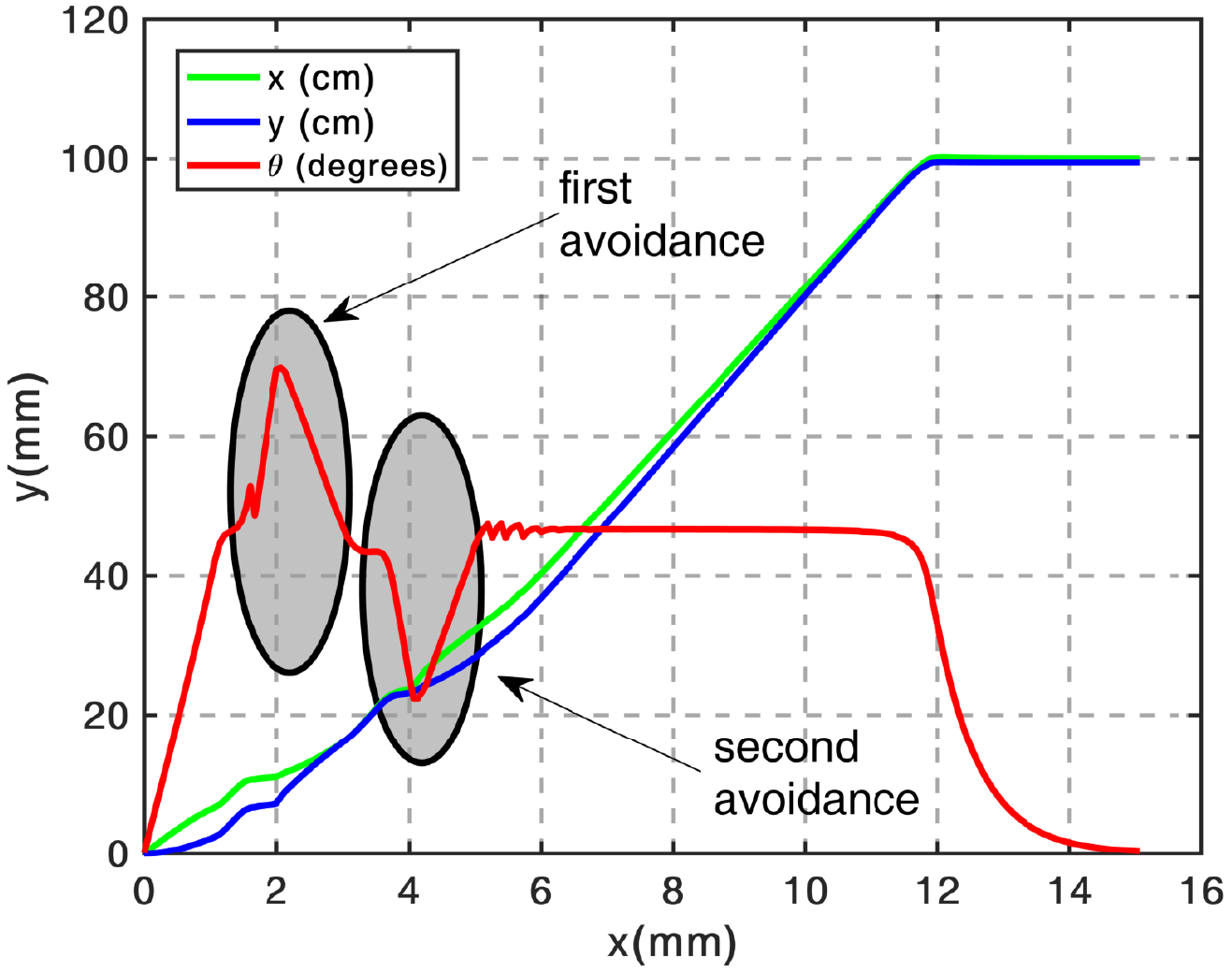} 
			\caption{} 
		\end{subfigure} 
		\hfill%
		\begin{subfigure}[b]{0.49\linewidth}
			\centering
			\includegraphics[width=4cm,height =4cm]{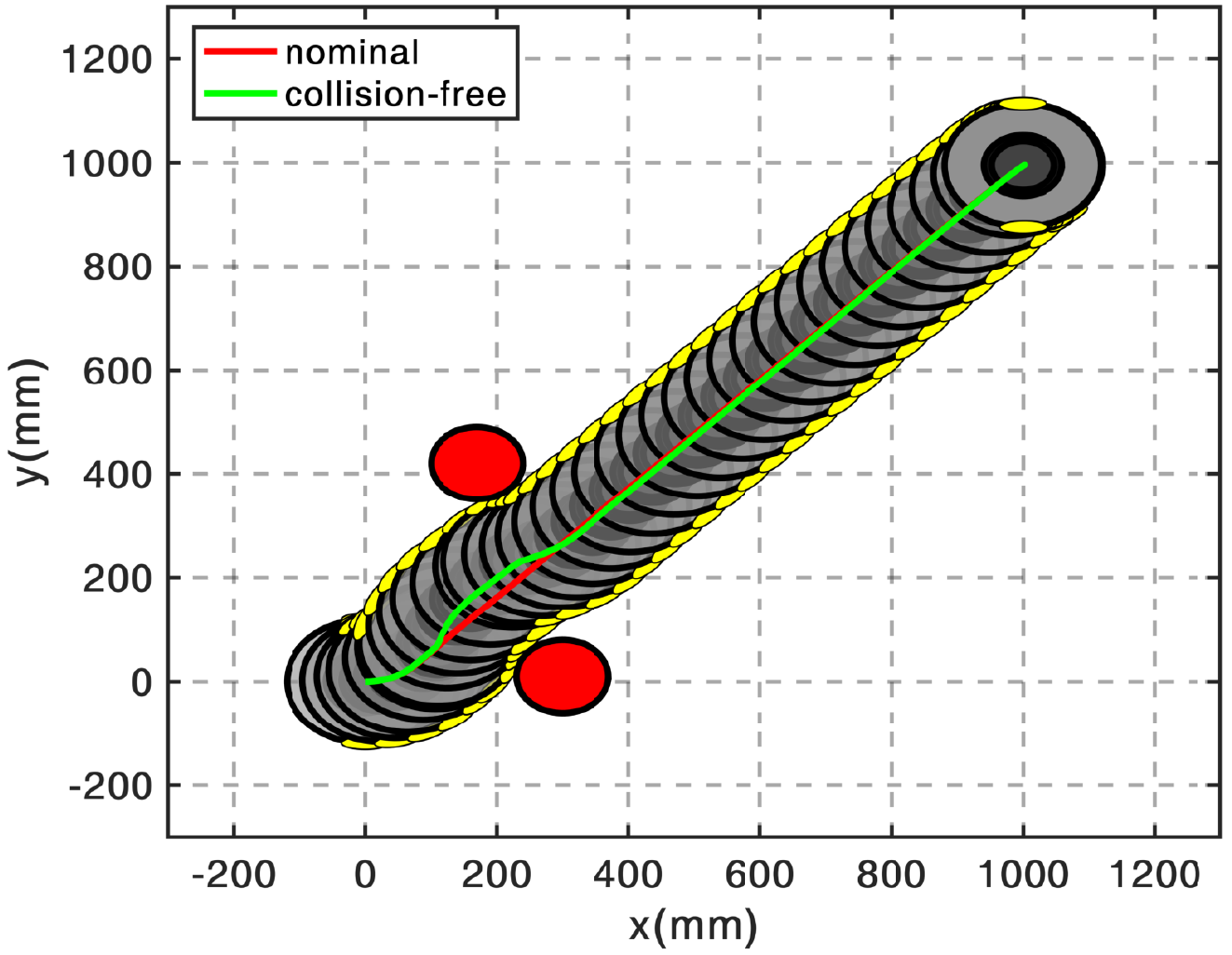} 
			\caption{} 
			\label{fig:: Ob_b} 
		\end{subfigure} 
		\caption{Differential drive using second-order needle variation actions in the presence of obstacles. Fig. \ref{fig:: Ob_b} shows the deviation from the nominal trajectory that is the solution to the no-obstacle task. The system performs two maneuvers to avoid each obstacle. These are evident in the angle deviation (compare to Fig. \ref{fig7:c}).}
		\label{fig:: DifDrive_Diag_wO}
	\end{figure}
	\subsubsection{Convergence with Obstacles}
	\begin{figure*}[h!]
		\centering
		\begin{subfigure}{0.5\columnwidth}
			\includegraphics[width=\linewidth, height = \linewidth]{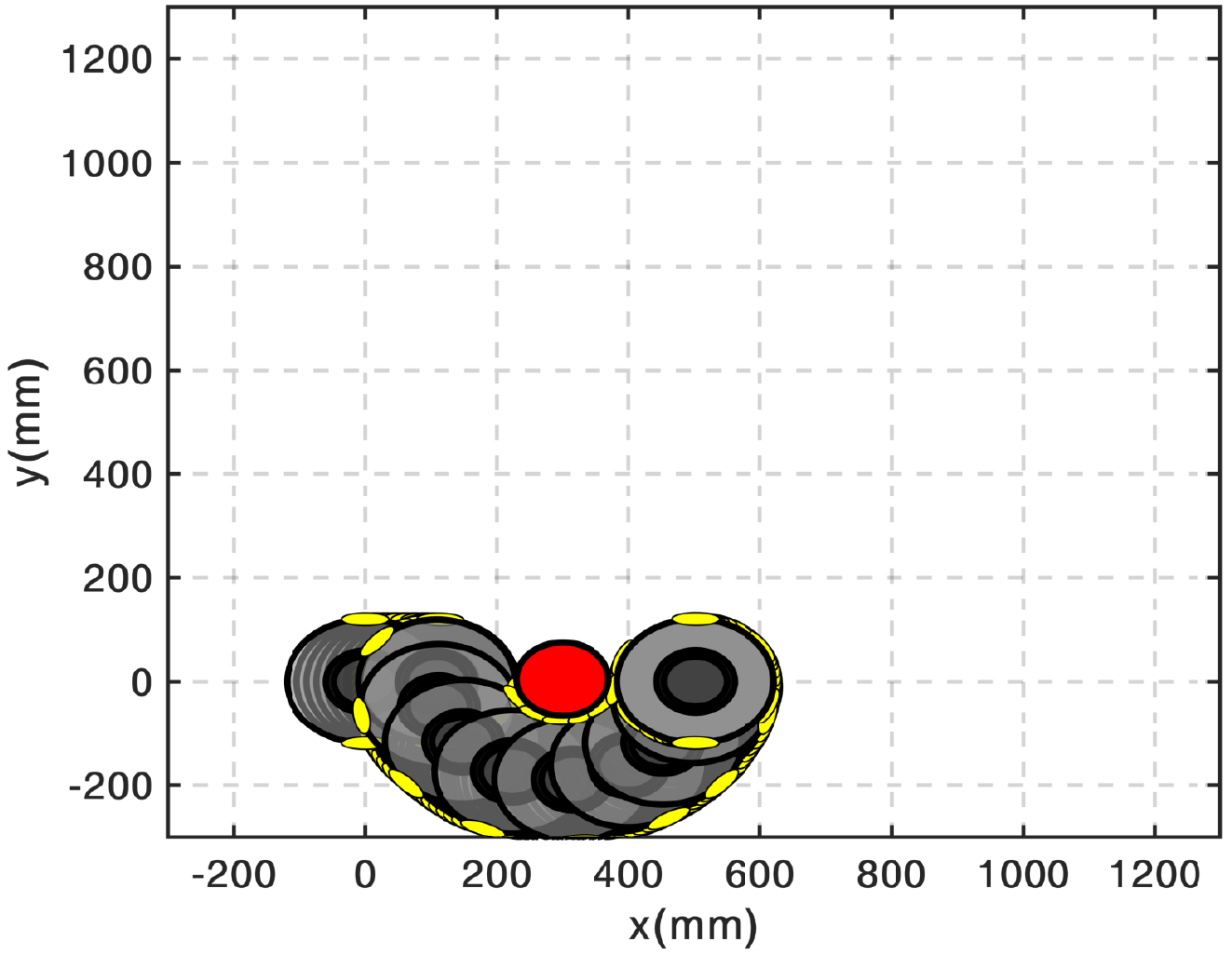}%
			\caption{}%
			\label{fig: 1}%
		\end{subfigure}\hfill%
		\begin{subfigure}{0.5\columnwidth}
			\includegraphics[width=\linewidth, height = \linewidth]{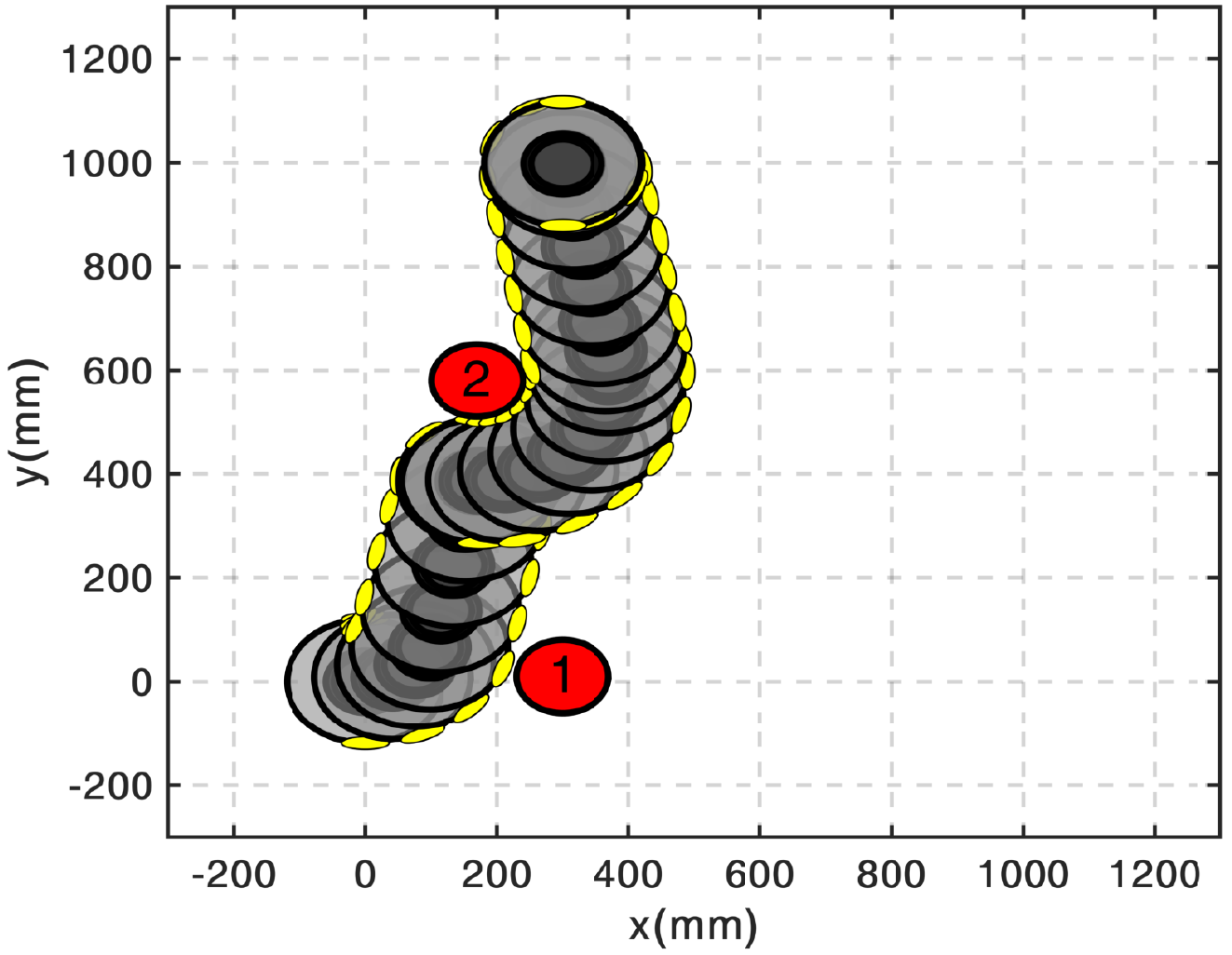}%
			\caption{}%
			\label{fig: 2r}%
		\end{subfigure}\hfill%
		\begin{subfigure}{0.5\columnwidth}
			\includegraphics[width=\linewidth, height = \linewidth]{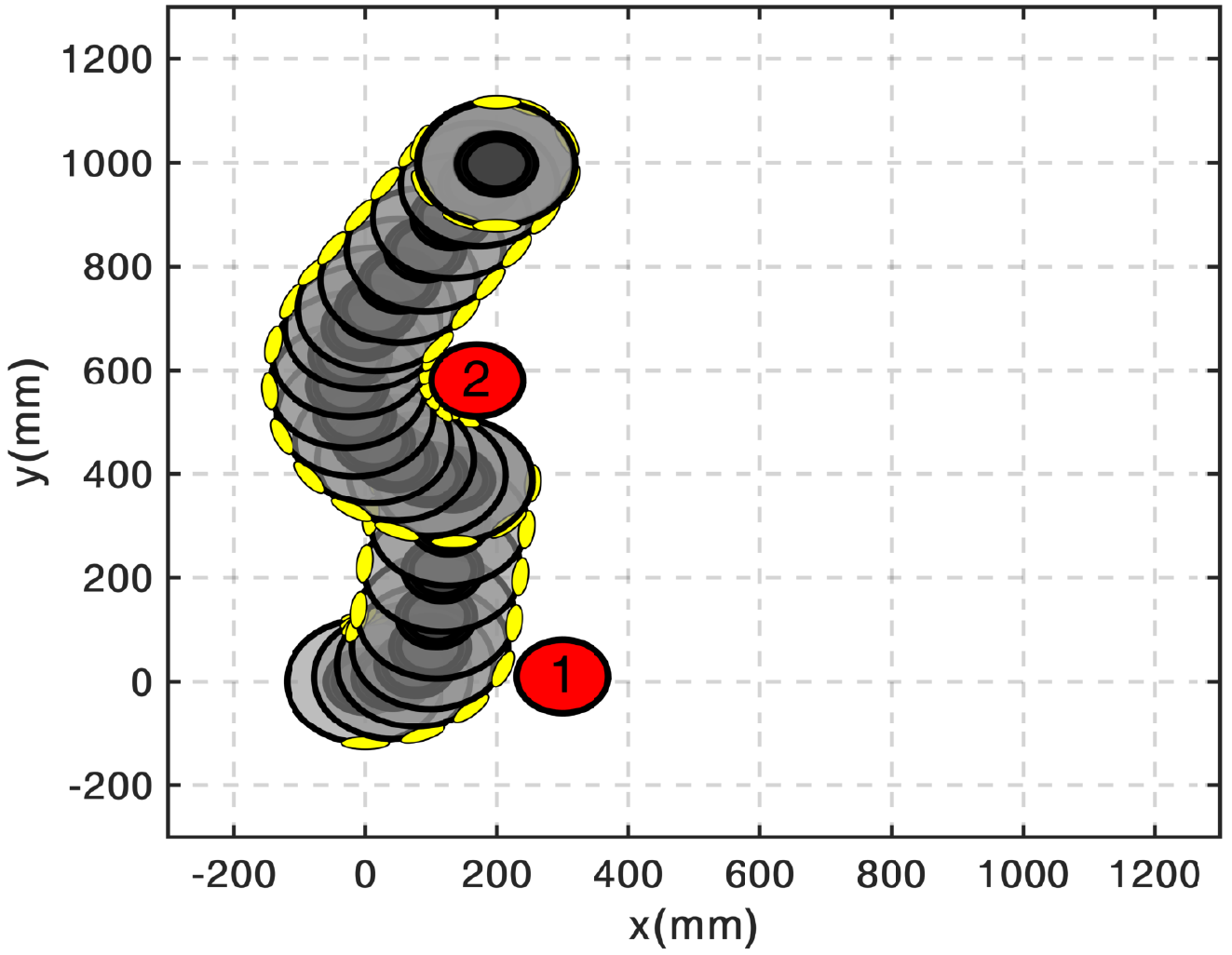}%
			\caption{}%
			\label{fig: 2l}%
		\end{subfigure}\hfill%
		\begin{subfigure}{0.5\columnwidth}
			\includegraphics[width=\linewidth, height = \linewidth]{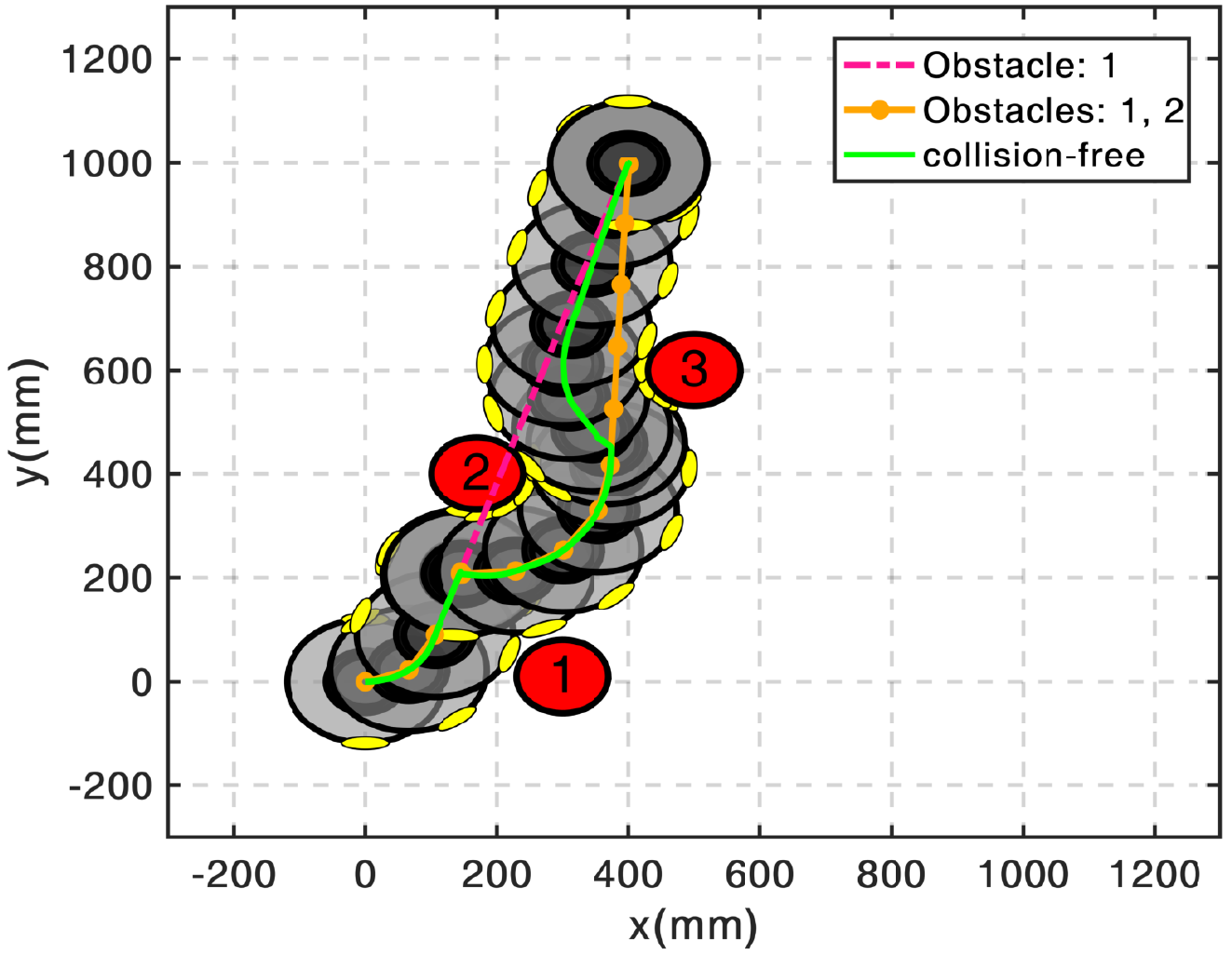}%
			\caption{}%
			\label{fig: 3}%
		\end{subfigure}\hfill%
		\caption{Trajectories of the differential drive in the presence of obstacles. Fig. \ref{fig: 3} compares the solution to the trajectories generated when considering only a) obstacle 1, and b) obstacles 1 and 2, both of which collide with the obstacles. Simulations run in real time in \MATLAB.}
		\label{fig: Obstacles}
	\end{figure*}
	
	Next, we illustrate the performance of the algorithm in the presence of obstacles. In all simulations, obstacles are considered in the objective in the form of a penalty function. In Fig. \ref{fig:: DifDrive_Diag_wO}, we test the system in the same task as Fig. \ref{Differential Drive} in the presence of two obstacles, indicated with red spheres. In comparison with Fig. \ref{fig7:e}, it is worth noting the two angle peaks, corresponding to each obstacle. After passing the obstacles, the system recovers the same angle profile. 
	
	Fig. \ref{fig: Obstacles} shows more complicated maneuvers using controls from \eqref{optcon}. The controller, without relying on a motion planner, is able to generate collision-free trajectories and safely converge to the target in all cases. These simulations also demonstrate another aspect of Algorithm \ref{algorithm}. The differential drive always drives up to an obstacle and then narrowly maneuvers around, instead of preparing a turn earlier on. This behavior is to be expected of needle variation actions that instantly reduce the cost.
	
	We next use the more complicated scenario of Fig. \ref{fig: 3} to evaluate the second-order expansion of the objective, shown in \eqref{eq:: DJ}, across the state-space (see Fig. \ref{fig:: DJ}). The first- and second-order mode insertion gradients are computed based on the second-order needle variation controls from \eqref{optcon}. States are sampled in the space for $x$ and $y$ in increments of 5~mm, with $\theta=0$ everywhere. These results correspond to $\lambda = 0.001$. The horizontal discontinuity that appears around $y=750$~mm is believed to be due to the effect of the penalty functions. As Fig. \ref{fig:: DJ} indicates, the change in cost is always negative, verifying Theorem \ref{Theorem}, even in the presence of obstacles. 
	\begin{figure}[]
		\begin{subfigure}[b]{0.492\linewidth}
			\centering
			\includegraphics[width=\linewidth]{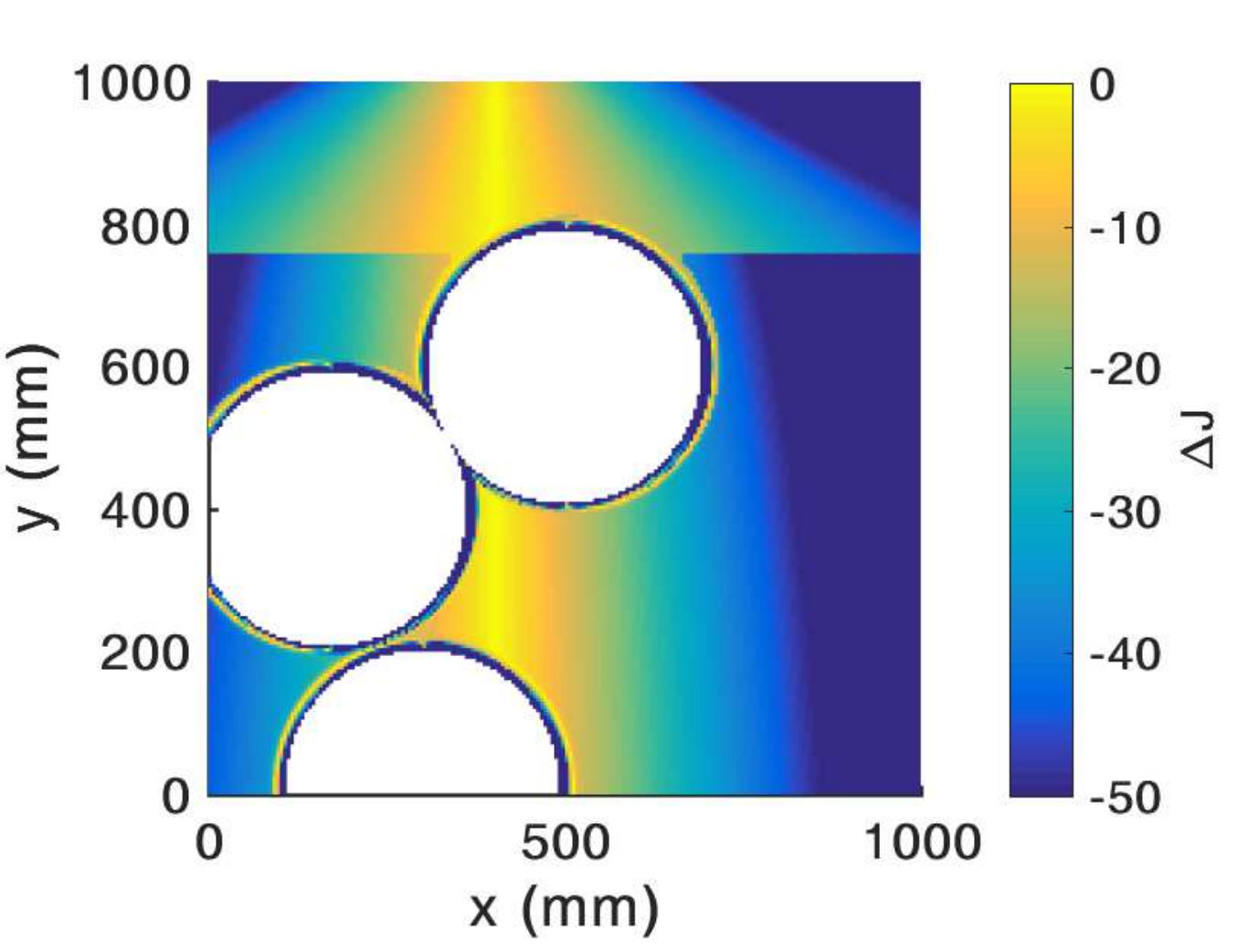} 
			\caption{} 
			\label{fig:: DJa}
		\end{subfigure} 
		\hfill%
		\begin{subfigure}[b]{0.492\linewidth}		\label{fig:: b}
			\centering
			\includegraphics[width=\linewidth]{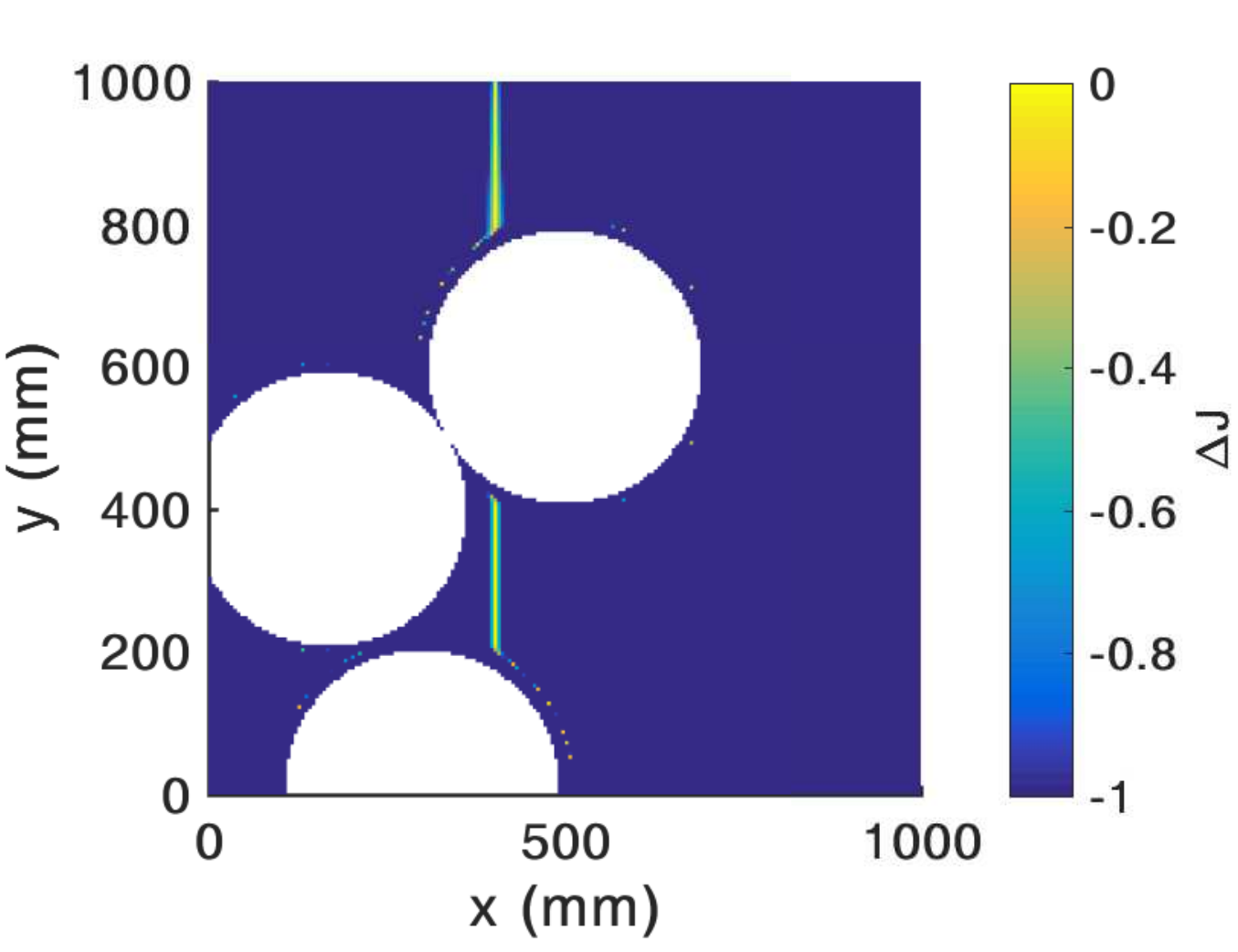} 
			\caption{} \label{fig: b} 
		\end{subfigure}
		\caption{Cost reduction $\Delta J$, modeled after \eqref{eq:: DJ}, for sampled $x$ and $y$ in the presence of obstacles, given second-order needle variation controls. The first- and second-order mode insertion gradients are evaluated with the controls from \eqref{optcon}. Figures \ref{fig:: DJa} and \ref{fig: b} are identical, but shown over a different range to illustrate that even when looking at small variations of the first-order mode insertion gradient, the second-order method is reliably negative. The bright vertical line in Fig. \ref{fig: b} is vertically aligned with the target located at [$400$~mm, $1000$~mm], where first-order solutions are singular. No data are sampled inside the white circles, as these indicate the infeasible occupied region.} \label{fig:: DJ}
	\end{figure}
	
	We further use a Monte Carlo simulation of 500 trials on the initial conditions to test convergence success (Fig. \ref{MC_Time}). We sample initial conditions $[x, y]$ from a uniform distribution in [-200~mm, 1000~mm] $\times$ [-400~mm, 800~mm], where $\theta_o = 0$ in all cases. All trials converged within 25 seconds. 
	\begin{figure}[]
		\begin{subfigure}[b]{0.48\linewidth}
			\centering
			\includegraphics[width=\linewidth,height = 0.15\textheight]{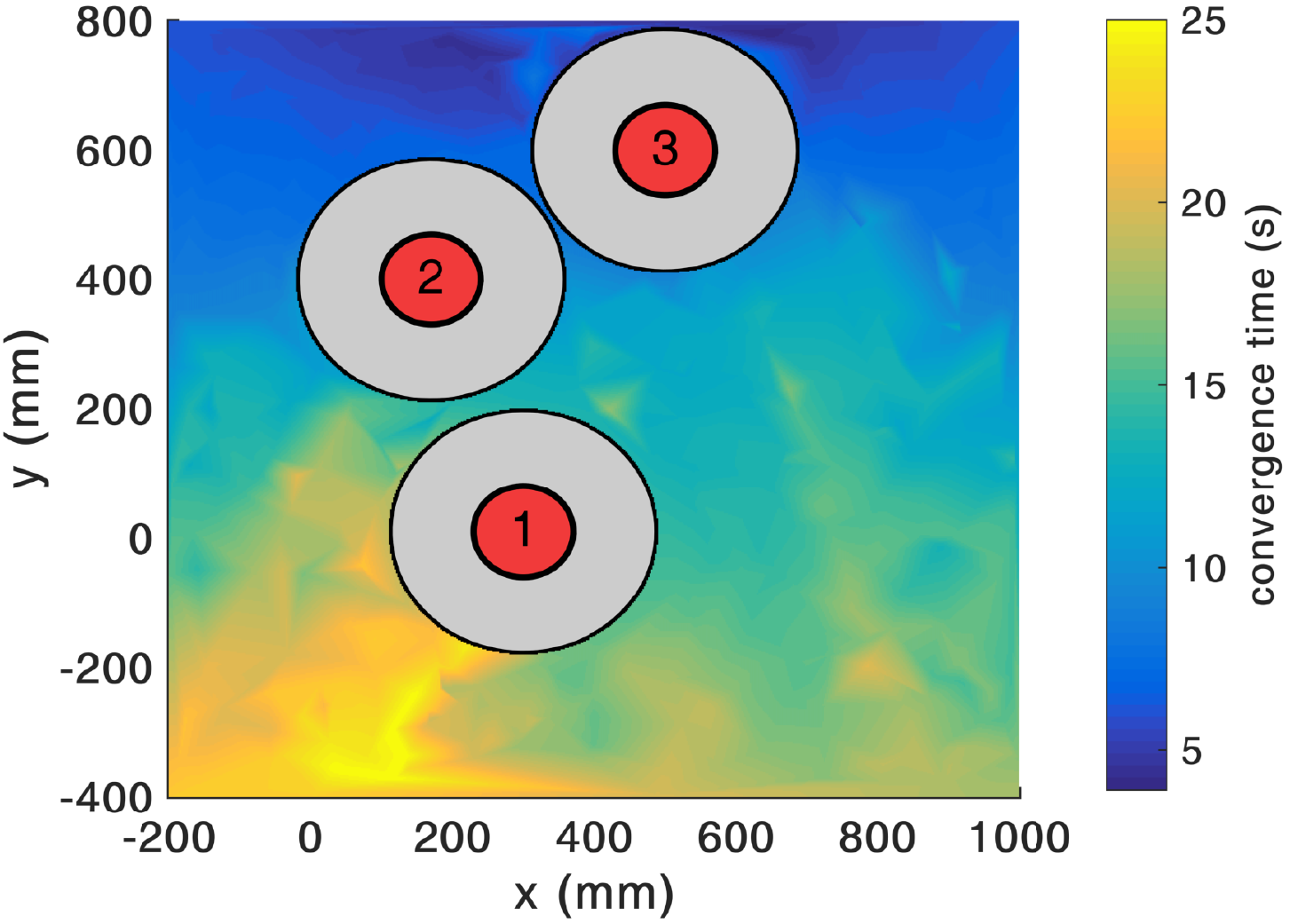} 
			\caption{}\label{Fig: 7a}
		\end{subfigure} 
		\hfill%
		\begin{subfigure}[b]{0.48\linewidth}
			\centering
			\includegraphics[width=\linewidth,height = 0.15\textheight]{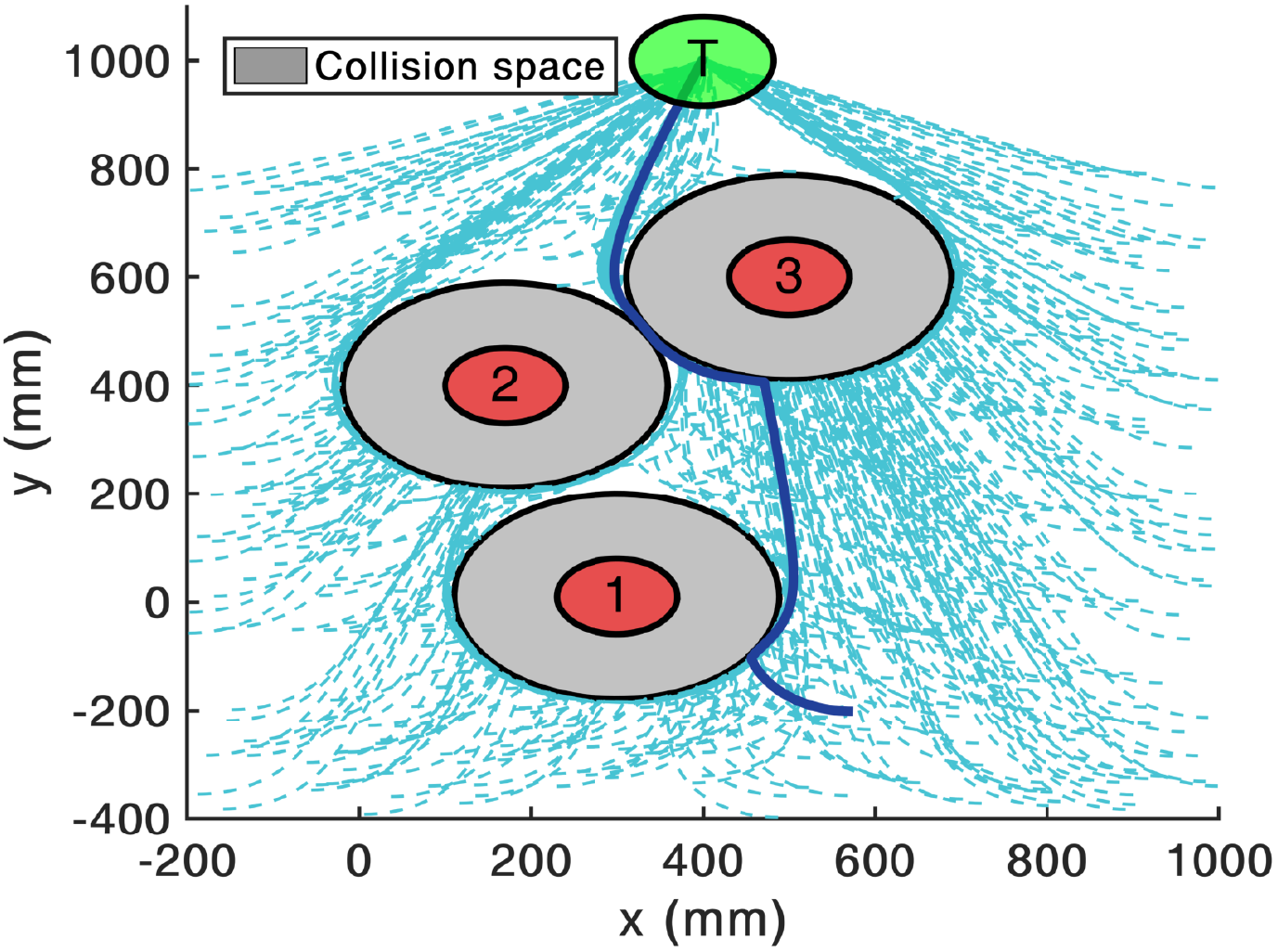} 
			\caption{}\label{Fig: 7b}
		\end{subfigure} 
		\caption{Performance of second-order needle variation actions in the presence of static obstacles. The controller is able to converge to the target for all 500 trials and avoid collisions. Fig. \ref{Fig: 7a} is an interpolated heat map that indicates the time to convergence as a function of initial position; Fig. \ref{Fig: 7b} shows the trajectories followed by the center of mass of the agent. The gray area indicates the collision space, taking into account the width of the differential drive. (the simulation runs in real time in \MATLAB). For visualization, watch Extension 1.}\label{MC_Time}
	\end{figure}
	
	Last, we test the differential drive in the presence of moving obstacles (Fig. \ref{fig: Moving}). The controller is again able to avoid collision and converge to the target, without relying on additional motion planning techniques. The feedback rate used is 20~Hz and the trajectory of the obstacles is known to the agent throughout the time horizon. In these simulations, $T = 0.3$~s.
	\begin{figure}[]
		\centering
		\includegraphics[width=\linewidth, height = 0.5\linewidth]{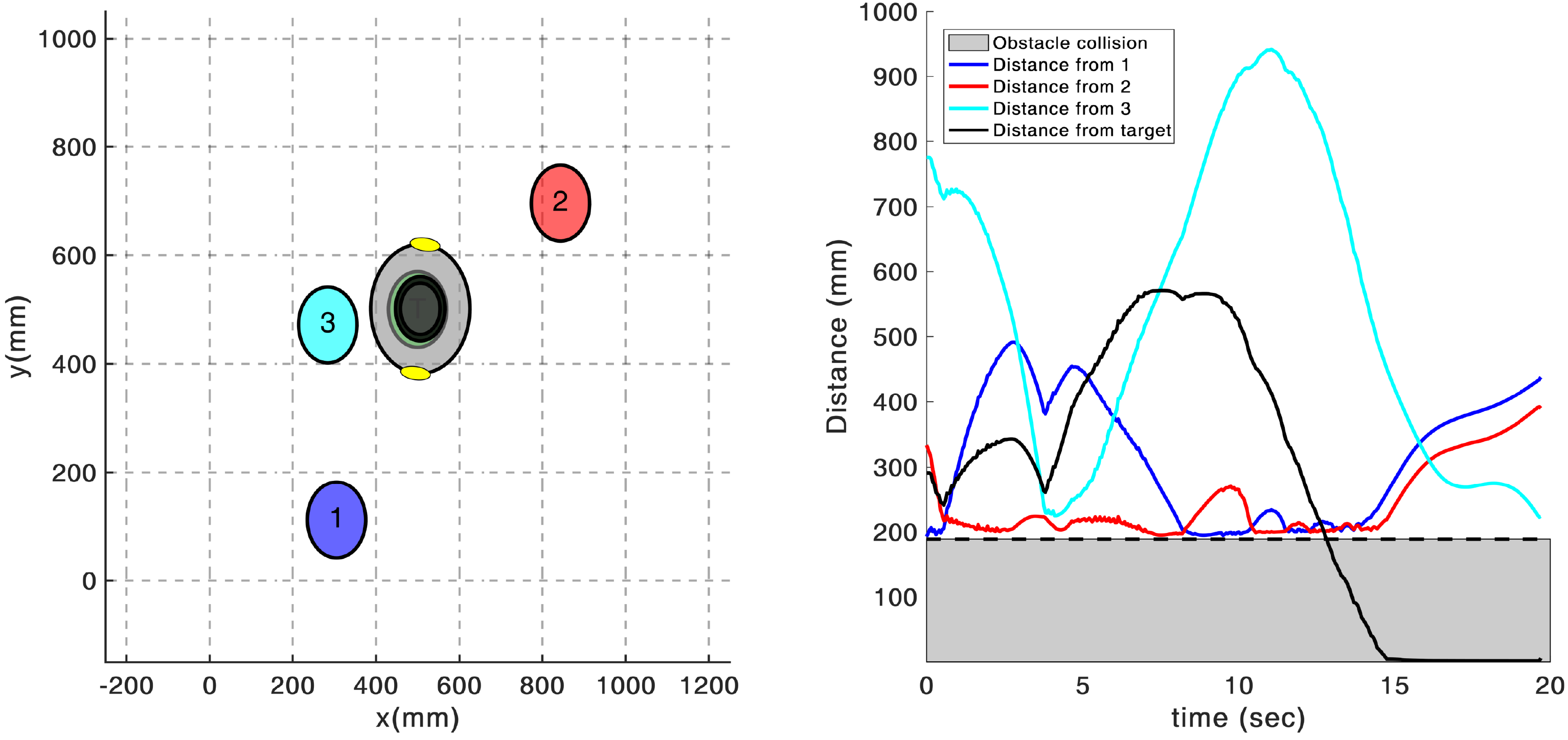}
		\caption{Performance of second-order needle variation actions in the presence of three moving obstacles. The left figure shows a snapshot of the simulation; the right figure plots the distance of the agent from each object and the target, where the gray area indicates the threshold minimum distance to avoid collision with the obstacles. The controller converges to the target in a collision-free manner (the simulation runs in real time in \MATLAB). For visualization, watch Extension 2.}\label{fig: Moving}
	\end{figure}
	\subsection{3D Kinematic Rigid Body}
	The underactuated kinematic rigid body is a three dimensional example of a system that is controllable with first-order Lie brackets. To avoid singularities in the state space, the orientation of the system is expressed in quaternions \citep{titterton2004strapdown, kuipers1999quaternions}. The states are $s~=~[x, y, z, q_0, q_1, q_2, q_3]$, where $b~=~[x, y, z]$ are the world-frame coordinates and $q~=~[q_0, q_1, q_2, q_3]$ are unit quaternions. Dynamics $f~=~[\dot{b},\dot{q}]^T$ are given by 
	
	\begin{gather}
	\dot{b}=R_qv, \label{dotb}
	\end{gather}
	\begin{gather}
	\dot{q}=\frac{1}{2}\begin{bmatrix} -q_1 & -q_2 & -q_3 \\ ~~q_0& -q_3& ~~q_2 \\ ~~q_3& ~~q_0& -q_1\\ -q_2& ~~q_1& ~~q_0 \end{bmatrix}\omega, \label{dotq}
	\end{gather}
	where $v$ and $\omega$ are the body frame linear and angular velocities, respectively \citep{da2015benchmark}. The rotation matrix for quaternions is
	\begin{gather*}
	R_q=\begin{bmatrix} q_0^2 + q_1^2 - q_2^2 - q_3^2& 2(q_1q_2 - q_0q_3)& 2(q_1q_3+q_0q_2) \\
	2(q_1q_2 + q_0q_3)&  q_0^2 - q_1^2+q_2^2 - q_3^2& 2(q_2q_3 - q_0q_1)\\
	2(q_1q_3 - q_0q_2)&  2(q_2q_3 + q_0q_1)&  q_0^2-q1^2 -q_2^2+ q_3^2\end{bmatrix}.
	\end{gather*}
	
	The system is kinematic: $v~=~F$ and $\omega~=~T$, where $F~=~(F_1, F_2, F_3)$ and $T~=~(T_1, T_2, T_3)$ describe respectively the surge, sway, and heave input forces, and the roll, pitch, and yaw input torques. We render the rigid body underactuated by removing the sway and yaw control authorities ($F_2~=~T_3~=~0$). 
	
	The four control vectors span a four-dimensional space. First-order Lie bracket terms add two more dimensions to span the state space ($\mathbb{R}^6$) (the fact that there are seven states in the model of the system is an artifact of the quaternion representation; it does not affect controllability).
	
	The vectors $h_1, h_2$, and $[h_2, h_3]$ span $\mathbb{R}^3$ associated with the world frame coordinates $\dot{x}, \dot{y}$, and $\dot{z}$. Similarly, vectors $h_3, h_4$, and $[h_4, h_3]$ span $\mathbb{R}^3$ associated with the orientation. Thereby, control vectors and their first-order Lie brackets span the state space and, from Theorem \ref{Theorem}, optimal actions shown in \eqref{optcon} will always reduce the cost function \eqref{cost}. 
	
	To verify this prediction, we present the convergence success of the system on 3D motion. Using Monte Carlo sampling with uniform distribution, initial locations are randomly generated such that $x_0, y_0, z_0 \in [-50, 50]$~cm keeping only samples for which the initial distance from the origin exceeded 6~cm. We regard as a convergence success each trial in which the rigid body is within 6~cm to the origin by the end of 60 seconds at any orientation. Results are generated at a sampling rate of 20~Hz using $Q~=~0$, $P_1~=~\text{diag}(100,200,100,0,0,0,0)$, $T~=~1.0$~s, $\gamma~=~-50000$, $\lambda~=~10^{-3}$, $R~=~10^{-6}\,\text{diag}(1,1,100,100)$ for \eqref{optcon}, and $R~=~\text{diag}(10,10,1000,1000)$ for controls in \eqref{optimalu}. Controls are saturated at $\pm 10$~cm/s for the linear velocities and $\pm 10$~rad/s for the angular ones. Using 280 simulations over 24 seconds, 80\% satisfy the success criterion within 12 seconds and 100\% of trajectories satisfy the success criterion within 20 seconds. None of the simulations converge for the first-order needle variation controls, because they cannot displace the system in the $\hat{y}$ direction.
	
	\subsection{Underactuated Dynamic 3D Fish}
	We represent the three dimensional rigid body with states $s~=~[b,~q,~v,~\omega]^T$, where $b~=~[x, y, z] $ are the world-frame coordinates, $q~=~[q_0,q_1, q_2, q_3]$ are the quaternions that describe the world-frame orientation, and $v~=~[v_x, v_y, v_z]$ and $\omega~=~[\omega_x, \omega_y, \omega_z]$ are the body-frame linear and angular velocities. 
	The rigid body dynamics are given by $\dot{b}$ and $\dot{q}$ shown in \eqref{dotb} and \eqref{dotq} and
	\begin{gather*}
	M \dot{v}~=~Mv \times \omega + F, \\
	J \dot{\omega}~=~J\omega \times \omega + T,
	\end{gather*}
	where the (experimentally determined) effective mass and moment of inertia of the rigid body are given by $M~=~\text{diag}(6.04, 17.31, 8.39)$~g and $J~=~\text{diag}(1.57, 27.78, 54.11)$~g$\cdot$cm$^2$, respectively. This example is inspired by work in \citet{mamakoukas2016,MacIver} and the parameters used for the effective mass and moment of inertia of a rigid body correspond to measurements of a fish. The control inputs are $F_2~=~T_3~=~0$ and $F_3\ge0$.
	
	The control vectors only span a four-dimensional space and, since they are state-independent, their Lie brackets are zero vectors. However, the Lie brackets containing the drift vector field $g$ (that also appear in the MIH expression) add from one to four (depending on the states) independent vectors such that control solutions in \eqref{optcon} guarantee decrease of the cost function \eqref{cost} for a wider set of states than controls in \eqref{optimalu}. 
	\begin{figure}[]
		\centering
		\includegraphics[width=0.5\linewidth, height = 0.15\textheight]{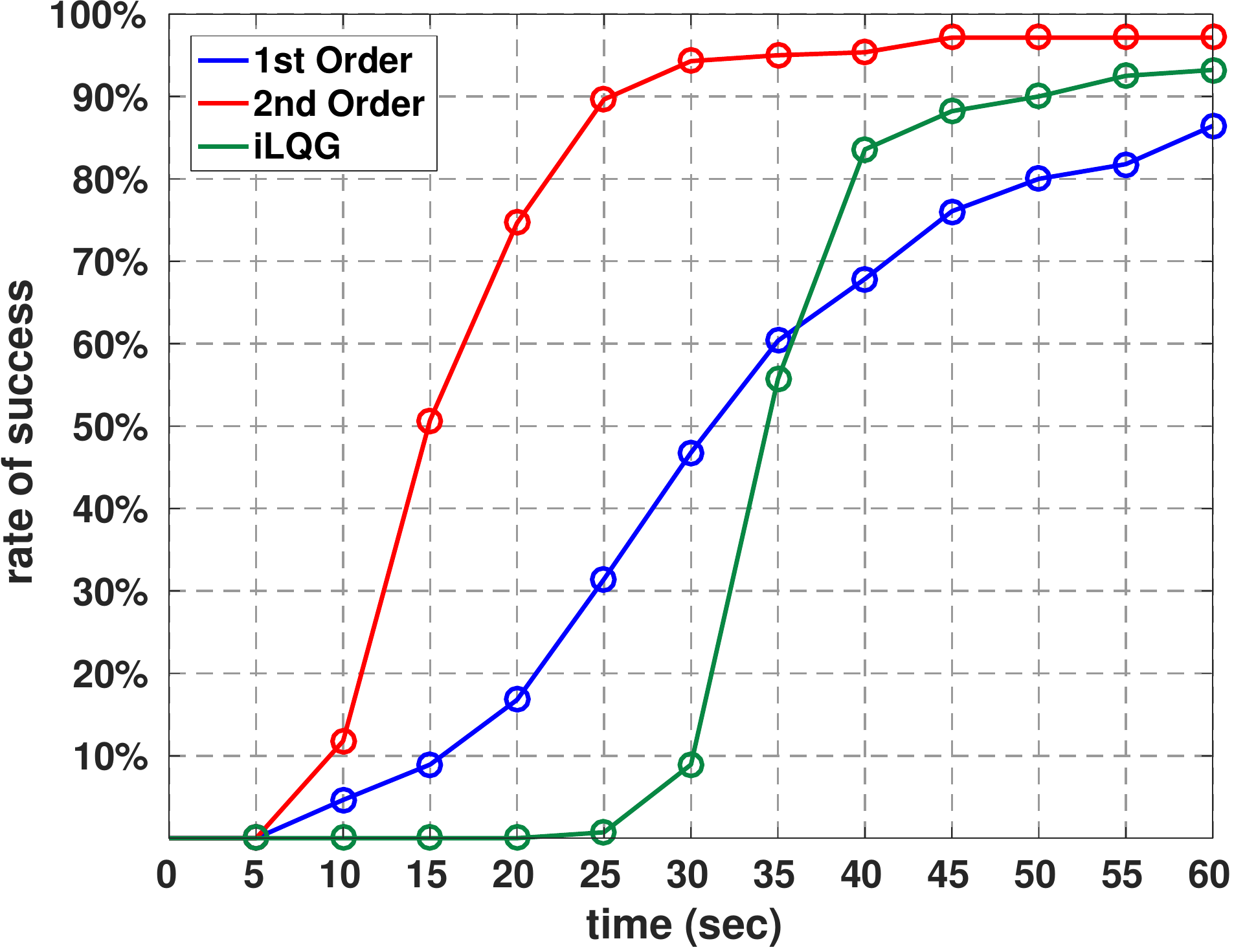}
		\caption{Convergence success rates of first- and second-order needle variation controls---\eqref{optimalu} and \eqref{optcon}, respectively---and iLQG for the underactuated \textit{dynamic} vehicle model. Simulation runs: 280} 
		\label{DynMC}
	\end{figure}
	\begin{figure}%
		\centering
		\includegraphics[width=0.98\linewidth,height = 0.22\textheight]{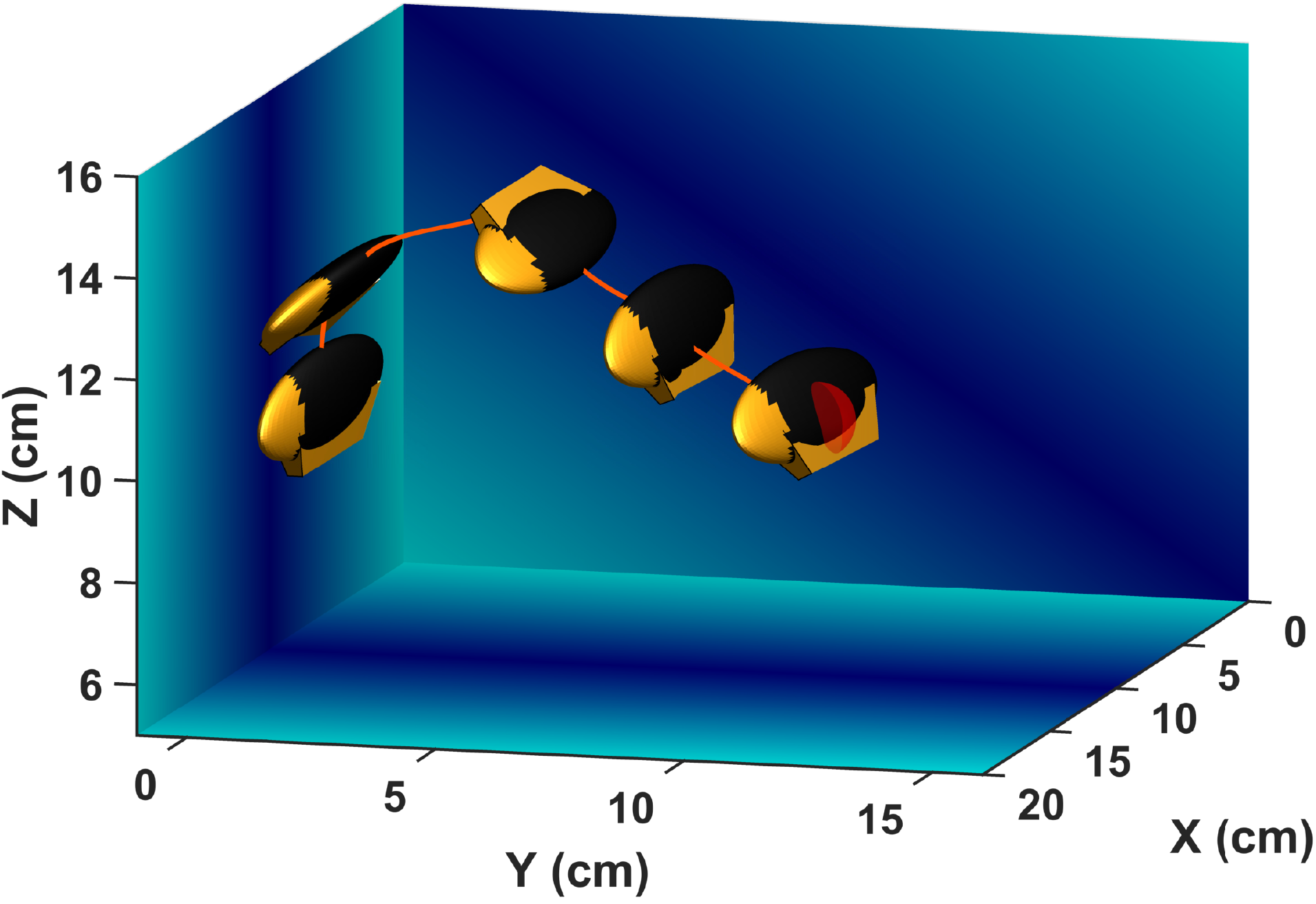}%
		\caption{Snapshots of a parallel displacement maneuver using an underactuated dynamic vehicle model with second-order controls given by \eqref{optcon}; first-order solutions \eqref{optimalu} are singular throughout the simulation. Animation of these results is available in Extension 3.}\label{SidewayMovement}%
	\end{figure}
	
	\begin{figure*}%
		\centering
		\begin{subfigure}{1.2\columnwidth}
			\includegraphics[width=0.9\columnwidth,height = 0.23\textheight]{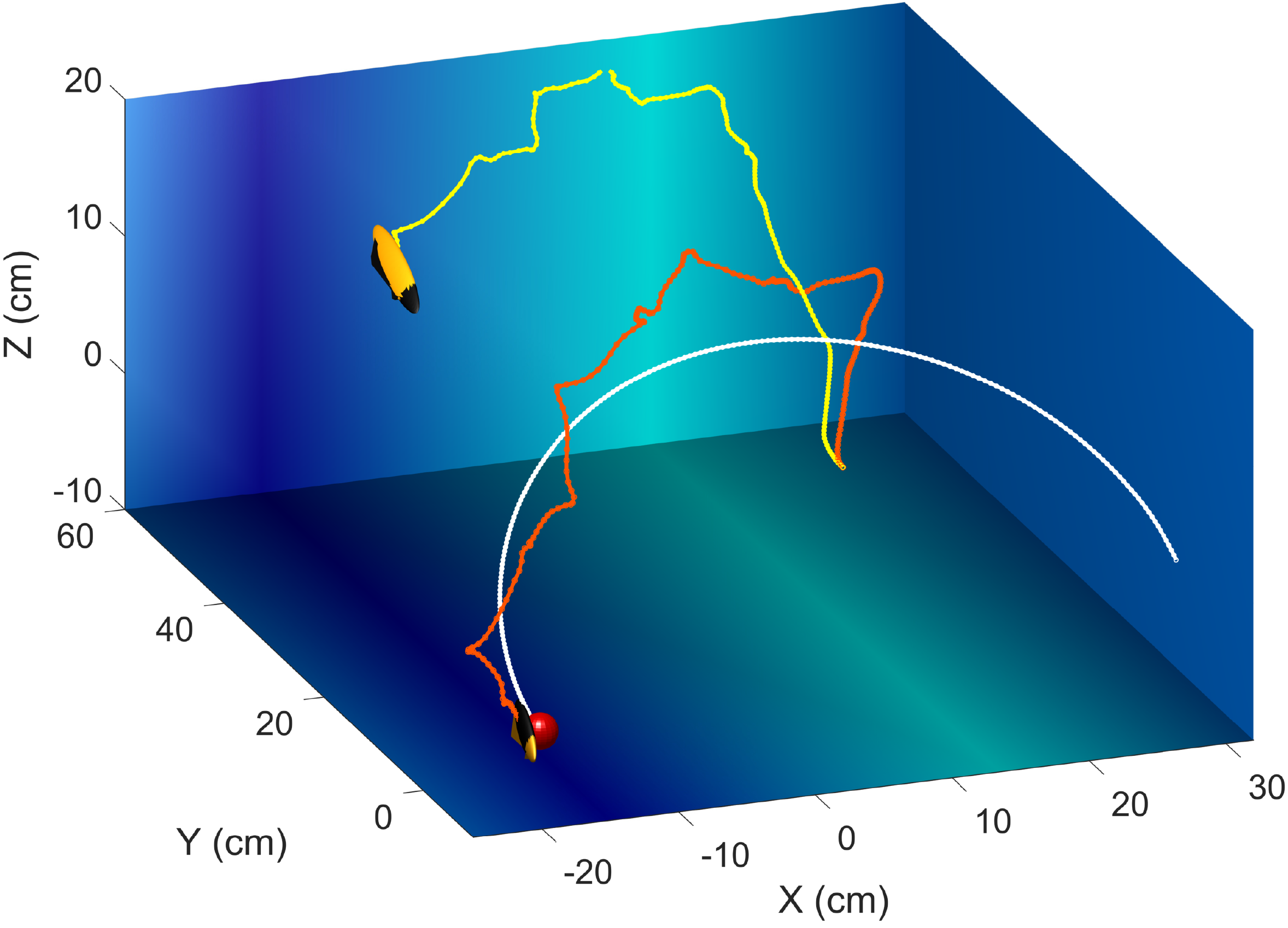}%
			\caption{}%
			\label{TrajTrack_drift}%
		\end{subfigure}\hfill%
		\begin{subfigure}{0.75\columnwidth}
			\includegraphics[width=0.9\columnwidth,height = 0.20\textheight]{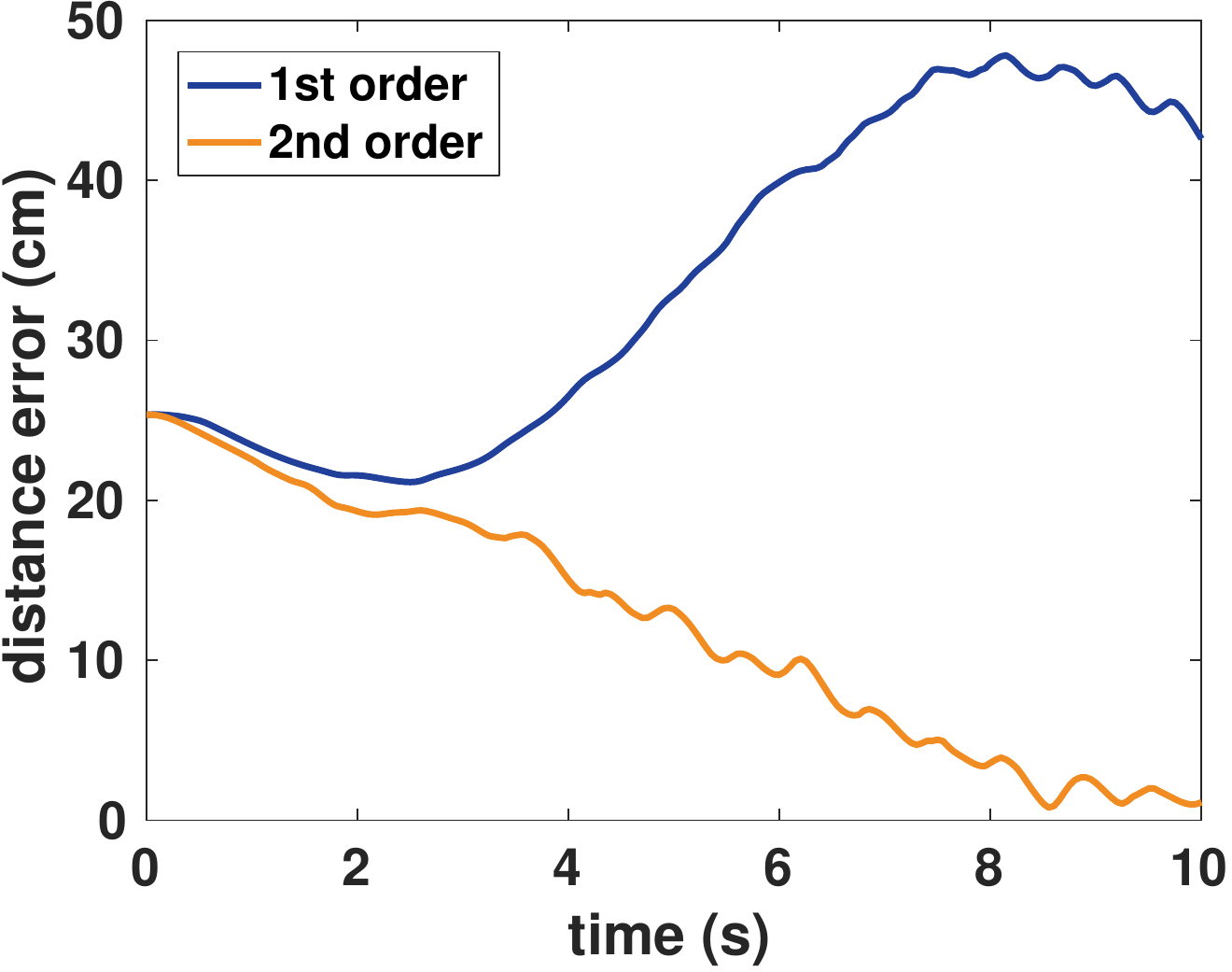}%
			\caption{}%
			\label{subfig:: errortrace}%
		\end{subfigure}\hfill%
		\caption{Tracking performance of the same system in the presence of +10~cm/s $\hat{y}$ fluid drift. The yellow system corresponds to first-order needle variation actions; the red one to second order. The target trajectory (red ball) is indicated with white traces over a 10-second simulation. Fig. \ref{subfig:: errortrace} shows the error distance as a function of time, clearly demonstrating the advantage of the second-order approach. Animation of these results is available in Extension 4.}
	\end{figure*}
	Simulation results based on Monte Carlo sampling are shown in Fig.~\ref{DynMC}. Initial coordinates $x_0, y_0, z_0$ are generated using a uniform distribution in $[-100, 100]$~cm, discarding samples for which the initial distance to the origin is less than 15~cm. Successful trials are the ones for which, within a simulation window of 60 seconds, the system approached within 5~cm to the origin (at any orientation) and whose magnitude of the linear velocities is, at the same time, less than 5~cm/s. Results are generated at a sampling rate of 20~Hz using \medmuskip=0mu
	\thinmuskip=0mu
	\thickmuskip=0mu$T~=~1.5$~s, $P_1~=~0$, $Q~=~\frac{1}{200}\text{diag}(10^3,10^3,10^3,0,0,0,0, 1, 1, 1, 2\cdot10^3,10^3,10^3)$, $\gamma~=~-5$, $R~=~\text{diag}(10^3,10^3,10^6,10^6)$ for \eqref{optimalu}, $R~=~\frac{1}{2}\,\text{diag}(10^{-6},10^{-6}, 10^{-3},10^{-3})$ for \eqref{optcon}, and $\lambda~=~10^{-4}$.\medmuskip=4mu
	\thinmuskip=3mu
	\thickmuskip=5mu ~The same control saturations ($F_1\in[-1, 1]$\,mN, $F_3\in[0,1]$\,mN, $T_1\in[-0.1, 0.1]$\,$\mu$N$\cdot$m, and $T_2\in[-0.1, 0.1]$\,$\mu$N$\cdot$m) are used for all simulations of the dynamic 3D fish. As shown in Fig. \ref{DynMC}, controls computed using second-order needle variations converge faster than those based on first-order needle variations, and 97\% of trials converge within 60 seconds. 
	
	Both methods converge over time to the desired location; as the dynamic model of the rigid body tumbles around and its orientation changes, possible descent directions of the cost function \eqref{cost} change and the control is able to push the system to the target. Controls for the first-order needle variation case \eqref{optimalu} are singular for a wider set of states than second-order needle variation controls \eqref{optcon} and, for this reason, they benefit more from tumbling.
	In a 3D parallel maneuver task, only second-order variation controls \eqref{optcon} manage to provide control solutions through successive heave and roll inputs, whereas controls based on first-order sensitivities \eqref{optimalu} fail (see Fig.~\ref{SidewayMovement}). 
	
	\indent As controls in \eqref{optcon} are non-singular for a wider subset of the configuration state space than the first-order solutions in \eqref{optimalu}, they will provide more actions over a period of time and keep the system closer to a time-varying target. Fig. \ref{TrajTrack_drift} demonstrates the superior trajectory tracking behavior of controls based on \eqref{optcon} in the presence of +10~cm/s $\hat{y}$ fluid drift. The trajectory of the target is given by $[x, y, z]$=\medmuskip=0mu
	\thinmuskip=0mu
	\thickmuskip=0mu[cos($\frac{\text{3t}}{10}$)\,(20+10cos($\frac{\text{t}}{5}$)), sin($\frac{\text{3t}}{10}$)\,(20\,+\,10\,cos($\frac{\text{t}}{5}$)), 10\,sin($\frac{\text{2t}}{5}$)], with $T=2$~s, $\lambda=0.01$, $Q~=~\text{diag}(10,10,10,0,0,0,0, 0, 0, 0, 1,1,0.1)$, $\gamma=-50000$, $P_1=\text{diag}(10,10,10,0,0,0,0, 0, 0, 0, 0, 0, 0)$, $R=\text{diag}(10^3,10^3,10^6,10^6)$ for \eqref{optimalu}, and $R=\text{diag}(10,10, 10^4,10^4)$ for \eqref{optcon}.\medmuskip=4mu
	\thinmuskip=3mu
	\thickmuskip=5mu~The simulation runs in real time using a C++ implementation on a laptop with Intel$^\circledR$ Core$^{\text{TM}}$ i5-6300HQ CPU @2.30GHz and 8GB RAM.
	
	The drift is known for both first- and second-order systems and accounted for in their dynamics in the form of $\dot{b}~=~\dot{b} + \dot{b}_\text{drift}$, where $\dot{b}_\text{drift}$ is a vector that points in the direction of the fluid flow. Simulation results demonstrate superior tracking of second-order needle variation controls that manage to stay with the target, whereas the system that corresponds to first-order needle variation controls is being drifted away by the flow. \\
	\indent We also tested convergence success of the +10~cm/s $\hat{y}$ drift case. Initial conditions $x,y,z$ are sampled uniformly from a $30$~cm radius from the origin, discarding samples for which the initial distance is less than $5$~cm. We consider samples to be successful if, during 60 seconds of simulation, they approached the origin within 5~cm. Out of 500 samples, controls based on second-order variations converged 91\% of the time (with average convergence time of 5.87~s), compared to 89\% for first-order actions (with average convergence time of 9.3~s). Simulation parameters are \medmuskip=0mu
	\thinmuskip=0mu
	\thickmuskip=0mu $T=1$~s, $\gamma=-25000$, $Q=10^{-3}\text{diag}(10,10,10,0,0,0,0, 1, 1, 1, 1,1,1)$, $P_1=\text{diag}(100,100,100,0,0,0,0, \frac{1}{2}, \frac{1}{2}, \frac{1}{2}, 0, 0, 0)$, $\lambda=10^{-4}$, $R=\text{diag}(0.1,0.1,10^4,10^4)$ for \eqref{optimalu}, and $R=\frac{1}{2}\text{diag}(10^{-5},10^{-5}, 1,1)$ for \eqref{optcon}.
	
	\section{Conclusion} 
	This paper presents a needle variation control synthesis method for nonlinearly controllable systems that can be expressed in control affine form. Control solutions provably exploit the nonlinear controllability of a system and, contrary to other nonlinear feedback schemes, have formal guarantees with respect to decreasing the objective. By optimally perturbing the system with needle actions, the proposed algorithm avoids the expensive iterative computation of controls over the entire horizon that other NMPC methods use and is able to run in real time for the systems considered here. 
	
	Simulation results on three underactuated systems compare first-order needle variation controls, second-order needle variation controls, and iLQG controls and demonstrate the superior convergence success rate of the proposed feedback synthesis. Because second-order needle variation actions are non-singular for a wider set of the state space than controls based on first-order sensitivity, they are also more suitable for time-evolving objectives, as demonstrated by the trajectory tracking examples in this paper. Second order needle variation controls are also calculated at little computational cost and preserve control effort. These traits, demonstrated in the simulation examples of this paper, render feedback synthesis based on second- and higher-order needle variation methods a promising alternative feedback scheme for underactuated and nonlinearly controllable systems. 
	
	In the future, we wish to generalize Theorems 1 and 2 to guarantee solutions for all controllable systems. Further, we are interested in showing that the second-order control responses \eqref{optcon} can be applied over the entire horizon, and not only as a needle action. This would create a second-order continuous feedback scheme like iLQG, but with the formal guarantees for controllable systems. To test the feasibility of the algorithm, we are planning on conducting underwater experiments with a fish-robot. 
	\theendnotes
	\section*{Acknowledgment}
	This work was supported by the Office of Naval Research under grant ONR N00014-14-1-0594 and by the National Science Foundation under grant NSF CMMI 1662233. Any opinions, findings, and conclusions or recommendations expressed here are those of the authors and do not necessarily reflect the views of the Office of Naval Research or the National Science Foundation.
\bibliographystyle{agsm}
\bibliography{references}
\appendix
\section{Index to multimedia Extensions}

The multimedia extensions to this article can be found online by following the hyperlinks from www.ijrr.org.
\begin{table}[h!]
	\centering
	\caption*{\textbf{Table of Multimedia Extensions}}
	\label{my-label}
	\resizebox{\columnwidth}{!}{%
		\begin{tabular}{lll}
			\hline
			Extension & Type & Description                 \\ \hline
			\href{https://www.dropbox.com/s/bql55gmqsh73gx1/StaticObstaclesConvergence.mp4?dl=0}{1}  & Video & Collision-free convergence \\ 
			& &in the presence of static obstacles \\ & &from random initial 		conditions \\
			\href{https://www.dropbox.com/s/80qi1hjetm864m0/MovingObstacles.mp4?dl=0}{2}  & Video & Collision-free convergence in the \\
			& &presence of moving obstacles     \\
			\href{https://www.dropbox.com/s/yk3q0b4xoytbum6/ParallelManeuver.mp4?dl=0}{3}    & Video & Parallel maneuver of 3D dynamic fish. \\ 
			\href{https://www.dropbox.com/s/0e1znvlw0qo7gdx/Tracking.mp4?dl=0}{4} & Video & Underactuated tracking of dynamic \\ & &3D Fish in the presence of drift.\\
			\hline
		\end{tabular}
	}
	
\end{table}
\section{Derivations and Proofs}
\subsection{Derivation of the Mode Insertion Hessian}
Consider switched systems that are defined by dynamics
\begin{gather}
\dot{x}(t) = f\big(x(t), \Lambda, t\big) = 
\begin{cases}
~f_1\big(x(t), t\big)\ \ ,& T_0 \le t < \tau_1 \\
~f_{2}\big(x(t), t\big)\ \ ,& \tau_1 \le t < \tau_1 + \lambda_1 \\
~f_1\big(x(t), t\big)\ \ ,& \tau_1 + \lambda_1 \le t < \tau_2 \\
~f_3\big(x(t), t\big)\ \ ,& \tau_2 \le t < \tau_2 + \lambda_2 \\
~f_1\big(x(t), t\big)\ \ ,& \tau_2 + \lambda_2 \le t < \tau_3 \\
~\vdotswithin{f_1\big(x(t), t\big)}\ \ \vdotswithin{,}&  \vdotswithin{\tau_1 \le t < \tau_1 }\\
~f_1 \big(x(t), t)\ \ ,& \tau_{L} + \lambda_{L} \le t < T_F
\end{cases}\\
\text{subject to: } x(T_0) = x_0 \notag,
\end{gather}
where $T_0$ is the initial time, $T_F$ is the final time, $x_0:\mathbb{R} \mapsto \mathbb{R}^{N}$ is the initial state, $L$ is the number of injected dynamics, $\tau = \{\tau_1, \tau_2, \cdots, \tau_L \} \in \mathbb{R}^{L}$ is a monotonically increasing set of switching times, $\Lambda = \{\lambda_1, \lambda_2, \cdots, \lambda_L\} \in \mathbb{R}^{L}$ is a set of control durations, $f_1 : \mathbb{R} \mapsto \mathbb{R}^{N}$ specify the default dynamics, and $f_i :\mathbb{R} \mapsto \mathbb{R}^{N}$ describe the $i^\text{th}$ injected dynamics. The switching times are assumed to be fixed. 

We note that, while the system dynamics $f$ depend on the set of control durations $\Lambda$, the same is not true for the individual switch mode dynamics $f_i$. In addition, we refer to individual elements in the set $\Lambda$ as either $\Lambda_i$ or $\lambda_i$. We measure the performance of the system with the integral of the Lagrangian, $\ell(\cdot)$, and a terminal cost $m(\cdot)$, similar to \eqref{cost}.
\begin{equation}
J(\Lambda) = \int\limits_{T_0}^{T_F} \ell(x(t))\,\mathrm{d}t + m(x(T_F)).
\end{equation} 

We can re-write the dynamics using step functions, such that
\begin{equation*}
\begin{split}
f(x(t), \Lambda, t)\ =&\big[1(t-T_0) - 1(t-\tau_1^-)\big]f_1\big(x(t), t\big) \\
&+ \big[1(t-\tau_1^+) - 1\big(t-(\tau_1+\lambda_1)^-\big)\big]f_2\big(x(t), t\big)\\ 
&+\big[1\big(t-(\tau_1+\lambda_1)^+\big) - 1(t-\tau_2^-)\big]f_1\big(x(t), t\big) \\
&+ \ldots \\
&+ \big[1(t-(\tau_{L-1}+\lambda_{L-1})^+) - 1(t-\tau_L^-)\big]f_1\big(x(t), t\big) \\
&+ \big[1(t-\tau_L^+) - 1\big(t-(\tau_L+\lambda_L)^-\big)\big]f_L\big(x(t), t\big) \\
&+\big[1\big(t-(\tau_L+\lambda_L)^+\big) - 1(t-T_F)\big]f_1\big(x(t), t\big).
\end{split}
\end{equation*}
The superscripts $+$ and $-$ help avoid ambiguity at the switching times. We use directional derivatives to differentiate, where the slot derivative $D_i\mathcal{F}(\cdot, \cdot)$ is the partial derivative of a function $\mathcal{F}$ with respect to its $i^\text{th}$ argument. That is, $D_x\mathcal{F}$ indicates the derivative of $\mathcal{F}$ with respect to $x$ and is the same as $\frac{\partial \mathcal{F}}{\partial x}$. Further, $D_{i,j}\mathcal{F}(\cdot, \cdot)$ denotes the second partial of a function $\mathcal{F}$ with respect to its first and second arguments. The step-function form of the dynamics makes it straightforward to compute the partial derivatives $D_1f\big(x(t), \Lambda, t\big)$ and $D_2f\big(x(t), \Lambda, t\big)$. Specifically,
\begin{equation*}
\begin{split}
D_1f\big(x(t), \Lambda, t\big)\ =&\big[1(t-T_0) - 1(t-\tau_1^-)\big]D_1f_1\big(x(t), t\big) \\
&+ \big[1(t-\tau_1^+) - 1\big(t-(\tau_1+\lambda_1)^-\big)\big]D_1f_2\big(x(t), t\big)\\ 
&+\big[1\big(t-(\tau_1+\lambda_1)^+\big) - 1(t-\tau_2^-)\big]D_1f_1\big(x(t), t\big) \\
&+ \ldots \\
&+ \big[1(t-(\tau_{L-1}+\lambda_{L-1})^+) - 1(t-\tau_L^-)\big]D_1f_1\big(x(t), t\big) \\
&+ \big[1(t-\tau_L^+) - 1\big(t-(\tau_L+\lambda_L)^-\big)\big]D_1f_L\big(x(t), t\big) \\
&+\big[1\big(t-(\tau_L+\lambda_L)^+\big) - 1(t-T_F)\big]D_1f_1\big(x(t), t\big)
\end{split}
\end{equation*}
and
\begin{align}\label{B}
D_2f\big(x(t), \Lambda, t\big) =& \Big\{\delta\big(t-(\tau_k+\lambda_k)^-\big)f_k(x(t), t) \notag\\
&- \delta\big(t-(\tau_k+\lambda_k)^+\big)f_1(x(t), t)\Big\}_{k=1}^L,
\end{align}
where $\delta(\cdot)$ is the Dirac delta functions. Using variational calculus, 
\begin{align}\label{eq::DJvar}
DJ(\Lambda)\cdot\theta = \int\limits_{T_0}^{T_F} D\ell\big(x(r)\big) \cdot z(r) \mathrm{d}r + Dm\big(x(T_F)\big) \cdot z(T_F)
\end{align}
where $z(t) :\mathbb{R} \mapsto \mathbb{R}^{N \times 1}$ is the variation of $x(t)$ due to the variation, $\theta$, in $\Lambda$. Also, 

\begin{gather*}
\dot{z}(t) = \frac{\partial}{\partial t} \frac{\partial x(t)}{\partial \Lambda}= \frac{\partial}{\partial \Lambda}\frac{\partial x(t)}{\partial t} = \frac{\partial}{\partial \Lambda}\dot{x}(t) = \frac{\partial}{\partial \Lambda}f\big(x(t), \Lambda, t\big) \\
= D_1f\big(x(t), \Lambda, t\big) \cdot z(t) + D_2f\big(x(t), \Lambda, t\big)\cdot \theta,~~~~\\
\text{subject to: } z(0) = \frac{\partial}{\partial \Lambda}x(0) = 0.
\end{gather*}

Define $A(t) \triangleq D_1f\big(x(t), \Lambda, t\big)$ and $B(t) \triangleq D_2f\big(x(t), \Lambda, t\big)$. Therefore, $\dot{z}$ is 
\begin{gather*}
\dot{z}(t) = A(t) \cdot z(t) + B(t) \cdot \theta \\
\text{subject to: } z(0) = 0.
\end{gather*}
The above differential equation has the solution 
\begin{gather}\label{eq::zeta}
z(t) = \int\limits_{T_0}^t \Phi(t,r) B(r) \cdot \theta \mathrm{d}r,
\end{gather}
where $\Phi(t,r)$ is the state transition matrix corresponding to $A(t)$. Substituting $z(\cdot)$ in $DJ(\Lambda) \cdot \theta$,
\begin{align*}
DJ(\Lambda) \cdot \theta =& \int\limits_{T_0}^{T_F} D\ell\big(x(r)\big) \int\limits_{T_0}^r \Phi(r,s) B(s) \cdot \theta\,\mathrm{d}s\,\mathrm{d}r 
\\&
+ Dm\big(x(T_F)\big) \int\limits_{T_0}^{T_F} \Phi(T_F,s) B(s) \cdot \theta\,\mathrm{d}s.
\end{align*}
Switching the order of integration in the first-integral,
\begin{align*}
DJ(\Lambda) \cdot \theta =& \int\limits_{T_0}^{T_F}\limits\int\limits_s^{T_F} D\ell\big(x(r)\big) \Phi(r,s) B(s) \cdot \theta\,\mathrm{d}r\,\mathrm{d}s 
\\&
+ \int\limits_{T_0}^{T_F} Dm\big(x(T_F)\big)\Phi(T_F,s) B(s) \cdot \theta\,\mathrm{d}s 
\\
=& \int\limits_{T_0}^{T_F}\limits\lunderbrace{\Big[\int\limits_s^{T_F} D\ell\big(x(r)\big) \Phi(r,s) \,\mathrm{d}r}
\\&
\runderbrace{+ \int\limits_{T_0}^{T_F} Dm\big(x(T_F)\big)\Phi(T_F,s)\Big]}_{\rho(s)^T}B(s)\,\mathrm{d}s \cdot \theta.
\end{align*}
Then, 
\begin{align*}
\rho(t) = \Phi(T_F, t)^TDm(x(T_F))^T + \int\limits_t^{T_F}\Phi(r, t)^TD\ell(x(r))^T \mathrm{d}r,
\end{align*} 
where $\rho(t)$ is the solution to the backwards differential equation:
\begin{equation}\label{eq::Ch2rho}
\begin{gathered}
\dot{\rho}(t) = -D_1f(x(t), \Lambda, t)^T\rho - D\ell(x(t)) \\
\text{subject to: } \rho(T_F) = Dm(x(T_F))^T.
\end{gathered}
\end{equation}

To avoid confusion, it is important to explain the notation used in the remaining of the derivation. We use $\theta$ to represent first-order and $\eta$ to represent second-order perturbations to control durations $\Lambda$, respectively. We use subscripts to refer to the perturbation acting on a specific (single) duration. For example, $\theta_i$ indicates the perturbation that takes place with respect to the $i^\text{th}$ control duration, $\lambda_i$. We index the order of perturbations with a superscript, so that $\theta^j$ indicates the $j^\text{th}$ (in order) perturbation to the set of control durations $\Lambda$. Therefore, $\theta_1^2$ indicates the perturbation that acts on the first control duration $\lambda_1$ and that is associated with the second perturbation. 

We write
\begin{align*}
\frac{\partial}{\partial \Lambda}(DJ(\Lambda) \cdot \theta^1) =& 
\frac{\partial}{\partial \Lambda}(\int\limits_{T_0}^{T_N} D\ell(x(r)) \cdot z^1(r) \mathrm{d}r \\&+ Dm(x(T_F)) \cdot z^1(T_F)),
\end{align*}
and, using the product rule, we compute 
\begin{gather}\label{22}
\begin{align}
D^2J(\Lambda) \cdot (\theta^1, \theta^2) + DJ(\Lambda) \cdot \eta =& \int\limits_{T_0}^{T_F}D^2\ell(x(r)) \cdot \big(z^1(r), z^2(r)\big)
\notag\\&
+ D\ell(x(r)) \cdot \zeta(r) \mathrm{d}r 
\notag\\&
+D^2m\big(x(T_F)\big)\cdot(z^1(T_F), z^2(T_F)) 
\notag\\&
+ Dm\big(x(T_F)\cdot \zeta(T_F),
\end{align}\end{gather}
where $\theta^1$ and $\theta^2$ are two first-order variations of $\Lambda$, $\eta$ is a second-order variation of $\Lambda$ and $\zeta(t)$ is the second-order variation of $x(t)$. Parameter $\dot{\zeta}(t)$ is found by taking the second-order switching time derivative of $\dot{x}$(t):
\begin{align*}
\dot{\zeta}(t) =& \frac{\partial ^2}{\partial \Lambda^2}\dot{x}(t) = \frac{\partial}{\partial \Lambda} \dot{z}^1(t) \\
=& \frac{\partial}{\partial \Lambda}\big(D_1f(x(t), \Lambda, t\big) \cdot z^1(t) + D_2f\big(x(t), \Lambda, t)\cdot \theta^1\big) \notag
\end{align*}
such that
\begin{equation}\label{23}
\begin{align*}
\dot{\zeta}(t) =&A(t) \cdot \zeta(t) + B(t) \cdot \eta \notag \\
&+ 
\begingroup\def\arraystretch{0.7}\setlength\arraycolsep{3.5pt}\begin{pmatrix}{z^1(t)}^T &{\theta^1}^T
\end{pmatrix}\endgroup
\begingroup\def\arraystretch{0.7}\begin{pmatrix} 
D_1^2f\big(x(t), \Lambda, t\big) & D_{1,2}f(x(t), \Lambda, t\big) \\
D_{2,1} f\big(x(t), \Lambda, t \big) &D_2^2f\big(x(t), \Lambda,t\big)
\end{pmatrix}\endgroup
\begingroup\def\arraystretch{0.7}\begin{pmatrix}
z^2(t) \\ \theta^2
\end{pmatrix}\endgroup
\end{align*}\\
\text{subject to}: \zeta(0) = \frac{\partial ^2}{\partial \Lambda^2} x(0) = 0. \notag
\end{equation}
Define 
\begin{align*}
C(t) \triangleq \begingroup\def\arraystretch{0.7}\begin{pmatrix}
D_1^2f\big(x(t), \Lambda, t\big) & D_{1,2}f(x(t), \Lambda, t\big) \\
D_{2,1} f\big(x(t), \Lambda, t \big) &D_2^2f\big(x(t), \Lambda,t\big)
\end{pmatrix}\endgroup,
\end{align*}
and notice that $\dot{\zeta}(t)$ is linear with respect to $\zeta(t)$ and therefore $\dot{\zeta}(t)$ has solution
\begin{align*}
\zeta(t) = \int\limits_{T_0}^{t} \Phi(t,r) \Big[B(r) \cdot \eta + \big(z^1(r)^T {\theta^1}^T\big)C(r)
\begingroup\def\arraystretch{0.6}\small\begin{pmatrix}
z^2(r) \\ \theta^2
\end{pmatrix}\endgroup
\Big] \mathrm{d} t.
\end{align*}
Substituting $\zeta(t)$ into \eqref{22}, we see that
\begin{gather*}
\begin{align*}
&D^2J(\Lambda) \cdot (\theta^1, \theta^2) + DJ(\Lambda)\cdot\eta = \int\limits_{T_0}^{T_F} \Big[{z^1(r)}^TD^2\ell(x{(r)})z^2(r) 
\\&+ D\ell(x(r))
\int\limits_{T_0}^{r} \Phi(r, s) 
\Big[B(s) \cdot \eta + (z^1(s)^T {\theta^1}^T) C(s) \begingroup\def\arraystretch{0.6}\small\begin{pmatrix}
z^2(s) \\ \theta^2
\end{pmatrix}\endgroup
\Big] \mathrm{d}s
\Big] \mathrm{d}r 
\\
&+ z^1(T_F)^T D^2m(x(T_F))z^2(T_F) + Dm(x(T_F))\cdot \\&\int\limits_{T_0}^{T_F} \Phi(T_F, s) \Big[B(s) \cdot \eta + (z^1(s)^T {\theta^1}^T) C(s)
\begingroup\def\arraystretch{0.6}\small\begin{pmatrix}
z^2(s) \\ \theta^2
\end{pmatrix}\endgroup
\Big]\mathrm{d}s.
\end{align*}\end{gather*}
Note that $DJ(\Lambda) \cdot \eta$ equals $\int\limits_{T_0}^{T_F} D\ell(x(r)) \int\limits_{T_0}^{r} \Phi(r,s) B(s) \cdot \eta\,\mathrm{d}s\mathrm{d} r 
+ Dm(x(T_F))\int\limits_{T_0}^{T_F}\Phi(T_F, s)B(s) \cdot \eta\,\mathrm{d} s$, which is clear from \eqref{eq::DJvar} and \eqref{eq::zeta}. 
Therefore, this leaves
\begin{align*}
D^2J(\Lambda) \cdot (\theta^1, \theta^2) =& \int\limits_{T_0}^{T_F} \Big[z^1(r)D^2\ell(x(r)) z^2(r) + D\ell(x(r))
\\&
\int\limits_{T_0}^\tau \Phi(r,s) (z^1(s)^T {\theta^1}^T) C(s)
\begingroup\def\arraystretch{0.6}\small\begin{pmatrix}
z^2(s) \\ \theta^2
\end{pmatrix}\endgroup
\mathrm{d}s \Big]\mathrm{d} r \\
&+z^1(T_F)^TD^2m(x(T_F))z^2(T_F)\\
&+Dm(x(T_F))\int\limits_{T_0}^{T_F}\Phi(T_F, s) \big(z^1(s)^T {\theta^1}^T\big)
\\&
C(s)\begingroup\def\arraystretch{0.6}\small\begin{pmatrix}
z^2(s) \\ \theta^2
\end{pmatrix}\endgroup
\mathrm{d}s.
\end{align*}
Split the integral over $\mathrm{d}r$, move $D\ell\big(x(r)\big)$ and $Dm(x(T_F))$ into their respective integrals and switch the order of integration of the double integral:
\begin{align*}
=& \int\limits_{T_0}^{T_F}z^1(r)^TD^2\ell(x(r))z^2(r)\,\mathrm{d}r 
\\&
+ \int\limits_{T_0}^{T_F}\int\limits_s^{T_F}D\ell(x(r)) \Phi(r,s) \big(z^1(s)^T {\theta^1}^T\big)C(s)
\begingroup\def\arraystretch{0.6}\small\begin{pmatrix}
z^2(s) \\ \theta^2
\end{pmatrix}\endgroup
\mathrm{d}r \mathrm{d}s \\&
+ z^1(T_F)^T D^2m(x(T_F)) z^2(T_F) 
\\&
+ \int\limits_{T_0}^{T_F}Dm\big(x(T_F)\big)\Phi(T_F,s)\big(z^1(s)^T, {\theta^1}^T\big) C(s) 
\begingroup\def\arraystretch{0.6}\small\begin{pmatrix}
z^2(s) \\ \theta^2
\end{pmatrix}\endgroup
\mathrm{d}s.
\end{align*}
We combine the integrals over $\mathrm{d}s$, and notice that $\rho(r)^T$, in \eqref{eq::Ch2rho}, enters the equations. Furthermore, we switch the dummy variable $s$ to $r$ and put everything back under one integral:
\begin{gather*}
\begin{align*}
=&\int\limits_{T_0}^{T_F}z^1(r)^TD^2\ell(x(r))z^2(r) + \rho(r)^T \big(z^1(r)^T {\theta^1}^T\big) C(r) 
\begingroup\def\arraystretch{0.6}\small\begin{pmatrix}
z^2(r)\\ \theta^2
\end{pmatrix}\endgroup
\mathrm{d}r \\&
+ z^1(T_F)^TD^2m(x(T_F))z^2(T_F).
\end{align*}\end{gather*}
Expand $C(\cdot)$ back out,
\begin{align*}
=&\int\limits_{T_0}^{T_F}z^1(r)^TD^2\ell(x(r))z^2(r) 
\\&
+ \rho(r)^T\big[z^1(r)^TD_1^2f(x(r), \Lambda, r)z^2(r)\big] \\
&+\rho(r)^T \Big[z^1(r)^T D_{1,2}f(x(r), \Lambda, r)\theta^2\big] \\&
+ \rho(r)^T\big[{\theta^1}^T D_{2,1}f(x(r), \Lambda, r) z^2(r)\big] \\
&+\rho(r)^T \Big[{\theta^1}^T	D_2^2f(x(r), \Lambda, r) \theta^2\big] \mathrm{d}r 
\\&+ z^1(T_F)^TD^2m(x(T_F))z^2(T_F).
\end{align*}
Switching to index notation, where $\rho_k(\cdot)$ is the $k^{\text{th}}$ component of $\rho(\cdot)$ and $f^k(\cdot, \cdot, \cdot)$ is the $k^{\text{th}}$ component of $f(\cdot, \cdot, \cdot)$, 
\begin{align*}
=&\int\limits_{T_0}^{T_F}z^1(r)^TD^2\ell(x(r)) z^2(r) 
\\&
+ z^1(r)^T \sum_{k=1}^n \rho_k (r) D_1^2f^k(x(r, \Lambda, r) z^2(r) \\
&+z^1(r)^T\sum_{k=1}^n \rho_k(r) D_{1,2}f^k(x(r), \Lambda, r) \theta^2 
\\&
+ {\theta^1}^T \sum_{k=1}^n\rho_k(r)D_{2,1}f^k(x(r), \Lambda, r) z^2(r)\\
&+{\theta^1}^T \sum_{k=1}^n\rho_k(r)D_2^2f^k(x(r), \Lambda, r) \theta^2 \mathrm{d}r 
\\&
+ z^1(T_F)^TD^2(x(T_F))z^2(T_F).
\end{align*}
Rearrange the terms allows $D^2J(\Lambda) \cdot (\theta^1, \theta^2)$ to be partitioned into the summation of parts $P_1, P_2, P_3$ given by
\begin{align*}
P_1 =&\int\limits_{T_0}^{T_F} z^1(r)^T\Big[D^2\ell(x(r)) + \sum_{k=1}^n \rho_k(r)D_1^2f^k\big(x(r), \Lambda, r\big)\big]
\\&
z^2(r)\mathrm{d}r + z^1(T_F)^TD^2m\big(x(T_F)\big)z^2(T_F), \\
P_2 =& \int\limits_{T_0}^{T_F}{\theta^2}^T \sum_{k=1}^n \rho_k(r) D_{2,1}f^k\big(x(r), \Lambda, r)z^1(r) + \\&
\theta^1 \sum_{k=1}^n \rho_k(r)D_{2,1}f^k\big(x(r), \Lambda, r\big)z^2(r)\mathrm{d}r, 
\\ 
P_3 =& \int\limits_{T_0}^{T_F} {\theta^1}^T \sum_{k=1}^n \rho_k(r)D_2^2f^k\big(x(r), \Lambda, r\big) \theta^2 \mathrm{d}r
\end{align*}
Looking at $P_1$ first, let 
\begin{align*}
g(r) = D^2\ell\big(x(r)\big) + \sum_{k=1}^n\rho_k(r)D_1^2f^k\big(x(r), \Lambda, r).
\end{align*}
Then,
\begin{align*}
P_1 = \int\limits_{T_0}^{T_F}z^1(r)^Tg(r)z^2(r) \mathrm{d}r + z^1(T_F)D^2m\big(x(T_F)\big)z^2(T_F).
\end{align*}
Substituting \eqref{eq::zeta} for $z^{\cdot}(\cdot)$, results in
\begin{align*}
=&\int\limits_{T_0}^{T_F}\big[\int\limits_{T_0}^r\Phi(r,s)B(s)\theta^1 \mathrm{d}s\big]^Tg(r) \int\limits_{T_0}^r\Phi(r, \omega)B(\omega)\theta^2\mathrm{d}w \mathrm{d}r\\
&+\big[\int\limits_{T_0}^{T_F}\Phi(T_F, s)B(s)\theta^1\mathrm{d}s\big]^T D^2m\big(x(T_F)\big)\int\limits_{T_0}^{T_F}\Phi(T_F, w)
\\&
B(w)\theta^2\mathrm{d}w.
\end{align*}
The integrals may be specified as follows:
\begin{align*}
=&\int\limits_{T_0}^{T_F}\int\limits_{T_0}^r \int\limits_{T_0}^r {\theta^1}^TB(s)^T\Phi(r,s)^Tg(r)\Phi(r,w)B(w)\theta^2 \mathrm{d}s \mathrm{d}w \mathrm{d}r\\
&+\int\limits_{T_0}^{T_F}\int\limits_{T_0}^{T_F}\theta^1B(s)^T\Phi(T_F, s)^TD^2m\big(x(T_F)\big)\Phi(T_F, w)B(w)\theta^2 
\\&
\mathrm{d}s\mathrm{d}w.
\end{align*}
Note that the volume of the triple integral is given by $r = max(s,w)$. Therefore, the order of integration may be switched to:
\begin{align*}
=&\int\limits_{T_0}^{T_F}\int\limits_{T_0}^{T_F}\int\limits_{max(s,w)}^{T_F} \!\!\!\!\!{\theta^1}^TB(s)^T \Phi(r, s)^Tg(r)\Phi(r,w)B(w)\theta^2 \mathrm{d}r
\\&
\mathrm{d}s\mathrm{d}w \\
&+\int\limits_{T_0}^{T_F}\int\limits_{T_0}^{T_F} \theta^1B(s)^T\Phi(T_F,s)^TD^2m\big(x(T_F)\big)\Phi(T_F, w)B(w)\theta^2
\\&
\mathrm{d}s\mathrm{d}w.
\end{align*}
We combine the double integral with the triple integral and rearrange the terms so that only the ones depending on $r$ are inside the internal integral:
\begin{align*}
=&\int\limits_{T_0}^{T_F} \int\limits_{T_0}^{T_F} B(s)^T \Big[\int\limits_{\mathclap {max(s,w)}}^T \Phi(r,s)^T g(r) \Phi(r,w) \mathrm{d}r 
\\&
+ \Phi(T_F,s)^TD^2m\big(x(T_F)\big)\Phi(T_F, w)\Big]B(w)\mathrm{d}s \mathrm{d}w \\&\cdot (\theta^1, \theta^2).
\end{align*}
Let
\begin{gather*}
\begin{align*}
\Omega(t) = \int\limits_t^{T_F}\Phi(r,t)g(r)\Phi(r,t)\mathrm{d}r + \Phi(T_F, t)^TD^2m\big(x(T_F)\big)\Phi(T_F, t)
\end{align*}\end{gather*}
where $\Omega(t) \in \mathbb{R}^{n\times n}$ is the integral curve to the following differential equation
\begin{align*}
\dot{\Omega}(t) =& - A(t)^T\Omega(t) - \Omega(t)A(t) - g(t)\\
=&- A(t)^T\Omega(t) - \Omega(t)A(t) - D^2\ell\big(x(t)\big) 
\\&
- \sum_{k=1}^n\rho_k(t)D_1^2f^k\big(x(t), \Lambda, t\big)\\
&\text{subject to}: \Omega(T_F) = D^2m\big(x(T_F)\big). 
\end{align*}
Then, depending on the relationship between $s$ and $w$, $P_1$ becomes
\begin{gather*}
\begin{align*}
P_1 = 
\begin{cases}
\int\limits_{T_0}^{T_F}\int\limits_{T_0}^{T_F}B(s)^T\Omega(s)\Phi(s,w)B(w) \mathrm{d}s \mathrm{d}w \cdot (\theta^1, \theta^2) & s > w \\
\int\limits_{T_0}^{T_F}\int\limits_{T_0}^{T_F}B(s)^T\Phi(w,s)^T\Omega(w)B(w)\mathrm{d}s \mathrm{d}w \cdot (\theta^1, \theta^2) & s < w \\
\int\limits_{T_0}^{T_F}\int\limits_{T_0}^{T_F}B(s)^T\Omega(s)B(w)\mathrm{d}s \mathrm{d}w \cdot (\theta^1, \theta^2) & s = w, 
\end{cases}
\end{align*}
\end{gather*}
$P_1$ is a scalar and equal to its transpose, therefore
\begin{align*}
P_1 \stackrel{s<w}{=} \int\limits_{T_0}^{T_F}\int\limits_{T_0}^{T_F}B(w)^T\Omega(w)\Phi(w,s)B(s)\mathrm{d}s \mathrm{d}w \cdot (\theta^2, \theta^1)
\end{align*}
Use $i$ and $j$ to index $\theta^1$ and $\theta^2$ respectively, where $\theta$ indicate the variations of $\Lambda$ and $i, j = 1, \ldots, L$. Integrating the $\delta$-functions in $B(s)$ and $B(w)$ will pick out times $s= \tau_i + \lambda_i$ and $w = \tau_j+\lambda_j$ such that $P_{1_{ij}}$ is given by
\begin{align*}
P_{1_{ij}}\begin{cases}\stackrel{i>j}{=}&
\Big[f_i\big(x(\tau_i + \lambda_i),t \big)-f_1\big(x(\tau_i+\lambda_i), t\big)\Big]^T 
\\&
\Omega(\tau_i+\lambda_i)\Phi(\tau_i+\lambda_i, \tau_j+\lambda_j)\\
&\Big[
f_j\big(x(\tau_j + \lambda_j),t \big)-f_1\big(x(\tau_j+\lambda_j), t\big)\Big] \cdot (\theta_i^1, \theta_j^2)\\
\stackrel{i<j}{=}&
\Big[f_j\big(x(\tau_j + \lambda_j),t \big)-f_1\big(x(\tau_j+\lambda_j), t\big)\Big]^T
\\&
\Omega(\tau_j+\lambda_j)\Phi(\tau_j+\lambda_j, \tau_i+\lambda_i)\\
&\Big[f_i\big(x(\tau_i + \lambda_i),t \big)-f_1\big(x(\tau_i+\lambda_i), t\big)\Big] \cdot (\theta_i^1, \theta_j^2)\\
\stackrel{i=j}{=}&
\Big[f_i\big(x(\tau_i+\lambda_i), t\big) - f_1\big(x(\tau_i+\lambda_i), t\big)\Big]^T\Omega(\tau_i+\lambda_i)\\
&\Big[f_i\big(x(\tau_i+\lambda_i), t\big) - f_1\big(x(\tau_i+\lambda_i), t\big)\Big] \cdot (\theta_i^1, \theta_i^2)
\end{cases}\end{align*}
Taking the limit $\Lambda\rightarrow 0$, $\lim_{\Lambda\rightarrow 0} P_{1_{ij}}$ becomes
\begin{align*}
\lim_{\Lambda\rightarrow 0} P_{1_{ij}}\begin{cases}\stackrel{i>j}{=}&
\left[f_i\left(x(\tau_i),t \right)-f_1\left(x(\tau_i), t\right)\right]^T \Omega(\tau_i) \Phi(\tau_i, \tau_j)\\&
\left[
f_j\left(x(\tau_j),t \right)-f_1\left(x(\tau_j), t\right)\right] \cdot (\theta_i^1, \theta_j^2) \\
\stackrel{i<j}{=}&\left[f_j\left(x(\tau_j),t \right)-f_1\left(x(\tau_j), t\right)\right]^T\Omega(\tau_j)\Phi(\tau_j, \tau_i)
\\&
\left[f_i\left(x(\tau_i),t \right)-f_1\left(x(\tau_i), t\right)\right] \cdot (\theta_i^1, \theta_j^2)\\
\stackrel{i=j}{=}&\left[f_i\left(x(\tau_i), t\right) - f_1\left(x(\tau_i), t\right)\right]\Omega(\tau_i)
\\&
\left[f_i\left(x(\tau_i), t\right) - f_1\left(x(\tau_i), t\right)\right] \cdot (\theta_i^1, \theta_i^2).
\end{cases}
\end{align*}

Now consider $P_2$, where
\begin{align*}
D_{2,1}f^k\left(x(t), \Lambda, t\right) =& \left\{\delta\left(t-(\tau_a+\lambda_a)^-\right)D_1f_a^k\left(x(t), t\right)^T 
\right.\\&\left.
- \delta\left(t-(\tau_a+\lambda_a)^+\right)D_1f_1^k\left(x(t), t\right)^T\right\}_{a=1}^L.
\end{align*}
Choose again the $i^\text{th}$ index of $\theta^1$ and the $j^\text{th}$ index of $\theta^2$, where $i,j = 1, \ldots, L$. This corresponds to the $i^\text{th}$ index of $z^1(t)$ and the $j^\text{th}$ index of $z^2(t)$, where the $k^\text{th}$ index of $z^\cdot(\cdot)$ is
\begin{align}\label{zk}
z_k^\cdot(t) =& \int\limits_{T_0}^t\Phi(t, r) \Big[\delta\big(r - (\tau_k+\lambda_k)^-\big)f_k\big(x(r), r\big) 
\notag\\&
- \delta\big(r - (\tau_k + \lambda_k)^+\big)f_1\big(x(r), r\big)\Big] \mathrm{d}r \theta_k^\cdot,
\end{align}
Specifying these indexes allows us to revert back to matrix representation for $\rho(\cdot)$ and $f(\cdot, \cdot, \cdot)$. Thus, 
\begin{align*}
P_{2_{ij}} =& \int\limits_{T_0}^{T_F} \theta_j^2\rho(r)^T\left[\delta\left(r-(\tau_j+\lambda_j)^-\right)D_1f_j\left(x(r), r\right) 
\right.\\&\left.
- \delta\left(r-(\tau_j+\lambda_j)^+\right)D_1f_1\left(x(r), r\right)\right] z_i^1(r) \\
&+\theta_i^1\rho(r)^T\left[\delta\left(r - (\tau_i+\lambda_i)^-\right)D_1f_i\left(x(r), r\right) 
\right.\\&\left.
- \delta\left(r - (\tau_i+\lambda_i)^+\right)D_1f_1\left(x(r), r\right)\right]z_j^2(r)\mathrm{d}r.
\end{align*}
Integrating over the $\delta$-functions picks out the times for which the $\delta$-functions' arguments are zero:
\begin{gather*}
\begin{align*}
=&\theta_j^2\rho\big((\tau_j+\lambda_j)^-\big)D_1f_j\big(x(\tau_j+\lambda_j)^-, (\tau_j+\lambda_j)^- \big) z_i^1\big((\tau_j+\lambda_j)^-\big) \\
&- \theta_j^2\rho\big((\tau_j+\lambda_j)^+\big)D_1f_1\big(x(\tau_j+\lambda_j)^+, (\tau_j+\lambda_j)^+\big)z_i^1\big((\tau_j+\lambda_j)^+\big)\\
&+ \theta_i^1\rho\big((\tau_i+\lambda_i)^-\big)^TD_1f_i\big(x(\tau_i+\lambda_i)^-, (\tau_i+\lambda_i)^-\big)z_j^2\big((\tau_i+\lambda_i)^-\big)\\
&- \theta_i^1\rho\big((\tau_i+\lambda_i)^+\big)^TD_1f_1\big(x(\tau_i+\lambda_i)^+, (\tau_i+\lambda_i)^+\big)z_j^2\big((\tau_i+\lambda_i)^+\big).
\end{align*}
\end{gather*}

The indexes $i$ and $j$ relate in three possible ways: $i < j$, $ i = j$, or $i > j$. The first and last case are the same, which is based on the fact that partial derivatives (with respect to perturbations indexed with $i$ and $j$) commute. 

	Recall that $\tau$ is a set of monotonically increasing times. Therefore, if $i > j$, then $\tau_i + \lambda_i > \tau_j + \lambda_j$. Given \eqref{zk}, $z_k^\cdot (t)$ is non-zero only after time $t=(\tau_k + \lambda_k)^-$. In other words, the state does not change up until the first injected control and so the state perturbation $z$ will be zero for all times prior to the perturbations to the control duration. Consequently, because $t_j + \lambda_j < t_i + \lambda_i$, given that $i > j$ and so $z_i^1(\tau_j+\lambda_j) = 0$. Therefore, the first two terms of $P_{2_{ij}}$ are zero and
\begin{align*}
\stackrel{i>j}{=}&\theta_i^1 \rho\big((\tau_i+\lambda_i)^-\big)^TD_1f_i\big(x(\tau_i+\lambda_i)^-, (\tau_i+\lambda_i)^-\big)
\\&
\Phi\big((\tau_i+\lambda_i)^-,(\tau_j+\lambda_j)^-\big) \Big[f_j\big(x(\tau_j+\lambda_j)^-, (\tau_j+\lambda_j)^-\big) 
\\&
- f_1\big(x(\tau_j+\lambda_j)^+, (\tau_j+\lambda_j)^+\big)\Big]\theta_j^2 -\theta_i^1 \rho\big((\tau_i+\lambda_i)^+\big)^T
\\&
D_1f_1\big(x(\tau_i+\lambda_i)^+, (\tau_i+\lambda_i)^+\big)\Phi\big((\tau_i+\lambda_i)^-,(\tau_j+\lambda_j)^-\big)\\
&\Big[f_j\big(x(\tau_j+\lambda_j)^-, (\tau_j+\lambda_j)^-\big) 
\\&
- f_1\big(x(\tau_j+\lambda_j)^+, (\tau_j+\lambda_j)^+\big)\Big]\theta_j^2.
\end{align*}
Omitting the no longer useful superscripts $+$ and $-$, we see that
\begin{align*}
\stackrel{i>j}{=}&\theta_i^1 \rho\big(\tau_i+\lambda_i\big)^TD_1f_i\big(x(\tau_i+\lambda_i), \tau_i+\lambda_i\big)\Phi\big(\tau_i+\lambda_i,\tau_j+\lambda_j\big)\\
&\Big[f_j\big(x(\tau_j+\lambda_j), \tau_j+\lambda_j\big) - f_1\big(x(\tau_j+\lambda_j), \tau_j+\lambda_j\big)\Big]\theta_j^2\\
&-\theta_i^1 \rho\big(\tau_i+\lambda_i\big)^TD_1f_1\big(x(\tau_i+\lambda_i), \tau_i+\lambda_i\big)\Phi\big(\tau_i+\lambda_i,\tau_j+\lambda_j\big)\\
&\Big[f_j\big(x(\tau_j+\lambda_j), \tau_j+\lambda_j\big) - f_1\big(x(\tau_j+\lambda_j), \tau_j+\lambda_j\big)\Big]\theta_j^2
\end{align*}\begin{align*}
\stackrel{i>j}{=}&\theta_i^1 \rho\big(\tau_i+\lambda_i\big)^T
\\&
\Big[D_1f_i\big(x(\tau_i+\lambda_i), \tau_i+\lambda_i\big) -D_1f_1\big(x(\tau_i+\lambda_i), \tau_i+\lambda_i\big) \Big]
\\&
\Phi\big(\tau_i+\lambda_i,\tau_j+\lambda_j\big)\\
&\Big[f_j\big(x(\tau_j+\lambda_j), \tau_j+\lambda_j\big) - f_1\big(x(\tau_j+\lambda_j), \tau_j+\lambda_j\big)\Big]\theta_j^2.
\end{align*}
Taking the limit $\Lambda\rightarrow 0$, 
\begin{align*}
\lim_{\Lambda\rightarrow 0} P_{2}\stackrel{i>j}{=}&\rho\left(\tau_i\right)^T\left[D_1f_i\left(x(\tau_i), \tau_i\right) -D_1f_1\left(x(\tau_i), \tau_i\right) \right]\\&\Phi\left(\tau_i,\tau_j\right) \left[f_j\left(x(\tau_j), \tau_j\right) - f_1\left(x(\tau_j), \tau_j\right)\right]\cdot (\theta_i^1, \theta_j^2).
\end{align*}

Now consider the $i = j$ case. Because $i = j$, the perturbations $\theta^1$ and $\theta^2$ are equivalent---in the sense that they are both perturbations to the same control duration $\lambda_i$---and therefore $z^1(t)$ and $z^2(t)$ are also equivalent. So,
\begin{align*}
P_{2_{ij}} \stackrel{i=j}{=}& 2\theta_i^2\rho\big((\tau_i+\lambda_i)^-\big)^TD_1f_i\big(x(\tau_i+\lambda_i)^-, (\tau_i+\lambda_i)^- \big)
\\&
z_i^1\big((\tau_i+\lambda_i)^-\big) - 2\theta_i^2\rho\big((\tau_i+\lambda_i)^+\big)^T
\\&
D_1f_1\big(x(\tau_i+\lambda_i)^+, (\tau_i+\lambda_i)^+\big)z_i^1\big((\tau_i+\lambda_i)^+\big)
\end{align*}
Substituting in for $z_i^1(\cdot)$,
\begin{align*}
=&2\theta_i^2\rho\left((\tau_i+\lambda_i)^-\right)^TD_1f_i\left(x(\tau_i+\lambda_i)^-, (\tau_i+\lambda_i)^- \right) 
\\&
\int\limits_{T_0}^\mathclap{(\tau_i+\lambda_i)^-}\Phi\left((\tau_i+\lambda_i)^-, r\right)\cdot 
\left[\delta\left(r - (\tau_i+\lambda_i)^-\right)f_i\left(x(r), r\right) 
\right.\\&\left.
- \delta\left(r - (\tau_i + \lambda_i)^+\right)f_1\left(x(r), r\right)\right]\mathrm{d}r \theta_i^1
\\& - 2\theta_i^2\rho\left((\tau_i+\lambda_i)^+\right)^TD_1f_1\left(x(\tau_i+\lambda_i)^+, (\tau_i+\lambda_i)^+\right) 
\\&
\int\limits_{T_0}^\mathclap{(\tau_i+\lambda_i)^+}\Phi\left((\tau_i+\lambda_i)^+, r\right)\cdot\left[\delta\left(r - (\tau_i+\lambda_i)^-\right)f_i\left(x(r), r\right) 
\right.\\&\left.
- \delta\left(r - (\tau_i + \lambda_i)^+\right)f_1\left(x(r), r\right)\right] \mathrm{d}r \theta_i^1.
\end{align*}
This time the arguments of the $\delta$-functions are zero at the upper bounds of their integrals. So, 
\begin{align*}
=& 2\theta_i^2\rho\left((\tau_i+\lambda_i)^-\right)^TD_1f_i\left(x(\tau_i+\lambda_i)^-, (\tau_i+\lambda_i)^- \right) \cdot\\
&\frac{1}{2} \Phi\left((\tau_i+\lambda_i)^-, (\tau_i+\lambda_i)^-\right)f_i\left(x(\tau_i+\lambda_i)^-, (\tau_i+\lambda_i)^-\right)\theta_i^1 \\
&- 2\theta_i^2\rho\left((\tau_i+\lambda_i)^+\right)^TD_1f_1\left(x(\tau_i+\lambda_i)^+, (\tau_i+\lambda_i)^+\right)\cdot\\
&\left[\Phi\left((\tau_i+\lambda_i)^+, (\tau_i+\lambda_i)^-\right)
f_i\left(x(\tau_i+\lambda_i)^-, (\tau_i+\lambda_i)^-\right) \right.\\
&\left.- \Phi\left((\tau_i+\lambda_i)^+, (\tau_i+\lambda_i)^+\right)\frac{1}{2}f_1\left(x(\tau_i+\lambda_i)^+, (\tau_i+\lambda_i)^+\right)\right]\theta_i^1.
\end{align*}
Recall that $\Phi((\tau_i+\lambda_i)^-, (\tau_i+\lambda_i)^-) = \Phi((\tau_i+\lambda_i)^+, (\tau_i+\lambda_i)^+) = I$ and that $\Phi(\cdot, \cdot)$ is a continuous operator, such that $\Phi((\tau_i+\lambda_i)^+,(\tau_i+\lambda_i)^-) = I$. Therefore, omitting the no longer helpful $-$ and $+$ super-scripts,
\begin{gather*} 
\begin{align*}
=& \rho(\tau_i+\lambda_i)^T\Big[D_1f_i\big(x(\tau_i+\lambda_i), \tau_i+\lambda_i \big)f_i\big(x(\tau_i+\lambda_i), \tau_i+\lambda_i\big)\\
& -2 D_1f_1\big(x(\tau_i+\lambda_i), (\tau_i+\lambda_i)\big)f_i\big(x(\tau_i+\lambda_i), (\tau_i+\lambda_i)\big) \\
&+ D_1f_1\big(x(\tau_i+\lambda_i), (\tau_i+\lambda_i)\big)f_1\big(x(\tau_i+\lambda_i), (\tau_i+\lambda_i)\big)\Big]\cdot(\theta_i^1,\theta_i^2).
\end{align*}
\end{gather*}
Taking the limit $\Lambda\rightarrow 0$, 
\begin{align*}
\lim_{\Lambda\rightarrow 0} P_{2}\stackrel{i=j}{=}&
\rho(\tau_i)^T\left[D_1f_i\left(x(\tau_i), \tau_i \right)f_i\left(x(\tau_i), \tau_i\right) 
\right.\\&\left.
-2 D_1f_1\left(x(\tau_i), (\tau_i)\right)f_i\left(x(\tau_i), (\tau_i)\right) \right.
\\ &\left.+ D_1f_1\left(x(\tau_i), (\tau_i)\right)f_1\left(x(\tau_i), (\tau_i)\right)\right]\cdot(\theta_i^1,\theta_i^2).
\end{align*}

Finally, $P_3$. Start with $D_2^2f^k\big(x(r), \lambda, r\big)$. For $i=j$,
\begin{align*}
D_2^2f^k\left(x(r), \lambda, r\right)_{ij} =&
\left(\frac{\partial}{\partial \Lambda_i} \delta\left(r - (\tau_i+\lambda_i)^-\right)\right)f_i^k\left(x(r), r\right) 
\\&- \left(\frac{\partial}{\partial \Lambda_i} \delta\left(r - (\tau_i+\lambda_i)^+\right)f_1^k\left(x(r), r\right)\right),
\end{align*}
and, for $i\ne j$, $D_2^2f^k\big(x(r), \lambda, r\big)_{ij} = 0$.
Revert back to matrix representation of $\rho(\cdot)$ and $f(\cdot, \cdot)$. For $i=j$, using chair rule on $D_2^2f^k\big(x(r), \Lambda, r\big)_{ij}$ results in:
\begin{align*}
D_2^2f^k\left(x(r), \Lambda, r\right)_{ij} =& - \dot{\delta}\left(r-(\tau_i+\lambda_i)^-\right)f_i^k\left(x(r), r\right) 
\\&+ \dot{\delta}\left(r - (\tau_i+\lambda_i)^+\right)f_1^k\left(x(r), r\right).
\end{align*}
Then,
\begin{align*}
P_3 =& \int\limits_{T_0}^{T_F} \left[-\rho(r)^T \dot{\delta}\left(r-(\tau_i+\lambda_i)^-\right)f_i\left(x(r), r\right) 
\right.\\&\left.
+ \rho(r)^T\dot{\delta}\left(r-(\tau_i+\lambda_i)^+\right)f_1\left(x(r), r\right)\right]\mathrm{d}r\cdot (\theta_i^1, \theta_i^2).
\end{align*}
Using integration by parts, 
\begin{align*}
=& \bigg[-\rho(r)^T\delta\big(r - (\tau_i + \lambda_i)^-\big)f_i\left(x(r), r\right) \bigg|_{T_0}^{T_F} 
\\&
+ \rho(r)^T\delta\left(r-(\tau_i+\lambda_i)^+\right)f_1\left(x(r), r\right)\bigg|_{T_0}^{T_F}
\\&\int\limits_{T_0}^{T_F}\left[\dot{\rho}(r)^Tf_i\left(x(r), r\right) 
\right.\\&\left.
+ \rho(r)^TD_1f_i\left(x(r), r\right) \dot{x}(t) + \rho(r)^TD_2f_i\left(x(r), r\right) \right]\cdot
\\&\delta\left(r - (\tau_i+\lambda_i)^-\right)\mathrm{d}r - \int\limits_{T_0}^{T_F}\Big[\dot{\rho}(r)^Tf_1\left(x(r), r\right) 
\\&
+ \rho(r)^T D_1 f_1\left(x(r), r\right) \dot{x}(t) + \rho(r)^TD_2f_1\big(x(r), r) \Big]
\\&
\delta\left(r-(\tau_i+\lambda_i)^+\right)\mathrm{d}r\bigg] \cdot (\theta_i^1, \theta_i^2).
\end{align*}
Integrating over the $\delta$-functions picks out the times for which the $\delta$-functions' arguments are zero:
\begin{gather*}\begin{align*}
=& \Big[\dot{\rho}\big((\tau_i+\lambda_i)^-\big)f_i\Big(x\big((\tau_i+\lambda_i)^-\big), (\tau_i+\lambda_i)^-\Big) \\&-\dot{\rho}\big((\tau_i+\lambda_i)^+\big)f_1\Big(x\big((\tau_i+\lambda_i)^+\big), (\tau_i+\lambda_i)^+\Big) \\
&+ \rho\big((\tau_i+\lambda_i)^-\big)^TD_1f_i\Big(x\big((\tau_i+\lambda_i)^-\big), (\tau_i+\lambda_i)^-\Big)\dot{x}\big((\tau_i+\lambda_i)^-\big) \\&- \rho\big((\tau_i+\lambda_i)^+\big)^TD_1f_1\Big(x\big((\tau_i+\lambda_i)^+\big), (\tau_i+\lambda_i)^+\Big)\dot{x}\big((\tau_i+\lambda_i)^+\big)\\
&+\rho\big((\tau_i+\lambda_i)^-\big)^TD_2f_i\Big(x\big((\tau_i+\lambda_i)^-\big), (\tau_i+\lambda_i)^-\Big) \\
&- \rho\big((\tau_i+\lambda_i)^+\big)^TD_2f_1\Big(x\big((\tau_i+\lambda_i)^+\big), (\tau_i+\lambda_i)^+\Big)\Big] \cdot (\theta_i^1, \theta_i^2).
\end{align*}\end{gather*}
Using \eqref{eq::Ch2rho}, 
\begin{gather*}
\begin{align*}
=& \bigg[
\Big[
-\rho\big((\tau_i+\lambda_i)^-\big)^TD_1f_i\Big(x\big((\tau_i+\lambda_i)^-\big), (\tau_i+\lambda_i)^-\Big) \\&- D\ell\bigg(x\big((\tau_i+\lambda_i)^-\big)\bigg)
\Big]\cdot f_i\Big(x\big((\tau_i+\lambda_i)^-\big), (\tau_i+\lambda_i)^-\Big) 
\\&- \Big[-\rho\big((\tau_i+\lambda_i)^+\big)^TD_1f_1\Big(x\big((\tau_i+\lambda_i)^+\big), (\tau_i+\lambda_i)^+\Big) \\&- D\ell\bigg(x\big((\tau_i+\lambda_i)^+\big)\bigg)\Big]\cdot f_1\Big(x\big((\tau_i+\lambda_i)^+\big), \lambda_i, (\tau_i+\lambda_i)^+\Big) \\
&+\rho\big((\tau_i+\lambda_i)^-\big)^TD_1f_i\Big(x\big((\tau_i+\lambda_i)^-\big), (\tau_i+\lambda_i)^-\Big)\\&
f_i\Big(x\big((\tau_i+\lambda_i)^-\big), (\tau_i+\lambda_i)^-\Big)-\rho\big((\tau_i+\lambda_i)^+\big)^T\\&D_1f_1\Big(x\big((\tau_i+\lambda_i)^+\big), (\tau_i+\lambda_i)^+\Big)f_1\Big(x\big((\tau_i+\lambda_i)^+\big), (\tau_i+\lambda_i)^+\Big)\\
&+\rho\big((\tau_i+\lambda_i)^-\big)^TD_2f_i\Big(x\big((\tau_i+\lambda_i)^-\big), (\tau_i+\lambda_i)^-\Big) 
\\&- \rho\big((\tau_i+\lambda_i)^+\big)^TD_2f_1\Big(x\big((\tau_i+\lambda_i)^+\big), (\tau_i+\lambda_i)^+\Big)
\bigg]
\cdot (\theta_i^1, \theta_i^2).
\end{align*}\end{gather*}
Canceling out terms, 
\begin{align*}
=& \left[
- D\ell\left(x\left((\tau_i+\lambda_i)^-\right)\right) \left(f_i\left(x\left((\tau_i+\lambda_i)^-\right), (\tau_i+\lambda_i)^-\right) \right.\right.\\ &\left.\left.-f_1\left(x\left((\tau_i+\lambda_i)^+\right), (\tau_i+\lambda_i)^+\right)\right) \right.\\&\left.+ \rho\left((\tau_i+\lambda_i)^-\right)^T\left(D_2f_i\left(x\left((\tau_i+\lambda_i)^-\right), (\tau_i+\lambda_i)^-\right) 
\right.\right.\\&\left.\left.- D_2f_1\left(x\left((\tau_i+\lambda_i)^+\right), (\tau_i+\lambda_i)^+\right)\right)
\right]\cdot \left(\theta_i^1, \theta_i^2 \right) .
\end{align*} 
Then, taking $\Lambda\rightarrow 0$ and omitting the superscripts,
\begin{align*}
\lim_{\Lambda\rightarrow 0}P_3=& \left[ \- D\ell\left(x(\tau_i)\right)\left(f_i\left(x(\tau_i), \tau_i\right)-f_1\left(x(\tau_i), \tau_i\right)\right) 
\ \right.\\&\left.+ \rho(\tau_i)^T\left(D_2f_i\left(x(\tau_i), \tau_i\right) -
D_2f_1\left(x(\tau_i), \tau_i\right)\right)\right]\cdot \\&(\theta_i^1, \theta_i^2).
\end{align*}
Therefore, for $i \ne j$, 
\begin{align*}
\lim_{\Lambda\rightarrow 0} D^2J =&\left[\left[f_i\left(x(\tau_i),\tau_i \right)-f_1\left(x(\tau_i), \tau_i\right)\right]^T \Omega(\tau_i)
\right.\left.\right.
\\&\left.+\rho\left(\tau_i\right)^T\left[D_1f_i\left(x(\tau_i), \tau_i\right) -D_1f_1\left(x(\tau_i), \tau_i\right) \right]\right] \cdot\\
&\Phi(\tau_i, \tau_j)\left[f_j\left(x(\tau_j), \tau_j\right) - f_1\left(x(\tau_j), \tau_j\right)\right]
\end{align*}
and, for $i=j$, 
\begin{align*}
\lim_{\Lambda\rightarrow 0} D^2J=& \Big[f_i\big(x(\tau_i), \tau_i\big) - f_1\big(x(\tau_i), \tau_i\big)\Big]^T\Omega(\tau_i)\\&\Big[f_i\big(x(\tau_i), \tau_i\big) - f_1\big(x(\tau_i), \tau_i\big)\Big] \notag\\
&+\rho(\tau_i)^T\Big[D_1f_i\big(x(\tau_i), \tau_i \big)f_i\big(x(\tau_i), \tau_i\big) \\&-2 D_1f_1\big(x(\tau_i), \tau_i\big)f_i\big(x(\tau_i), \tau_i\big) 
\notag\\ &+ D_1f_1\big(x(\tau_i), \tau_i\big)f_1\big(x(\tau_i), \tau_i\big)+ D_2f_i\Big(x(\tau_i), \tau_i\Big) \\&-
D_2f_1\Big(x(\tau_i), \tau_i\Big)\Big] \notag\\
&- D\ell\left(x(\tau_i)\right)\left(f_i\left(x(\tau_i), \tau_i\right)-f_1\left(x(\tau_i), \tau_i\right)\right).
\end{align*}
Given dynamics of the form \eqref{dynamics}, the MIH (for $i=j$) takes the form in \eqref{MIH}. 
\subsection{Proof of Proposition 5}
\begin{proof}
	The following analysis shows the algebraic dependence of the MIH expression on the first-order Lie brackets $[h_i, h_j]$ and $[g, h_i]$ and proves Proposition 5 if either: 1) $\rho^T [h_i, h_j]~\ne~0$ or 2) $\rho^T[g, h_i]~\ne~0$, as guaranteed by Proposition 4. 
	
	Consider controls such that $u_j~=~v_i~\forall~ j,i\ne k$ and $v_k = 0$ and expresses the MIH expression \eqref{MIH} as 
	\begin{align*}
	\frac{d^2J}{d\lambda_+^2}~=~u^T \mathcal{G} u - u_k((D_xl_1) h_k - \rho^T[g, h_k]) ,
	\end{align*} 
	where $\mathcal{G}_{ij}~=~0~\forall~i, j \in[1,M]\setminus \{k\}$, $\mathcal{G}_{ik}~=~\mathcal{G}_{ki}~=~\frac{1}{2}[h_i, h_k]$, and $\mathcal{G}_{kk}~=~h_k^T\Omega h_k + \rho^TD_xh_k\cdot h_k$. 
	The matrix $\mathcal{G}$ is shown to be either indefinite or negative semidefinite if there exists a Lie bracket term $[h_i, h_k]$ such that $\rho^T [h_i, h_k]~\ne~0$.
	From Proposition \eqref{AdjLie}, there exist $i,~j~\in~[1, M]$ such that either $\rho^T [h_i, h_j]~\ne~0$ or $\rho^T	[g, h_i]~\ne~0$. Let $k~\in~[1, M]$ be an index chosen such that either $\rho^T [h_i, h_k]~\ne~0$ or $\rho^T	[g, h_i]~\ne~0$ for some $i~\in~[1,M]~\setminus~\{k\}$.
	
	We use summation notation to express the MIH as 
	\begin{align*}
	\frac{d^2J}{d\lambda^2}~=~& (\sum_{i=1}^M h_i (u_i-v_i)\big)^T\,\Omega\sum_{j=1}^Mh_j(u_j-v_j) 
	\\
	&+ \rho^T
	\big[\frac{\partial g}{\partial x}g + \frac{\partial g}{\partial x} \sum_j^M h_j u_j + \sum_j^M \frac{\partial h_j}{\partial x} u_j g \\
	& + \sum_j^M \frac{\partial h_j}{\partial x}u_j \sum_j^M h_j u_j + \frac{\partial g}{\partial x}g + \frac{\partial g}{\partial x} \sum_i^M h_i v_i 
	\\&
	+ \sum_i^M \frac{\partial h_i}{\partial x} v_i g + \sum_i^M \frac{\partial h_i}{\partial x}v_i \sum_i^M h_i v_i - 2 \frac{\partial g}{\partial x}g 
	\\&
	-2 \frac{\partial g}{\partial x} \sum_j^M h_j u_j -2 \sum_i^M \frac{\partial h_i}{\partial x} v_i g -2 \sum_i^M \frac{\partial h_i}{\partial x}v_i \sum_j^M h_j u_j 
	\big] 
	\\
	&- \frac{\partial \ell}{\partial x}(\sum_{i=1}^M h_i (u_i-v_i)),
	\end{align*}
	which can be simplified to
	\begin{align*}
	=~&(\sum_{i=1}^M h_i (u_i-v_i)\big)^T\,\Omega\sum_{j=1}^Mh_j(u_j-v_j) 
	\\
	&+ \rho^T	
	\Big(-\frac{\partial g}{\partial x} \sum_j^M h_j u_j + \sum_j^M \frac{\partial h_j}{\partial x} u_j g 
	+ \sum_j^M \frac{\partial h_j}{\partial x}u_j \sum_j^M h_j u_j 
	\\
	&+ \frac{\partial g}{\partial x} \sum_i^M h_i v_i - \sum_i^M \frac{\partial h_i}{\partial x} v_i g + \sum_i^M \frac{\partial h_i}{\partial x}v_i \sum_i^M h_i v_i 
	\\
	&-2\sum_i^M \frac{\partial h_i}{\partial x}v_i \sum_j^M h_j u_j 
	\Big) - \frac{\partial \ell}{\partial x}(\sum_{i=1}^M h_i (u_i-v_i)).
	\end{align*}
	Rearranging the expression into quadratic and linear terms in the control input, we rewrite the MIH expression \eqref{MIH} as 
	\begin{align*}
	=~&(\sum_{i=1}^M h_i (u_i-v_i)\big)^T\,\Omega\sum_{j=1}^Mh_j(u_j-v_j) 
	\\
	&+ \rho^T \big(\sum_j^M \frac{\partial h_j}{\partial x}u_j \sum_j^M h_j u_j + \sum_i^M \frac{\partial h_i}{\partial x}v_i \sum_i^M h_i v_i 
	\\
	&-2 \sum_i^M \frac{\partial h_i}{\partial x}v_i \sum_j^M h_j u_j \big) - \frac{\partial \ell}{\partial x}(\sum_{i=1}^M h_i (u_i-v_i)) \\
	& + 
	\rho^T\big(-\frac{\partial g}{\partial x} \sum_j^M h_j u_j + \sum_j^M \frac{\partial h_j}{\partial x} u_j g + \frac{\partial g}{\partial x} \sum_i^M h_i v_i 
	\\
	&- \sum_i^M \frac{\partial h_i}{\partial x} v_i g \big).
	\end{align*}
	Further considering controls such that $u_j~=~v_i~\forall~ j,~i~\ne~k$ and $v_k~=~0$,
	\begin{align*}
	=~&(h_k u_k)^T\Omega (h_k u_k) + \rho^T \mathcal{D} - \frac{\partial \ell}{\partial x}(h_k u_k) 
	\\
	&+ \rho^T\big(-\frac{\partial g}{\partial x} h_k u_k + \frac{\partial h_k}{\partial x} u_k g \big),
	\end{align*}
	where $u_k$ is the $k^{\text{th}}$ control input and $\mathcal{D}$ is 
	\begin{align*}
	\mathcal{D}~=~& \sum_j^M \frac{\partial h_j}{\partial x}u_j \sum_j^M h_j u_j + \sum_i^M \frac{\partial h_i}{\partial x}v_i \sum_i^M h_i v_i 
	\\
	&-2 \sum_i^M \frac{\partial h_i}{\partial x}v_i \sum_j^M h_j u_j \\
	=~&(\sum_{j\ne k}^{M} \frac{\partial h_j}{\partial x} u_j)\sum_{j\ne k}^{M}h_ju_j + (\frac{\partial h_k}{\partial x}u_k) \sum_{j\ne k}^{M} h_ju_j 
	\\
	&+ (\sum_{j\ne k}^{M}\frac{\partial h_j}{\partial x}u_j)h_k u_k 
	\\
	&+ \frac{\partial h_k}{\partial x}u_k h_k u_k + (\sum_{i\ne k}^{M} \frac{\partial h_i}{\partial x} v_i)\sum_{i\ne k}^{M}h_iv_i 
	\\
	&-2 (\sum_{i\ne k}^{M}\frac{\partial h_i}{\partial x}v_i)\sum_{j\ne k}^{M}h_ju_j - 2(\sum_{i\ne k}^{M}\frac{\partial h_i}{\partial x} v_i)h_k u_k \\
	=~& u_k \sum_{j\ne k}^{M}\frac{\partial h_k}{\partial x} h_ju_j - u_k (\sum_{j\ne k}^{M}\frac{\partial h_j}{\partial x}u_j h_k) + \frac{\partial h_k}{\partial x}u_k h_k u_k \\
	=~& u_k \big[\sum_{j\ne k}^{M}u_j \big(\frac{\partial h_k}{\partial x}h_j - \frac{\partial h_j}{\partial x}h_k\big)\big] + \frac{\partial h_k}{\partial x}u_k h_k u_k \\
	=~& u_k \big[\sum_{j\ne k}^{M}u_j [h_j, h_k] \big]+ \frac{\partial h_k}{\partial x}u_k h_k u_k,
	\end{align*}
	where terms cancel because $u_j~=~v_i~\forall~j,~i~\ne~k$.
	We use the property $x^T A x~=~x^T \left(\frac{1}{2}(A+A^T)\right) x$ and we write $\mathcal{D}$ in a matrix form,
	\begin{gather*}
	= u^T
	\begin{pmatrix} 0 & ... & \frac{1}{2}[h_1, h_k] & ... &0 \\
	\vdots & \ddots & \frac{1}{2}[h_2, h_k] & \iddots & \vdots \\
	\frac{1}{2}[h_1, h_k] & \frac{1}{2}[h_2, h_k] & \frac{\partial h_k}{\partial x} h_k & \frac{1}{2}[h_{M-1}, h_k] & \frac{1}{2}[h_M, h_k]\\ 
	\vdots & \iddots & \frac{1}{2}[h_{M-1}, h_k] & \ddots & \vdots \\
	0 & ... & \frac{1}{2}[h_M, h_k] & ... & 0 
	\end{pmatrix}
	u.
	\end{gather*}
	The dotted entries in the matrix represent zero terms. Combining all terms, the MIH can be written as
	\begin{align}
	\frac{d^2J}{d\lambda^2}~=~& u^T \mathcal{G} u - \frac{\partial \ell}{\partial x}(h_k u_k) + \rho^T\big(-\frac{\partial g}{\partial x} h_k u_k + \frac{\partial h_k}{\partial x} u_k g \big)\\
	=~& u^T\mathcal{G}u - u_k\big(\frac{\partial \ell}{\partial x}h_k - \rho^T[g, h_k]\big),
	\end{align}
	where
	\begin{gather*}
	\mathcal{G} =\begin{pmatrix} 0 & ... & \frac{1}{2}\rho^T[h_1, h_k] & ... & 0 \\
	\vdots & \ddots & \frac{1}{2}\rho^T[h_2, h_k] & \iddots & \vdots \\
	\frac{1}{2}\rho^T[h_1, h_k] & \frac{1}{2}\rho^T[h_2, h_k] & \mathcal{C}_1 & \frac{1}{2}\rho^T[h_{M-1}, h_k] & \frac{1}{2}\rho^T[h_M, h_k]\\ 
	\vdots & \iddots & \frac{1}{2}\rho^T[h_{M-1}, h_k] & \ddots &\vdots \\
	0 & ... & \frac{1}{2}\rho^T[h_M, h_k] & ... & 0
	\end{pmatrix},
	\end{gather*}
	and $\mathcal{C}_1~=~h_k^T \Omega h_k +\rho^T\frac{\partial h_k}{\partial x} h_k$.
	
	Given that $\frac{dJ}{d\lambda_+}~=~0$, then, by Proposition \ref{AdjVec}, $\rho^T h_i~=~0~\forall~i~\in~[1,M]$. In addition, by Proposition \ref{nonzerorho}, $\rho~\ne~0$ and, by Proposition \ref{AdjLie}, there exist $i,~j~\in~[1,M]$ such that $\rho^T [h_i,h_j]~\ne~0$ or $\rho^T[g,h_i]~\ne~0$. It is more convenient to consider two cases that capture all possible scenarios: 1) $\rho^T [h_i,h_j]~\ne~0$ and 2) $\rho^T [h_i,h_j]~=~0$ (which implies $\rho^T [g,h_i]~\ne~0$, by Proposition 3).
	\begin{case}
		$\rho^T[h_i, h_j]~\ne~0$.
	\end{case}	
	Let $\mathcal{G}_{[i,j]}$ denote a $2 \times 2$ matrix obtained from $\mathcal{G}$ by deleting all but its $i^\text{th}$ and $j^\text{th}$ row and $i^\text{th}$ and $j^\text{th}$ column
	\begin{align*}
	\mathcal{G}_{[i,j]}~=~\begin{bmatrix} \mathcal{G}_{ii} & \mathcal{G}_{ij} \\ \mathcal{G}_{ji} & \mathcal{G}_{jj} \end{bmatrix},
	\end{align*}
	where $\mathcal{G}_{ij}~=~\mathcal{G}_{ji}$ because $\mathcal{G}$ is symmetric.
	The principal minors of $\mathcal{G}$ of order 2 are given by $\Delta_2~=~\det(\mathcal{G}_{[i,j]})~=~\mathcal{G}_{ii}\,\mathcal{G}_{jj}~-~\mathcal{G}_{ij}^2\,\mathcal{G}_{ji}~\forall~i~\ne~j~\in~[1,M]$. 
	
	Consider first the diagonal terms of $\mathcal{G}_{[i,j]}$. We note that, because $i~\ne~j$ and $\mathcal{G}_{ii}~=~0~\forall~i~\ne~k$, then either $\mathcal{G}_{ii}~=~0$ or $\mathcal{G}_{jj}~=~0$. Therefore, $\Delta_2~=~-~\mathcal{G}_{ij}\mathcal{G}_{ji}~\forall~i~\ne~j~\in~[1,M]$. Next, consider the off-diagonal elements. We note that $\mathcal{G}_{ij}~=~0~\forall~i,~j~\in~[1,M]~\setminus\{k\}$. Given that $\mathcal{G}_{ik}~=~\mathcal{G}_{ki}~=~\frac{1}{2} \rho^T[h_i , h_k]~\forall~i~\in~[1,M]~\setminus~\{k\}$, we have $\Delta_2~=~0~\forall~i~\in~[1,M]~\setminus~\{k\}$ and $\Delta_2~=~-~\frac{1}{4}(\rho^T[h_i, h_k])^2$, otherwise. We summarize these cases as follows
	\begin{align*}
	\Delta_2~=~\begin{cases}
	0&\forall~i,~j~\in~[1,M]~\setminus~\{k\}\\ 
	-\mathcal{G}_{ij}^2~=~- \frac{1}{4}(\rho^T[h_i, h_k])^2~\le~0&\text{otherwise}.
	\end{cases}
	\end{align*}
	If there exists $i~\in~[1,M]$ such that $\rho^T[h_i, h_j]~\ne~0$, there is at least one negative second-order principal minor. Therefore, $\mathcal{G}$ is indefinite and so there exist controls $u~\in~\mathbb{R}^M$ such that $u^T\mathcal{G}u~<~0$. 
	
	Choose $u~\in~\mathbb{R^M}$ such that $u^T\mathcal{G}u~<~0$ and let $u_k~\in~\mathbb{R}$ represent the $k^{th}$ element of $u$. If $u_k\big(\rho^T [g_, h_k] - D_xl h_k\big)~\le~0$, then
	\begin{align*}
	u^T\mathcal{G}u~<~0~\Longrightarrow~u^T\mathcal{G}u + u_k\big(\rho^T [g_, h_k] - D_xl h_k\big)~<~0.
	\end{align*}
	Else, if $u_k\big(\rho^T [g_, h_k] - D_xl h_k\big)~>~0$, choose $u'~=~-u$ so that 
	\begin{align*}
	u_k'\big(\rho^T [g_, h_k] - D_xl h_k\big)~=~- u_k\big(\rho^T [g_, h_k] - D_xl h_k\big)~<~0 
	\end{align*}
	and
	\begin{align*}
	&u'^T\mathcal{G}u' + u_k'\big(\rho^T [g_, h_k] - D_xl h_k\big) \\
	=~&u^T\mathcal{G}u - u_k\big(\rho^T [g_, h_k] - D_xl h_k\big)~<~0.
	\end{align*}
	Therefore, if $\rho^T[h_i, h_j]~\ne~0$, there always exists $u~\in~\mathbb{R}^M$ such that $\frac{d^2J}{d\lambda^2}~<~0$.
	\begin{case}
		$\rho^T[h_i, h_j]~=~0$.
	\end{case}	
	If $\rho^T [h_i, h_j]~=~0~\forall~i,~j~\in~[1,M]$, then, shown in Proposition \ref{AdjLie}, there exists $i~\in~[1,M]$ such that $\rho^T[g,h_i]~\ne~0$. For $\rho^T[h_i,h_j]~=~0$, the MIH becomes
	\begin{align}\label{MIHsimple3}
	u_k^2\big(h_k^T\Omega h_k + \rho^TD_xh_k\big) + u_k\big(\rho^T [g_, h_k] - D_xl h_k\big),
	\end{align}
	which is a quadratic expression of the form $ax^2 + bx + c$. Quadratic expressions become negative if and only if $a~\le~0$ or $b^2 - 4ac~>~0$. Therefore, \eqref{MIHsimple3} takes negative values if and only if, for some time $t~\in~[t_o, t_o+T]$,
	\begin{enumerate}
		\item $h_k^T\Omega h_k + \rho^TD_xh_k~\le~0$ OR
		\item $\rho^T [g_, h_k] - D_xl h_k~\ne~0$.
	\end{enumerate}	
	We consider the second condition: $\rho^T [g_, h_k] - D_xl h_k~\ne~0$. Because $\frac{dJ}{d\lambda_+}~=~0~\forall~u~\in~\mathbb{R}^M,~\forall~t~\in~[t_o, t_o+T]$, then 
	\begin{align*}
	\frac{dJ}{d\lambda_+}~=~0~\Longrightarrow~& \rho^T h_i~=~0~\forall~i~\in~[1,M],~\forall~t~\in~[t_o,~t_o+T]\\
	\Longrightarrow~& \rho^T h_k~=~0~\forall~t~\in~[t_o,~t_o+T]\\
	\Longrightarrow~& \rho^T\,h_k~=~0~\text{for $t~=~t_o+T$}\\&\text{AND}~\dot{\rho}^Th_k~=~0 ~\forall~t ~\in~[t_o,~t_o+T]\\
	\Longrightarrow~& D_xm\,h_k~=~0~\text{for $t~=~t_o+T$}\\&\text{AND}
	\\&~(-D_xl -\rho^T D_xf_2)\,h_k =0~\forall~t\in~[t_o,~t_o+T]\\
	\Longrightarrow~& (x-x_d)^T\,P_1\,h_k~=~0~\text{for $t~=~t_o+T$}\\&\text{AND}~(-(x-x_d)^T\,Q -\rho^T D_xf_2)\,h_k~=~0~\\
	&\forall~t~\in~[t_o,~t_o+T].\\
	\end{align*}
	Consider positive-definite weight matrices $Q~=~\delta P_1 \succ 0$, where $\delta$ is a scale factor. Then,
	\begin{align*}
	{(x-x_d)^TP_1h_k~=~0}\rvert_{t_o+T}~\Leftrightarrow~&{(x-x_d)^T\delta P_1h_k~=~0}\rvert_{t_o+T}\\\Leftrightarrow~&{(x-x_d)^T Qh_k~=~0}\rvert_{t_o+T},
	\end{align*}
	and 
	\begin{align*}
	\frac{dJ}{d\lambda_+} =0~\Longrightarrow~&D_xm\,h_k~=~0~\\\Leftrightarrow~& D_xl\,h_k~=~0~\text{AND}~\rho^T D_xf_2\,h_k =0.
	\end{align*}
	
	Then, $\rho^T [g, h_k] - D_xlh_k~=~\rho^T[g, h_k] \ne0$. Therefore, there exist control solutions $u\in\mathbb{R}^M$ such that the MIH expression becomes negative.
\end{proof}

\subsection{Derivation of equation \eqref{optcon}}\label{Apx::A}	
In the following derivation, we treat the first- and second-order mode insertion gradient terms separately. 

Associate $f_1$ with default control $v$ and $f_2$ with injected control $u$, such that
\begin{equation*}
\begin{split}
f_1 & \triangleq f(x(t), v(t)) = g(x(t)) + h(x(t))v(t)\\
f_2 & \triangleq f(x(t), u(t)) = g(x(t)) + h(x(t))u(t)
\end{split}
\end{equation*}
For simplicity, we drop the arguments as necessary. For the mode insertion gradient, the update step is straightforward
\begin{align}\label{Apx: eq1}
\frac{\partial}{\partial u} \frac{dJ}{d\lambda_+} =& \frac{\partial}{\partial u} \rho^T (f_2 - f_1) 
= \frac{\partial}{\partial u} \rho^T h (u - v) 
= \rho^T h, \\
\frac{\partial^2}{\partial u^2} \frac{dJ}{d\lambda_+} =& 0.
\end{align}
The update step on the MIH is more complicated and so we divide the MIH expression into three parts
\begin{equation*}
\frac{d^2J}{d\lambda_+^2} = \mathcal{A}_1 + \mathcal{A}_2 + \mathcal{A}_3,
\end{equation*}
where the terms $ \mathcal{A}_1, \mathcal{A}_2, \mathcal{A}_3 $ are given by the following set of equations
\begin{align*}
\mathcal{A}_1 =& (f_2 - f_1)^T\Omega(f_2-f_1) \notag\\ 
\mathcal{A}_2 =&
\rho^T(D_xf_2 \cdot f_2 + D_xf_1\cdot f_1 - 2 D_xf_1\cdot f_2)\notag\\
\mathcal{A}_3 =& - D_x l \cdot (f_2 - f_1).
\end{align*}
Let $l_2 = \frac{d^2J}{d\lambda_+^2}$. Using the G\^{a}teux derivative,
\begin{align*}
\begin{split}
\frac{\partial l_2}{\partial u}
& = \frac{\partial l_2 (u + \epsilon \eta)}{\partial \epsilon} \Bigr|_{\epsilon = 0} = \frac{\partial \mathcal{A}_1(\cdot)}{\partial \epsilon} + \frac{\partial \mathcal{A}_2(\cdot)}{\partial \epsilon} + \frac{\partial \mathcal{A}_3(\cdot)}{\partial \epsilon} \Bigr|_{\epsilon = 0}.
\end{split}
\end{align*}
Then,
\begin{align*}
\frac{\partial \mathcal{A}_1}{\partial u}
=&
\frac{\partial}{\partial \epsilon}\mathcal{A}_1(u+\epsilon \eta)\Bigr|_{\epsilon = 0}\notag\\
=&
\frac{\partial}{\partial \epsilon} ( [h\big((u+\epsilon\eta) - v\big)]^T\Omega~[h\big( (u+\epsilon\eta)-v\big)	])\Bigr|_{\epsilon = 0}\notag\\
=&
(h\eta)^T\Omega~(h \cdot (u+\epsilon\eta - v))\notag
+
(h(u+\epsilon\eta - v))^T\Omega~h \eta\Bigr|_{\epsilon = 0}\notag\\
=&
\eta^Th^T\Omega~hu
-
\eta^Th^T\Omega~h v
+
u^{T}h^T\Omega~h \eta
-
v^Th^T\Omega~h \eta
\notag\\
=&
\Big(u^{T}h^T \big(\Omega^T + \Omega\big)h - v^Th^T \big(\Omega^T + \Omega\big)h\Big) \eta
\end{align*}
\begin{align*}
\frac{\partial \mathcal{A}_2}{\partial u}
=&\frac{\partial}{\partial \epsilon}\mathcal{A}_2(u+\epsilon \eta)\Bigr|_{\epsilon = 0}\\
=&
\frac{\partial}{\partial \epsilon}\rho^T\left(D_xf_1\cdot f_1 - 2~D_xf_1\cdot f_2 + 
D_xf_2\cdot f_2\right)\Bigr|_{\epsilon = 0}\\
=&
\frac{\partial}{\partial \epsilon}\rho^T\left(D_x(g + h v)\cdot (g + h v) 
\right.\\
&\left.
- 2D_x(g + h v) \left(g + h(u +\epsilon\eta)\right)
\right.\\
&\left.+ 
D_x\left(g + h(u +\epsilon\eta)\right)\left(g +h(u +\epsilon\eta) \right)\right)\Bigr|_{\epsilon = 0}\\
=&
\rho^T
\left(- 
2~D_x(g+hv)h\eta
+ 
D_x(h\eta)\cdot(g +hu) \right.\\
&\left.+ D_x(g + hu)\cdot h\eta\right)\\
=&
\rho^T\left(-2D_x(g + hv) h\eta + D_x(h\eta)g + D_x(h\eta)hu	
\right.\\&\left.
+
D_xg h \eta + D_x(hu)h\eta \right)\notag \\
=&
\rho^T \left(-D_xg h\eta -2D_x(hv) h\eta 
+
D_x( h~\eta)\cdot g
\right.\\	&\left.
+ 
D_x( h~\eta)\cdot hu + D_x( hu)\cdot h\eta \right)\notag\\
=&\rho^T\Big(-D_xg \cdot h\eta -2D_x(\sum_{k=1}^{m} h_ku_{d_k})\cdot h\eta
\\&+
D_x(\sum_{k=1}^{m} h_k\eta_{k})\cdot g + 
D_x(\sum_{k=1}^{m} h_k\eta_{k})hu 
\\&
+ 
D_x(\sum_{k=1}^{m} h_ku_{2_k})\cdot h\eta \Big) \notag\\
=& 
- \rho^T D_xg h\eta
- 2 v^T (\sum_{k=1}^{n} (D_xh_k)\rho_{k})h\eta
\\ &+ \eta^T (\sum_{k=1}^{n} (D_xh_k)\rho_{k})g \notag 
+ \eta^T (\sum_{k=1}^{n} (D_xh_k)\rho_{k})hu 
\\&+ u^{T}(\sum_{k=1}^{n} (D_xh_k)\rho_{k}) h\eta 
\\
=&
- \rho^T D_xg h\eta
- 2 v^T (\sum_{k=1}^{n} (D_xh_k)\rho_{k})\cdot h\eta
\\&+ g^T(\sum_{k=1}^{n} (D_xh_k)\rho_{k})^T \eta \notag +
u^{T} h^T \cdot (\sum_{k=1}^{n} (D_xh_k)\rho_{k})^T \eta \\&+ u^{T}(\sum_{k=1}^{n} (D_xh_k)\rho_{k})\cdot h\eta
\end{align*}\begin{align*} 
=&\Big[-\rho^T D_xg \cdot h
- 2 v^T (\sum_{k=1}^{n} (D_xh_k)\rho_{k})\cdot h
\\&+ g^T (\sum_{k=1}^{n} (D_xh_k)\rho_{k})^T 
+ u^{T} h^T \cdot (\sum_{k=1}^{n} (D_xh_k)\rho_{k})^T \notag \\&+ u^{T}(\sum_{k=1}^{n} (D_xh_k)\rho_{k})\cdot h \Big]\cdot\eta. 
\end{align*}
Last,
\begin{align*}
\frac{\partial \mathcal{A}_3}{\partial u}
=&\frac{\partial}{\partial \epsilon}\mathcal{A}_3(u+\epsilon \eta)\Bigr|_{\epsilon = 0}\notag\\
=& -\frac{\partial l}{\partial x} \frac{\partial}{\partial \epsilon}
\Big( g + h(u +\epsilon \eta) - g - hv\Big)\Bigr|_{\epsilon = 0}\notag\\
=&-
\frac{\partial l}{\partial x} h\eta
\end{align*}
Therefore,
\begin{align*}
\frac{\partial l_2}{\partial u} =&
u^{T}\Big[ h^T \big(\Omega^T + \Omega\big) h + 
h^T (\sum_{k=1}^{n} (D_xh_k)\rho_{k})^T
\\&+ 
(\sum_{k=1}^{n} (D_xh_k)\rho_{k}) h \Big]
- v^T\Big[ h^T \big(\Omega^T + \Omega\big) h \\&+
2 (\sum_{k=1}^{n} (D_xh_k)\rho_{k} )h\Big]-\rho^T D_xg \cdot h
\\&
+ g^T (\sum_{k=1}^{n} (D_xh_k)\rho_{k})^T - \frac{\partial l}{\partial x} h.
\end{align*}
Solving for the minimizer, $\frac{\partial l_2}{\partial u}^T = 0$, we get 
\begin{align}\label{Apx: eq3}
\frac{\partial l_2}{\partial u}^T = 0 \Rightarrow& \Big[ h^T \big(\Omega^T + \Omega\big) h + 
h^T (\sum_{k=1}^{n} (D_xh_k)\rho_{k})^T\notag
\\&+ 
(\sum_{k=1}^{n} (D_xh_k)\rho_{k}) h \Big]u = \Big[ h^T \big(\Omega^T + \Omega\big) h \notag\\&+2h^T
(\sum_{k=1}^{n} (D_xh_k)\rho_{k} )^T\Big]v+D_xg^T\rho h\notag
\\&
-(\sum_{k=1}^{n} (D_xh_k)\rho_{k})g +h^TD_xl^T\notag\\
\Rightarrow& u =  \Big[ h^T \big(\Omega^T + \Omega\big) h + 
h^T (\sum_{k=1}^{n} (D_xh_k)\rho_{k})^T\notag
\\&+ 
(\sum_{k=1}^{n} (D_xh_k)\rho_{k}) h \Big]^{-1}\Big[ h^T \big(\Omega^T + \Omega\big) h \notag\\&+2h^T
(\sum_{k=1}^{n} (D_xh_k)\rho_{k} )^T\Big]v+D_xg^T\rho h\notag
\\&
-(\sum_{k=1}^{n} (D_xh_k)\rho_{k})g +h^TD_xl^T
\end{align}
The terms shown in \eqref{Apx: eq1} and \eqref{Apx: eq3}, together with a penalty term for the control, are the gradient and Hessian terms used for the Newton update step that appears in \eqref{optcon}.
\end{document}